\documentclass[11pt]{article}
\usepackage{fullpage}
\usepackage{amsmath,amsthm,amsfonts,amssymb,dsfont,bm}
\usepackage{enumerate,color,xcolor}
\usepackage[colorlinks,linkcolor=cyan,citecolor=blue]{hyperref}
\usepackage{url}
\usepackage{enumitem}
\usepackage{graphicx} 
\usepackage{float}
\usepackage{subcaption}
\usepackage{comment}
\numberwithin{equation}{section}

\theoremstyle{plain}
\newtheorem{theorem}{Theorem}[section]
\newtheorem{corollary}[theorem]{Corollary}
\newtheorem{lemma}[theorem]{Lemma}
\newtheorem{proposition}[theorem]{Proposition}
\newtheorem{assumption}[theorem]{Assumption}

\theoremstyle{remark}
\newtheorem{remark}[theorem]{Remark}

\newcommand{\R}{\mathbb{R}}
\newcommand{\ip}[2]{\left\langle #1,#2\right\rangle}
\newcommand{\tr}{\mathrm{Tr}}

\title{Variance Reduction for Stochastic Gradient Generalized Non-reversible Langevin Monte Carlo Algorithms}
\author{Bingye Ni\thanks{FinTech Thrust, Hong Kong University of Science and Technology (Guangzhou), Guangzhou, Guangdong, People's Republic of China; \texttt{bni151@connect.hkust-gz.edu.cn}}
\and
Xiaoyu Wang\thanks{Corresponding author. FinTech Thrust, Hong Kong University of Science and Technology (Guangzhou), Guangzhou, Guangdong, People's Republic of China; \texttt{xiaoyuwang@hkust-gz.edu.cn}}
  \and
  Yingli Wang\thanks{School of Mathematical Science, Fudan University, Shanghai, People's Republic of China; \texttt{yingliwang@fudan.edu.cn}}
  \and
  Lingjiong Zhu\thanks{Department of Mathematics, Florida State University, Tallahassee, Florida, United States of America; \texttt{zhu@math.fsu.edu}}}
\date{\today}

\begin{document}

\maketitle

\begin{abstract}
We study the leading-order fluctuation of stochastic gradient Euler--Maruyama estimators for generalized non-reversible Langevin dynamics. Under structural assumptions tailored to the small-stepsize central limit theorem 
and under an unbiased stochastic gradient oracle, we prove that the empirical average over a horizon of order the inverse squared stepsize satisfies a central limit theorem in the vanishing-stepsize regime. 
The limiting variance is characterized through the Poisson equation of the limiting full-gradient diffusion. We then rewrite this constant in an operator form that links it to the continuous-time asymptotic variance and, under standard operator-theoretic assumptions, derive a sufficient condition under which an anti-symmetric perturbation strictly reduces the leading-order fluctuation constant relative to the reversible baseline. We also identify bounded smooth predictive observables that are directly covered by the main theorem. As a separate Gaussian calculation beyond the bounded-test-function regime, we obtain closed-form formulas for quadratic Hamiltonians and linear observables. The framework covers non-reversible Langevin dynamics and augmented-state examples including Hessian-free high-resolution dynamics and a positive-definite subclass of gradient-adjusted underdamped Langevin dynamics that allow stochastic gradients.
Numerical experiments on basic examples and Bayesian linear regression using synthetic data, and Bayesian logistic regression using real data support the predicted Gaussian fluctuations and show that the non-reversible schemes consistently reduce the root mean squared error (RMSE) relative to their reversible baselines.
\end{abstract}

\section{Introduction}

The problem of \textit{sampling} a target distribution
$\pi(\theta)\propto\exp(-U(\theta))$
for a function $U:\mathbb{R}^d \to \mathbb{R}$
arises in many applications, including Bayesian statistical inference \cite{gelman1995bayesian}, Bayesian formulations of inverse problems \cite{stuart2010inverse}, and Bayesian classification and regression tasks in machine learning \cite{andrieu2003introduction,teh2016consistency,DistMCMC19,GIWZ2024,DIGing,regime2025}.

\textit{Langevin algorithms} are popular Markov chain Monte Carlo (MCMC) methods that allow one to sample from a target distribution of interest.
Although Langevin algorithms have a long history and are well studied in computational physics and chemistry (see e.g. \cite{LRS2010,coffey2012langevin,leimkuhler2016molecular} and the references therein), their applications to machine learning lead to a number of computational challenges and require rethinking and redesigning \cite{ma2015complete,teh2016consistency}.
The classical Langevin algorithm is based on the discretization of
{\it overdamped Langevin dynamics} \cite{Dalalyan,DM2017,DK2017,Raginsky}:
\begin{equation}\label{eq:overdamped-2}
d\theta_{t}=-\nabla U(\theta_{t})dt+\sqrt{2}dB_{t},
\end{equation}
where $U:\mathbb{R}^{d}\rightarrow\mathbb{R}$
and $B_{t}$ is a standard $d$-dimensional Brownian motion. Under some mild assumptions on $U(\cdot)$, the diffusion \eqref{eq:overdamped-2} admits a unique stationary distribution with the density $\pi(\theta)\propto\exp(-U(\theta))$,
also known as the \emph{Gibbs distribution} \cite{chiang1987diffusion,stroock-langevin-spectrum}. In computing practice, the diffusion \eqref{eq:overdamped-2} is simulated by considering its discretization,
and one of the most commonly used discretization schemes is the Euler--Maruyama discretization of
\eqref{eq:overdamped-2}, often known as the \textit{unadjusted Langevin algorithm} (ULA) in the literature; see e.g. \cite{DM2017}:
\begin{equation}\label{discrete:overdamped}
\theta_{k+1}=\theta_{k}-\eta\nabla U(\theta_{k})+\sqrt{2\eta}\xi_{k+1},
\end{equation}
where $\xi_{k}$ are independent and identically distributed (i.i.d.) $\mathcal{N}(0,I_{d})$ Gaussian random vectors and $\eta>0$ is the stepsize.
The first non-asymptotic result of \eqref{discrete:overdamped}
is due to \cite{Dalalyan}, which was improved soon after
by \cite{DM2017} with a particular emphasis on the dependence on the dimension $d$.
Both works consider the total variation (TV) distance for the convergence guarantees.
Later, \cite{DM2016} studied the convergence guarantees in the 2-Wasserstein distance,
and \cite{DMP2016} studied variants of \eqref{discrete:overdamped} when $U$ is not smooth.
\cite{CB2018} studied the convergence in the Kullback-Leibler (KL) divergence.

In many settings in machine learning, the full gradient $\nabla U$ can be expensive or impractical to compute. 
For example, in Bayesian learning problems, $U$ can be written as a finite-sum form as the sum of many component functions over all the data points
and the number of data points can be large \cite{xu2018global}. In such settings, algorithms that rely on \emph{stochastic gradients}, i.e., unbiased stochastic estimates of the gradient obtained by a randomized sampling of the data points, is often more efficient \cite{bottou2010large}. This fact motivated the development of Langevin algorithms that can support stochastic gradients. 
In particular, if one replaces the full gradient $\nabla U$ in \eqref{discrete:overdamped} by a stochastic gradient, 
the resulting algorithm
is known as the \emph{stochastic gradient Langevin dynamics} (SGLD) \cite{welling2011bayesian,Raginsky}.
Non-asymptotic convergence for SGLD have been studied; see e.g. \cite{DK2017,Raginsky}.

In the literature, \textit{non-reversible Langevin dynamics} (NLD) and the associated algorithms have been extensively studied.
Compared with their reversible counterparts, they are known to exhibit faster convergence or smaller continuous-time asymptotic variance in a variety of settings, including constrained sampling.
In particular, by adding an \textit{anti-symmetric} matrix $J=-J^{\top}$
to the overdamped Langevin dynamics \eqref{eq:overdamped-2}, 
the \textit{non-reversible Langevin dynamics} (NLD) takes the form:
\begin{equation}\label{eqn:anti}
d\theta_{t}=-(I_{d}+J)\nabla U(\theta_{t})dt+\sqrt{2}dB_{t},
\end{equation}
where $B_{t}$ is a standard $d$-dimensional Brownian motion, and its stationary distribution has density $\pi(\theta)\propto
e^{-U(\theta)}$, which is the same Gibbs distribution as that of the overdamped Langevin dynamics in \eqref{eq:overdamped-2}. 
In computing practice, the diffusion \eqref{eqn:anti} is simulated by considering its discretization,
such as the Euler–Maruyama discretization:
\begin{equation}\label{discrete:anti}
\theta_{k+1}=\theta_{k}-\eta(I_{d}+J)\nabla U(\theta_{k})+\sqrt{2\eta}\xi_{k+1},
\end{equation}
where $\xi_{k}$ are i.i.d. $\mathcal{N}(0,I_{d})$ Gaussian vectors and $\eta>0$ is the stepsize.
The stochastic gradient version of \eqref{discrete:anti}, known as the \emph{non-reversible stochastic gradient Langevin dynamics} (NSGLD), has also been studied; see e.g. \cite{HWGGZ20}.

The dynamics \eqref{eqn:anti} is called \emph{non-reversible Langevin dynamics} (NLD) because this diffusion is a non-reversible Markov process due to the addition of the $J$ matrix whereas the overdamped Langevin diffusion is reversible (see \cite{HHS93,HHS05} for
details) with optimal choice of $J$ discussed in \cite{Lelievre-optdrift,WHC2014}. The main idea explored here is that non-reversible processes converge to their equilibrium often faster than their reversible counterparts \cite{HHS93,HHS05}
and this has been applied
to sampling \cite{rey2016improving,DLP2016,reyGraphs,DPZ17,reyLDP,FSS20}
and non-convex optimization \cite{GGZ2,HWGGZ20}.

It is well-known in the literature that under ergodicity assumption on $(\theta_{t})_{t\geq 0}$, the continuous-time time-average estimator
$\pi_T(h) := \frac{1}{T}\int_0^T h(\theta_t)dt$
satisfies a central limit theorem (CLT):
\begin{align}\label{eqn:CLT}
\sqrt{T}\left(\pi_T(h)-\pi(h)\right)\Rightarrow\mathcal{N}(0,\sigma_h^2),
\end{align}
as $T\rightarrow\infty$, where $\pi$ is the unique invariant distribution for $(\theta_{t})_{t\geq 0}$ and the asymptotic variance $\sigma_h^2$ is characterized by the Poisson equation $-\mathcal{L}\phi = h-\pi(h)$ via
\[
\sigma_h^2 = 2\left\langle h-\pi(h),(-\mathcal{L})^{-1}(h-\pi(h))\right\rangle_{L^2(\pi)};
\]
see e.g.\ \cite{DLP2016}.
Moreover, \cite{rey2016improving} proved a general principle for Markov processes: adding an irreversible perturbation $\mathcal{A}$ (anti-self-adjoint in $L^2(\pi)$) to a reversible infinitesimal generator $\mathcal{L}_0$ always satisfies $\sigma^2(h)\leq\sigma_0^2(h)$ for all $h\in L^2(\pi)$, and simultaneously improves the large deviations rate function for the empirical measure
$\frac{1}{t}\int_0^t\delta_{\theta_s}ds$.
In the Langevin setting, \cite{DLP2016} gave a detailed study of the nonreversible dynamics
$d\theta_t = \left(\nabla\log\pi(\theta_t)+\alpha\gamma(\theta_t)\right)dt+\sqrt{2}dB_t$,
where $\gamma$ is divergence-free with respect to $\pi$ (e.g.\ $\gamma(\theta)=J\nabla\log\pi(\theta)$ for anti-symmetric $J$ as in \eqref{eqn:anti}), $\alpha\in\mathbb{R}$ and $B_{t}$ is a standard $d$-dimensional Brownian motion. They proved that $\sigma_h^2(\alpha)\leq\sigma_h^2(0)$ for all $\alpha\in\R$ and all $f\in L^2(\pi)$, and identified the optimal perturbation explicitly for Gaussian targets and linear or quadratic observables. The variance $\sigma_h^2(\alpha)\to 0$ as $|\alpha|\to\infty$ if and only if the nullspace of $\mathcal{A}=\gamma\cdot\nabla$ in the weighted Sobolev space $\mathcal{H}^1$ contains only constants.
More recently, \cite{WTWZ2026} extended this to the constrained setting, studying the skew-reflected non-reversible Langevin dynamics (SRNLD)
on a compact domain $K\subset\R^d$. They established a large deviation principle for the empirical measure $\frac{1}{t}\int_0^t\delta_{\theta_s}ds$ via the G\"{a}rtner--Ellis theorem, showed that the SRNLD rate function dominates that of the reversible reflected Langevin dynamics, and proved an explicit asymptotic variance reduction for the time-average estimator $\pi_T(h)$.
Finally, it is worth noting that the non-asymptotic convergence guarantees and the acceleration of NLD in \eqref{eqn:anti} to the target distribution $\pi$ have also been obtained in \cite{HHS93,HHS05}.
However, all of these results pertain to the \emph{continuous-time} setting. 

For the Euler--Maruyama discretization in \eqref{discrete:anti} (with stochastic gradients), 
the discretization error has been rigorously analyzed in e.g. \cite{HWGGZ20}. 
It was also extended to the constrained setting (without stochastic gradients) in \cite{DFTWZ2025} recently. 
However, even though the non-asymptotic convergence guarantees
and the acceleration have been obtained for the discretized algorithms
for NLD, the acceleration was due to the acceleration of the convergence
of NLD in continuous-time to the target distribution instead
of the discretization error \cite{HWGGZ20,DFTWZ2025}. 
Moreover, the non-asymptotic convergence guarantees for the 
discretized algorithms for NLD are obtained only through upper bounds on Wasserstein distances, 
and such upper bounds may not be sharp \cite{HWGGZ20,DFTWZ2025}.

On the other hand, in the recent applied probability literature, for Euler--Maruyama discretizations of a stochastic differential equation (with constant and invertible diffusion coefficient $\sigma\in\mathbb{R}^{d\times d}$):
\begin{equation}\label{general:SDE}
dX_{t}=g(X_{t})dt+\sigma dB_{t},
\end{equation}
\cite{lu2022central} established a small-stepsize \textit{central limit theorem} (CLT) for the time-average estimator 
\begin{equation}
\pi_\eta(h):=\frac{1}{[\eta^{-2}]}\sum_{k=0}^{[\eta^{-2}]-1}h(x_{k}),
\end{equation}
as $\eta\rightarrow 0$, where $x_{k}$ is the Euler--Maruyama discretization of \eqref{general:SDE}
\begin{equation}
x_{k+1}=x_{k}+\eta g(x_{k})+\sqrt{\eta}\sigma\xi_{k+1},    
\end{equation}
where $(\xi_k)_{k\ge1}$ are i.i.d.\ standard Gaussian random vectors $\mathcal{N}(0,I_{d})$ in $\R^d$ and $\eta>0$ is the stepsize. 
The limiting variance for the CLT is characterized by the Poisson equation associated with the infinitesimal generator of the limiting diffusion. Such a CLT result can be applied
to the Langevin settings for the Euler--Maruyama discretization in \eqref{discrete:overdamped}
for the overdamped Langevin dynamics in \eqref{eq:overdamped-2} and 
the Euler--Maruyama discretization in \eqref{discrete:anti}
for the non-reversible Langevin dynamics in \eqref{eqn:anti} to help understand
and gain insights at the \textit{algorithm} level.

The result of \cite{lu2022central} is stated for exact Euler--Maruyama schemes. In contrast, the algorithms studied in this paper also allow \emph{stochastic gradients} that can be applicable in many machine learning problems \cite{welling2011bayesian,Raginsky}. Therefore, we do not apply their CLT theorem directly. Instead, we adapt their Poisson-equation martingale decomposition and show that, under an $\eta$-uniform conditional moment bound on the stochastic gradient error, the additional stochastic gradient martingale is negligible because it enters the update at order $\eta$, whereas the Brownian discretization noise enters at order $\sqrt\eta$. This fixed-variance scaling is essential for the stated leading-order variance formula.

This paper connects these two lines of work in the stochastic gradient setting. We focus on the leading-order fluctuation constant of the stochastic gradient discretized empirical measure in the small-stepsize regime $\eta\to0$, rather than on the continuous-time asymptotic variance itself. This viewpoint allows us to analyze and compare overdamped and variants of Langevin dynamics at the algorithm level, namely through their stochastic gradient Euler--Maruyama discretizations.

To demonstrate the breadth of the framework, we propose and study \textit{stochastic gradient generalized non-reversible Langevin Monte Carlo} (SG-GNLMC), that not only covers the classical NLD, but also two augmented-state examples: Hessian-free high-resolution (HFHR) dynamics and a positive-definite subclass of gradient-adjusted underdamped Langevin (GAUL) dynamics, that also allow stochastic gradients.

We also record bounded smooth predictive observables that are directly covered by our main theorem. In addition, as a separate Gaussian calculation beyond the bounded-test-function regime, we show that in the quadratic setting the fluctuation constant can be computed in closed form for linear observables, turning the abstract variance-reduction principle into an explicit matrix formula.

The contributions of our paper are summarized as follows:
\begin{enumerate}[label=(C\arabic*)]
    \item Under Assumptions~\ref{ass:U}, \ref{ass:diss}, and the stochastic gradient oracle condition in Assumption~\ref{ass:sg}, we establish a small-stepsize CLT for SG-GNLMC. We show that the empirical average over $[\eta^{-2}]$ stochastic gradient iterates satisfies the following CLT as $\eta\rightarrow 0$ (Theorem~\ref{thm:clt}):
$\eta^{-1/2}\left(\widehat{\mu}_\eta^{\rm SG}(h)-\mu(h)\right)
    \Rightarrow
    \mathcal N(0,V(h))$,
with $V(h)=\mu\left(\left\|\sigma^\top\nabla\phi\right\|^2\right)$,
    where $\phi$ solves the Poisson equation for the limiting full-gradient diffusion. The stochastic gradient perturbation is conditionally unbiased and satisfies an $\eta$-uniform conditional $p$-th moment bound. Under this fixed-variance scaling, it enters the update at order $\eta$, and hence does not alter the leading-order fluctuation constant.
    The proof of Theorem~\ref{thm:clt} builds upon verifying the structural conditions used in the Poisson-equation analysis of \cite{lu2022central} (Proposition~\ref{prop:ltx}), and a series of novel technical lemmas that establish a moment bound (Lemma~\ref{lem:sg-moment}), Poisson-equation law of large numbers along the stochastic gradient chain (Lemma~\ref{lem:sg-poisson-lln}), the estimates that show the negligibility of the stochastic gradient Taylor terms (Lemma~\ref{lem:sg-negligible})
    and exact-gradient Taylor remainder terms (Lemma~\ref{lem:exact-remainder-sg-path}).
    
    \item We connect $V(h)$ to the continuous-time asymptotic variance representation and provide a sufficient condition under which the non-reversible perturbation strictly reduces the leading-order fluctuation constant relative to the reversible baseline, i.e. $V(h)<V_0(h)$, where $V_0(h)$ denotes the corresponding fluctuation constant for the reversible baseline obtained by setting $\mathcal Q=0$ (Theorem~\ref{thm:strict-vr}).
    \item  We identify a class of bounded smooth predictive observables, including logistic predictive probabilities, Brier-type scores, and entropy-type uncertainty scores, that are directly covered by the main $C_b^2$ theory (Proposition~\ref{prop:predictive-observable}). This clarifies the relevance of the fluctuation result to common machine learning prediction tasks. 
    In addition, we illustrate the general framework on three examples: the classical NLD (Proposition~\ref{prop:nld-example}), HFHR (Proposition~\ref{prop:hfhr-example}), and a positive-definite subclass of GAUL (Proposition~\ref{prop:gaul-example}) that also allow stochastic gradients.

    \item As a separate Gaussian calculation beyond the bounded-test-function regime of Theorem~\ref{thm:clt}, we solve the Poisson equation explicitly for quadratic Hamiltonians and linear observables (Proposition~\ref{prop:quadratic-linear}) and derive closed-form expressions for the small-stepsize fluctuation constant (Proposition~\ref{prop:quad-sg-linear-clt}). In particular, this yields fully explicit variance-reduction formulas for NLD (Corollary~\ref{cor:nld-isotropic}), HFHR (Corollary~\ref{cor:HFHR}), and GAUL (Corollary~\ref{cor:GAUL}) with stochastic gradients in the quadratic setting.
    
\item 
We provide numerical experiments for SGNLMC, SGHFHRMC, and SGGAULMC in Sections~\ref{sec:numerical}.
We confirm the closed form fluctuation constants in the quadratic examples,
and show consistent RMSE reductions and Gaussian fluctuation behavior in nonquadratic examples (Section~\ref{sec:basic:numerical}). 
The Bayesian linear regression experiments using synthetic data (Section~\ref{sec:Bayesian:linear})  and the Bayesian logistic regression experiments using real data (Section~\ref{sec:real:data:numerical}) further demonstrate the non-reversible improvement over the corresponding reversible baselines.
\end{enumerate}

\paragraph{Notations.}

We summarize here the notations that will be used throughout the paper.

\begin{itemize}
\item 
For a probability measure $\mu$, we write $L^2(\mu)$ for the Hilbert space of square-integrable real-valued functions on $\R^n$, and $H^0:=\{f\in L^2(\mu):\mu(f)=0\}$ for its mean-zero subspace. We use $C^m(\R^n)$ for the space of $m$ times continuously differentiable functions, $C_b^m(\R^n)$ for the subspace whose derivatives up to order $m$ are bounded, and $C_c^\infty(\R^n)$ for the space of smooth functions with compact support.
\item 
For any probability measure $\mu$ on $\mathbb{R}^{n}$ and $\mu$-integrable $h$ where $h:\mathbb{R}^{n}\rightarrow\mathbb{R}$, we denote $\mu(h):=\int_{\R^n} h(\mathbf{z})\mu(d\mathbf{z})$.
\item 
$\|\cdot\|$ denotes the Euclidean norm on vectors, while $\|\cdot\|_{\mathrm{op}}$ denotes the operator norm on matrices and multilinear maps.
\item
For a scalar-valued function $f:\R^n\to\R$, we write $\nabla f$ for its gradient, $\nabla^2 f$ for its Hessian matrix, and, more generally, $\nabla^k f(\mathbf z)$ for its $k$-th derivative at $\mathbf z$, viewed as a $k$-linear map. In particular, for $H:\R^n\to\R$,
$\nabla^3H(\mathbf z)[u,v,w]
=
\sum_{i,j,\ell=1}^n
\frac{\partial^3 H(\mathbf z)}{\partial \mathbf z_i\partial \mathbf z_j\partial \mathbf z_\ell}
u_i v_j w_\ell$,
for any $u,v,w\in\R^n$.
Equivalently, after fixing $u,v\in\R^n$, we write $\nabla^3H(\mathbf z)[u,v]\in\R^n$ for the unique vector characterized by
$\ip{\nabla^3H(\mathbf z)[u,v]}{w}=\nabla^3H(\mathbf z)[u,v,w]$,
for any $w\in\R^n$.
If $g:\R^n\to\R^n$ is vector-valued, then $\nabla^2g(\mathbf z)$ is the bilinear map taking values in $\R^n$, defined componentwise by the second derivatives of $g$. For a multilinear map $T$, $\|T\|_{\mathrm{op}}$ denotes the induced operator norm; for example,
$\|\nabla^3H(\mathbf z)\|_{\mathrm{op}}
:=
\sup_{\|u\|=\|v\|=\|w\|=1}
\left|\nabla^3H(\mathbf z)[u,v,w]\right|$
and
$\|\nabla^2g(\mathbf z)\|_{\mathrm{op}}
:=
\sup_{\|u\|=\|v\|=1}
\left\|\nabla^2g(\mathbf z)[u,v]\right\|$.
The identity $\nabla^2g(\mathbf z)=-(\mathcal D+\mathcal Q)\nabla^3H(\mathbf z)$ is understood as
$\nabla^2g(\mathbf z)[u,v]
=
-(\mathcal D+\mathcal Q)\left(\nabla^3H(\mathbf z)[u,v]\right)$,
for any $u,v\in\R^n$.
\end{itemize}

%%%%%%%%%%%%%%%%%%%%%%%%%%%%%%%%%%%%%%%%%%%%%%%%%%%%%%%%

\section{Model and Assumptions}\label{sec:model_and_assumptions}

The following stochastic differential equation (SDE) has been studied in the literature as a generalized Langevin dynamics \cite{ma2015complete}:
\begin{equation}\label{eq:gnld}
d\mathbf{z}_{t}=\mathbf{f}(\mathbf{z}_{t})dt+\sqrt{2\mathcal{D}(\mathbf{z}_{t})}d\mathbf{B}_t,
\end{equation}
where $\mathbf{B}_t$ is an $n$-dimensional Brownian motion with $\mathbf{z}_{0}\in\mathbb{R}^{n}$ and $\mathcal{D}(\mathbf{z})$ is a positive semidefinite diffusion matrix, 
and $\mathbf{f}(\mathbf{z})$ is given by
\begin{equation}\label{eq:drift-decomp}
\mathbf{f}(\mathbf{z})=-\left[\mathcal{D}(\mathbf{z})+\mathcal{Q}(\mathbf{z})\right]\nabla H(\mathbf{z})+\Gamma(\mathbf{z}),
\end{equation}
where $\Gamma(\mathbf{z})=\left(\Gamma_1(\mathbf{z}),\ldots,\Gamma_n(\mathbf{z})\right)^{\top}$
with $\Gamma_i(\mathbf{z}):=\sum_{j=1}^{n}\frac{\partial}{\partial\mathbf{z}_j}\left(\mathcal{D}_{ij}(\mathbf{z})+\mathcal{Q}_{ij}(\mathbf{z})\right)$ for every $i=1,\ldots,n$ and $\mathcal{Q}(\mathbf{z})$ is an anti-symmetric matrix ($\mathcal{Q}^\top=-\mathcal{Q}$) representing the deterministic traversing effects seen in Hamiltonian Monte Carlo, where $H(\mathbf{z})$
is the Hamiltonian defined as:
\begin{equation}\label{eqn:Hamiltonian}
H(\mathbf{z}) = H(\theta,r) = U(\theta)+G(\theta,r),
\end{equation}
where $\mathbf z$ represents either $\theta$ itself, or an augmented state space $\mathbf z=(\theta,r)$ with $\theta\in\mathbb R^d$ the model parameter and $r\in\mathbb R^m$ a set of auxiliary variables with $d+m=n$. The function $G$ denotes the auxiliary part of the Hamiltonian.
It is known (see e.g.\ \cite{ma2015complete}) that $\mu\propto e^{-H(\mathbf{z})}$ is a stationary distribution of \eqref{eq:gnld} provided
\begin{equation}\label{eq:QW}
\sum_{i=1}^{n}\sum_{j=1}^{n}\frac{\partial^2}{\partial\mathbf{z}_i\partial\mathbf{z}_j}
\left(\mathcal{Q}_{ij}(\mathbf{z})e^{-H(\mathbf{z})}\right)=0.
\end{equation}

In this paper, we study the \textit{generalized non-reversible Langevin dynamics} (GNLD) \eqref{eq:gnld} with \emph{constant} matrices $\mathcal{D}$ and $\mathcal{Q}$, so that $\Gamma=0$ and the dynamics reduces to
\begin{equation}\label{eq:sde}
d\mathbf{z}_t = -(\mathcal{D}+\mathcal{Q})\nabla H(\mathbf{z}_t)dt + \sqrt{2\mathcal{D}}d\mathbf{B}_t,
\end{equation}
where $\mathcal{D}$ is a constant symmetric positive-definite matrix and $\mathcal{Q}$ is a constant anti-symmetric matrix (i.e. $\mathcal{Q}^\top=-\mathcal{Q}$); specific examples of \eqref{eq:sde} are discussed in Section~\ref{sec:examples}.
The infinitesimal generator of \eqref{eq:sde} is given by
\begin{equation}\label{eq:generator}
\mathcal{A} f(\mathbf{z}) = -\langle(\mathcal{D}+\mathcal{Q})\nabla H(\mathbf{z}),\nabla f(\mathbf{z})\rangle + \operatorname{Tr}\left(\mathcal{D}\nabla^2 f(\mathbf{z})\right).
\end{equation}

Define the target measure
\begin{equation}\label{eq:pi}
\mu(d\mathbf{z})=Z^{-1}e^{-H(\mathbf{z})}d\mathbf{z},
\end{equation}
where $Z>0$ is the normalizing constant, and, for $\mu$-integrable $h$, $\mu(h)=\int_{\R^n} h(\mathbf{z})\mu(d\mathbf{z})$.
Here and below, $\|\cdot\|$ denotes the Euclidean norm on vectors, while $\|\cdot\|_{\mathrm{op}}$ denotes the operator norm on matrices and multilinear maps.

Throughout the paper, the continuous-time dynamics \eqref{eq:sde} should be understood as the limiting full-gradient diffusion. The stochastic gradient approximation is introduced only at the discretization level. Thus, the invariant measure, the generator, and the Poisson equation below are still associated with the full-gradient GNLD.

Next, we introduce the \textit{stochastic gradient generalized non-reversible Langevin Monte Carlo} (SG-GNLMC) algorithm. Let $(\mathsf Z,\mathcal Z)$ be an auxiliary measurable space, and let
$\widehat{\nabla H}:\mathbb R^n\times\mathsf Z\to\mathbb R^n$
be a stochastic gradient oracle for $\nabla H$. The random variables
$(\zeta_k)_{k\ge1}$ take values in $\mathsf Z$ and represent the randomness used by the oracle.
The stochastic gradient Euler--Maruyama scheme is
\begin{equation}\label{eq:em}
\mathbf{z}_{k+1}^{\eta}
=
\mathbf{z}_k^{\eta}
-\eta(\mathcal{D}+\mathcal{Q})
\widehat{\nabla H}(\mathbf z_k^{\eta},\zeta_{k+1})
+\sqrt{2\eta\mathcal{D}}\xi_{k+1},
\qquad k\ge0,
\end{equation}
where $(\xi_k)_{k\ge1}$ are i.i.d. standard Gaussian random vectors
$\mathcal N(0,I_n)$. The oracle randomness $(\zeta_k)_{k\ge1}$ is specified through the filtration and conditional moment assumptions in Assumption~\ref{ass:sg}.
Equivalently, with $\mathcal D=\frac12\sigma\sigma^\top$, we may write
\[
\mathbf z_{k+1}^{\eta}
=
\mathbf z_k^{\eta}
+\eta g(\mathbf z_k^{\eta})
+\sqrt{\eta}\sigma\xi_{k+1}
+\eta\varepsilon_{k+1},
\]
where
$g(\mathbf z):=-(\mathcal D+\mathcal Q)\nabla H(\mathbf z)$,
and
\[
\varepsilon_{k+1}
:=
-(\mathcal D+\mathcal Q)
\left[
\widehat{\nabla H}(\mathbf z_k^{\eta},\zeta_{k+1})
-\nabla H(\mathbf z_k^{\eta})
\right].
\]

We define the stochastic gradient empirical measure by
$\widehat{\mu}_{\eta}^{\rm SG}(\cdot)=
\frac{1}{[\eta^{-2}]}\sum_{k=0}^{[\eta^{-2}]-1}\delta_{\mathbf{z}_k^{\eta}}(\cdot)$,
where $[a]$ denotes the integer part of $a\in\R$ such that for a test function $h$,
\[
\widehat{\mu}_{\eta}^{\rm SG}(h)
=
\frac{1}{[\eta^{-2}]}\sum_{k=0}^{[\eta^{-2}]-1}
h(\mathbf z_k^\eta).
\]

\begin{assumption}\label{ass:U}
There exist constants $L,M>0$, $m>0$, and $b\ge 0$ such that:
\begin{enumerate}[label=(\roman*)]
    \item $H\in C^4(\R^n)$, $\sup_{\mathbf{z}\in\R^n}\|\nabla^2 H(\mathbf{z})\|_{\mathrm{op}}\le L$, and $\sup_{\mathbf{z}\in\R^n}\|\nabla^3 H(\mathbf{z})\|_{\mathrm{op}}\le M$;
    \item $\ip{\nabla H(\mathbf{z}_1)-\nabla H(\mathbf{z}_2)}{\mathbf{z}_1-\mathbf{z}_2}\ge m\|\mathbf{z}_1-\mathbf{z}_2\|^2-b$ for all $\mathbf{z}_1,\mathbf{z}_2\in\R^n$.
\end{enumerate}
\end{assumption}

\begin{remark}\label{rem:assU-scope}
In Assumption~\ref{ass:U}(i), 
the condition $\sup_{\mathbf{z}\in\R^n}\|\nabla^2 H(\mathbf{z})\|_{\mathrm{op}}\le L$
is essentially an $L$-Lipschitz condition
on the gradient of $H$, also known as the $L$-smoothness condition,
which is commonly assumed in the machine learning literature \cite{Raginsky,GGZ}.
The additional condition 
$\sup_{\mathbf{z}\in\R^n}\|\nabla^3 H(\mathbf{z})\|_{\mathrm{op}}\le M$
essentially assumes $H$ is $M$-Hessian Lipschitz, 
which is another common assumption 
in the machine learning literature \cite{DM2016,Ma2021}.
Assumption~\ref{ass:U}(ii) is a dissipativity condition that is a common assumption in the machine learning literature in general that allows $H$ to be non-convex, and it is satisfied 
in many non-convex learning problems, such as one-hidden-layer neural networks \cite{akiyama2022excess}, non-convex formulations of classification problems \cite{GGZ}, 
sampling and Bayesian learning and global convergence in non-convex optimization \cite{Raginsky,GGZ}
and many other examples \cite{erdogdu2022}. 
It is included as a convenient growth hypothesis compatible with the framework of \cite{lu2022central}; the one-sided dissipativity hypothesis that is actually needed for the drift $g$ is stated separately in Assumption~\ref{ass:diss}. Moreover, under Assumption~\ref{ass:U}(i)--(ii), $L$-smoothness together with $(m,b)$-dissipativity implies that $H$ admits quadratic upper and lower bounds; see, for instance, Lemma~2 in \cite{Raginsky}. In particular, the normalizing constant $Z=\int_{\R^n}e^{-H(\mathbf z)}d\mathbf z$ is finite, and thus the Gibbs measure \eqref{eq:pi} is well defined.
\end{remark}

\begin{assumption}[One-sided dissipativity]\label{ass:diss}
Let $g(\mathbf{z}) := -(\mathcal{D}+\mathcal{Q})\nabla H(\mathbf{z})$. There exist constants $K_1>0$ and $K_2\ge 0$ such that
\[
\ip{g(\mathbf{z}_1)-g(\mathbf{z}_2)}{\mathbf{z}_1-\mathbf{z}_2}
\le -K_1\|\mathbf{z}_1-\mathbf{z}_2\|^2 + K_2,
\qquad \mathbf{z}_1,\mathbf{z}_2\in\R^n.
\]
\end{assumption}

\begin{assumption}[Stochastic gradient oracle]\label{ass:sg}
Let $\mathcal F_k$ be the filtration generated by $\mathbf z_0^\eta$ and by $\{\xi_j,\zeta_j:1\le j\le k\}$. Assume that $\mathbf z_0^\eta$ is independent of $\{(\xi_j,\zeta_j):j\ge1\}$ and that
$\sup_{\eta\in(0,1)}\mathbb E\|\mathbf z_0^\eta\|^p<\infty$
for some $p>4$. Define
\[
\delta_{k+1}
:=
\widehat{\nabla H}(\mathbf z_k^\eta,\zeta_{k+1})
-\nabla H(\mathbf z_k^\eta),
\qquad
\varepsilon_{k+1}:=-(\mathcal D+\mathcal Q)\delta_{k+1}.
\]
We assume that the stochastic gradient oracle is Markovian in the following sense: conditional on $\mathcal F_k$, the law of $\widehat{\nabla H}(\mathbf z_k^\eta,\zeta_{k+1})$ depends on the past only through $\mathbf z_k^\eta$. Moreover, for every $k\ge0$,
\begin{enumerate}[label=(\roman*)]
\item $\mathbb E[\varepsilon_{k+1}\mid\mathcal F_k]=0$;
\item $\varepsilon_{k+1}$ is conditionally independent of $\xi_{k+1}$ given $\mathcal F_k$;
\item there exists $C_p>0$ such that
\begin{align}\label{conditional:moment:bound}
\mathbb E\left[
\|\varepsilon_{k+1}\|^p
\mid\mathcal F_k
\right]
\le
C_p\left(1+\|\mathbf z_k^\eta\|^p\right).
\end{align}
\end{enumerate}
\end{assumption}

\begin{remark}[Fixed-variance stochastic gradient scaling]\label{rem:sg-scope}
Assumption~\ref{ass:sg} is a fixed-variance stochastic gradient condition. More precisely, the conditional moment bound \eqref{conditional:moment:bound}
%\[
%\mathbb E\left[
%\|\varepsilon_{k+1}\|^p
%\mid\mathcal F_k
%\right]
%\le
%C_p\left(1+\|\mathbf z_k^\eta\|^p\right)
%\]
holds with a constant $C_p$ independent of $\eta$. This covers the usual fixed-mini-batch stochastic gradient oracle when the corresponding gradient error has a finite conditional $p$-th moment \cite{Raginsky}.

The independence of $C_p$ from $\eta$ is important. If the stochastic gradient variance were allowed to grow as $\eta\to0$, for example like $\eta^{-1}$, then the term $\eta\varepsilon_{k+1}$ could become comparable to the Brownian increment $\sqrt\eta\sigma\xi_{k+1}$. In that high-variance regime, the stochastic gradient noise may contribute to the leading fluctuation constant. Such $\eta$-dependent high-variance stochastic gradient regimes are not covered in this paper.
\end{remark}

\begin{remark}\label{rem:diss-suff}
Assumption~\ref{ass:diss} is the dissipativity hypothesis needed to invoke \cite{lu2022central}. In the special case of NLD, i.e. $\mathcal{D}=I_{d}$, Assumption~\ref{ass:U}(i)--(ii) yields
\[
\ip{g(\mathbf{z}_1)-g(\mathbf{z}_2)}{\mathbf{z}_1-\mathbf{z}_2}
\le -(m-\|\mathcal{Q}\|_{\mathrm{op}}L)\|\mathbf{z}_1-\mathbf{z}_2\|^2+b,
\]
so that the convenient sufficient condition $m>\|\mathcal{Q}\|_{\mathrm{op}}L$ implies Assumption~\ref{ass:diss} with $K_1=m-\|\mathcal{Q}\|_{\mathrm{op}}L$ and $K_2=b$. We do not assume this sufficient condition in general.
\end{remark}

\begin{remark}\label{rem:diss-relax}
For the special case of NLD discussed in Remark~\ref{rem:diss-suff}, the condition
$m>\|\mathcal{Q}\|_{\mathrm{op}}L$
is only a convenient sufficient criterion for Assumption~\ref{ass:diss}, and need not be sharp. For the quadratic Hamiltonian $H(\mathbf{z})=\frac12\mathbf{z}^{\top}P\mathbf{z}$ with symmetric positive definite $P$, the NLD-type drift $g(\mathbf{z})=-(I_n+\mathcal{Q})P\mathbf{z}$ satisfies
\begin{align}\label{second:term}
\ip{g(\mathbf{z}_1)-g(\mathbf{z}_2)}{\mathbf{z}_1-\mathbf{z}_2}
=
-\Delta\mathbf{z}^{\top}P\Delta\mathbf{z}
-\Delta\mathbf{z}^{\top}\mathcal{Q}P\Delta\mathbf{z},
\qquad
\Delta\mathbf{z}:=\mathbf{z}_1-\mathbf{z}_2.
\end{align}
If $P\mathcal{Q}=\mathcal{Q}P$, then $\mathcal{Q}P$ is anti-symmetric and the term $\Delta\mathbf{z}^{\top}\mathcal{Q}P\Delta\mathbf{z}$ in \eqref{second:term} vanishes identically. In that case, Assumption~\ref{ass:diss} follows from the coercivity of $P$ alone, without requiring $m>\|\mathcal{Q}\|_{\mathrm{op}}L$. This shows that the bound $m>\|\mathcal{Q}\|_{\mathrm{op}}L$ is mainly a rough sufficient condition. More refined sufficient conditions for broader classes of Hamiltonians may also be possible, but we do not pursue them here.
\end{remark}

\begin{remark}[Dissipativity for augmented-state examples]\label{rem:aug-diss}
Assumption~\ref{ass:diss} is imposed on the full augmented drift, and it does not follow automatically from the dissipativity of $U$ alone. For the HFHR drift in Section~\ref{sec:examples}, however, it can be verified under a simple parameter condition. Let
$q:=\nabla U(\theta_1)-\nabla U(\theta_2)$, $x:=\theta_1-\theta_2$ and
$y:=r_1-r_2$.
Assume
$\langle q,x\rangle\ge m\|x\|^2-b$,
$\|q\|\le L\|x\|$.
For
$g_{\rm HFHR}(\theta,r)
=
\begin{pmatrix}
r-\alpha\nabla U(\theta)\\
-\gamma r-\nabla U(\theta)
\end{pmatrix}$,
we have
\[
\begin{aligned}
\langle g_{\rm HFHR}(\theta_1,r_1)-g_{\rm HFHR}(\theta_2,r_2),(x,y)\rangle
&=
-\alpha\langle q,x\rangle-\gamma\|y\|^2+\langle y,x\rangle-\langle q,y\rangle\\
&\le
-\alpha m\|x\|^2-\gamma\|y\|^2+(1+L)\|x\|\|y\|+\alpha b.
\end{aligned}
\]
Thus, if
$\alpha\gamma m>\frac{(1+L)^2}{4}$,
then the matrix
$\begin{pmatrix}
\alpha m & -(1+L)/2\\
-(1+L)/2 & \gamma
\end{pmatrix}$
is positive definite, and Assumption~\ref{ass:diss} holds for HFHR with suitable constants $K_1>0$ and $K_2\ge0$.

For GAUL in Section~\ref{sec:examples} with a general preconditioner $C$, one needs a corresponding preconditioned dissipativity condition. A convenient sufficient condition is that there exist $m_C>0$ and $b_C\ge0$ such that
\[
\left\langle
C\left(\nabla U(\theta_1)-\nabla U(\theta_2)\right),
\theta_1-\theta_2
\right\rangle
\ge
m_C\|\theta_1-\theta_2\|^2-b_C.
\]
Then, for
$g_{\rm GAUL}(\theta,r)
=
\begin{pmatrix}
Cr-aC\nabla U(\theta)\\
-\gamma r-\nabla U(\theta)
\end{pmatrix}$,
the bound
\[
\begin{aligned}
\langle g_{\rm GAUL}(\theta_1,r_1)-g_{\rm GAUL}(\theta_2,r_2),(x,y)\rangle
&\le
-a m_C\|x\|^2-\gamma\|y\|^2
+(\|C\|_{\rm op}+L)\|x\|\|y\|+a b_C
\end{aligned}
\]
holds. Consequently, Assumption~\ref{ass:diss} is satisfied whenever
$a\gamma m_C>\frac{(\|C\|_{\rm op}+L)^2}{4}$.
For $C=I_d$, this reduces to the HFHR-type condition with $a$ in place of $\alpha$.
\end{remark}

%%%%%%%%%%%%%%%%%%%%%%%%%%%%%%%%%%%%%%%%%%%%%%%%%%%%%%%%

\section{Main Results}\label{sec:main}

We begin by verifying that the Gibbs measure $\mu$ introduced in \eqref{eq:pi} is indeed the invariant measure of the constant-coefficient GNLD \eqref{eq:sde}. This provides the probabilistic foundation for the fluctuation analysis carried out in the rest of this section.

\begin{lemma}[Invariance of the target measure]\label{lem:invariance}
Under Assumption~\ref{ass:U} and the standing assumptions that
$\mathcal D$ is constant symmetric positive definite and
$\mathcal Q$ is constant anti-symmetric, the measure $\mu$ in
\eqref{eq:pi} is invariant for \eqref{eq:sde}.
\end{lemma}

\begin{proof}
Since $\mathcal D$ and $\mathcal Q$ are constant, we have
$\Gamma_i(\mathbf z)
=
\sum_{j=1}^n
\frac{\partial}{\partial \mathbf z_j}
\left(\mathcal D_{ij}+\mathcal Q_{ij}\right)
=0$,
$i=1,\ldots,n$.
Thus \eqref{eq:sde} is a constant-coefficient special case of the generalized
Langevin dynamics \eqref{eq:gnld}--\eqref{eq:drift-decomp}.
It remains to verify condition \eqref{eq:QW}. Since $\mathcal Q$ is constant,
$\sum_{i,j=1}^n
\frac{\partial^2}{\partial \mathbf z_i\partial \mathbf z_j}
\left(\mathcal Q_{ij}e^{-H}\right)
=
\sum_{i,j=1}^n
\mathcal Q_{ij}\frac{\partial^2}{\partial \mathbf z_i\partial \mathbf z_j}(e^{-H})$.
Using
$\frac{\partial^2}{\partial \mathbf z_i\partial \mathbf z_j}\left(e^{-H}\right)
=
\left(\frac{\partial}{\partial \mathbf z_i}H\,\frac{\partial}{\partial\mathbf z_j}H-\frac{\partial^2}{\partial \mathbf z_i\partial \mathbf z_j}H\right)e^{-H}$,
we obtain
\begin{equation}\label{eq:QW-constant-expansion}
\sum_{i,j=1}^n
\mathcal Q_{ij}\frac{\partial^2}{\partial \mathbf z_i\partial \mathbf z_j}\left(e^{-H}\right)
=
\left(
\nabla H^\top \mathcal Q\nabla H
-
\operatorname{Tr}\left(\mathcal Q\nabla^2H\right)
\right)e^{-H}.
\end{equation}
The first term on the right-hand side of \eqref{eq:QW-constant-expansion},
namely $\nabla H^\top\mathcal Q\nabla H$,
is zero because $\mathcal Q$ is anti-symmetric. The second term on the right-hand side of \eqref{eq:QW-constant-expansion}, namely $\operatorname{Tr}(\mathcal Q\nabla^2H)$, is zero because $\mathcal Q$ is anti-symmetric while $\nabla^2H$ is
symmetric. Hence \eqref{eq:QW} holds. The claimed invariance follows from
\cite[Theorem~1]{ma2015complete}.
\end{proof}

%%%%%%%%%%%%%%%%%%%%%%%%%%%%%%%%%%%%%%%%%%%%
\subsection{Small-stepsize CLT and variance reduction}

Now, we are ready to state our first main result, that provides the small-stepsize central limit theorem for SG-GNLMC algorithm.

\begin{theorem}[Small-stepsize central limit theorem for SG-GNLMC]\label{thm:clt}
Suppose Assumptions~\ref{ass:U}, \ref{ass:diss}, and~\ref{ass:sg} hold. Let $(\mathbf z_k^\eta)_{k\ge0}$ be generated by the stochastic gradient scheme \eqref{eq:em}, and let $h\in C_b^2(\mathbb R^n)$. Then
\begin{equation}\label{eq:clt}
\eta^{-1/2}
\left(
\widehat{\mu}_{\eta}^{\rm SG}(h)-\mu(h)
\right)
\Rightarrow
\mathcal N\left(0,V(h)\right),
\qquad\text{as $\eta\to0$},
\end{equation}
where
\[
V(h)=\mu\left(\|\sigma^\top\nabla\phi\|^2\right),
\]
with $\mathcal D=\frac12\sigma\sigma^\top$ and $\phi$ solves the limiting full-gradient Poisson equation
\[
h-\mu(h)=\mathcal A\phi.
\]
%In particular, under the $\eta$-uniform conditional moment bound in Assumption~\ref{ass:sg}, the fixed-variance stochastic gradient perturbation does not change the leading-order fluctuation constant.
\end{theorem}

\begin{proof}
The proof will be provided in Section~\ref{sec:main:proof}.
\end{proof}

\begin{remark}[Operator representation of $V(h)$]\label{rem:V-operator}
Let $\hat h:=h-\mu(h)$ and $\phi$ be the solution of $\mathcal A\phi=\hat h$. Write $\mathcal A=\mathcal A^{\mathrm{sym}}+\mathcal A^{\mathrm{as}}$, where
\[
\mathcal A^{\mathrm{sym}}f
:=
-\langle\mathcal D\nabla H,\nabla f\rangle+\operatorname{Tr}\left(\mathcal D\nabla^2 f\right),
\qquad
\mathcal A^{\mathrm{as}}f
:=
-\langle\mathcal Q\nabla H,\nabla f\rangle.
\]
Since $\mathcal A^{\mathrm{as}}$ is anti-self-adjoint in $L^2(\mu)$, we have
\[
-\langle \phi,\mathcal A\phi\rangle_{L^2(\mu)}
=
\mu\left(\langle\nabla\phi,\mathcal D\nabla\phi\rangle\right)
=
\frac12\mu\left(\left\|\sigma^\top\nabla\phi\right\|^2\right)
=
\frac12 V(h).
\]
Since $\hat h=\mathcal A\phi$, it follows that
$V(h)=-2\langle \mathcal A\phi,\phi\rangle_{L^2(\mu)}=2\left\langle \hat h,-\phi\right\rangle_{L^2(\mu)}$.
In particular, whenever $-\mathcal A$ is invertible on $H^0$, this identity can be written as
$V(h)=2\left\langle \hat h,(-\mathcal A)^{-1}\hat h\right\rangle_{L^2(\mu)}$.
The same representation applies to the reversible baseline generator $\mathcal A_0$.
\end{remark}

\begin{remark}
By Theorem~\ref{thm:clt}, under the fixed-variance stochastic gradient scaling of Assumption~\ref{ass:sg}, the stochastic gradient perturbation does not change the leading-order fluctuation constant. Therefore, the comparison of fluctuation constants in the stochastic gradient setting is the same operator comparison as for the limiting full-gradient diffusion.
\end{remark}

Theorem~\ref{thm:clt} identifies the fluctuation constant $V(h)$, but it does not yet explain when non-reversibility is genuinely beneficial. The next result gives a sufficient condition under which the anti-symmetric perturbation strictly improves on the reversible baseline at the level of the small-stepsize fluctuation constant.

\begin{theorem}[Sufficient condition for strict reduction of the fluctuation constant]\label{thm:strict-vr}
Let $H^0 := \{f\in L^2(\mu):\mu(f)=0\}$, $\hat h := h-\mu(h)$, and let $\phi_0\in C^2(\R^n)$ solve the Poisson equation for the reversible baseline ($\mathcal{Q}=0$):
\[
\mathcal{A}_0\phi_0 = \hat h,
\]
where $\mathcal{A}_0 f(\mathbf{z}) := -\langle\mathcal{D}\nabla H(\mathbf{z}),\nabla f(\mathbf{z})\rangle + \operatorname{Tr}\left(\mathcal{D}\nabla^2 f(\mathbf{z})\right)$.
Write $\mathcal{A} = \mathcal{A}_0 + \mathcal{A}^{\mathrm{as}}$, where $\mathcal{A}^{\mathrm{as}} := -\langle\mathcal{Q}\nabla H,\nabla\cdot\rangle$ is the anti-symmetric part. Assume:
\begin{enumerate}[label=(\roman*)]
  \item $-\mathcal{A}_0$ is positive, self-adjoint, and invertible on $H^0$;
  \item the operator $\mathcal{B} := (-\mathcal{A}_0)^{-1/2}(-\mathcal{A}^{\mathrm{as}})(-\mathcal{A}_0)^{-1/2}$ is bounded on $H^0$.
\end{enumerate}
Let $V(h)$ and $V_0(h)$ denote the leading small-stepsize fluctuation constants of the SG-GNLMC schemes associated with $\mathcal Q$ and with $\mathcal Q=0$, respectively. By Theorem~\ref{thm:clt}, these constants coincide with the operator quantities associated with the corresponding limiting full-gradient diffusions. Then $V(h)\leq V_0(h)$. Moreover, if
\begin{equation}\label{eq:strict-cond}
\int_{\R^n} \left|\left\langle\mathcal{Q}\nabla H(\mathbf{z}),\nabla\phi_0(\mathbf{z})\right\rangle\right|^2\mu(d\mathbf{z}) > 0,
\end{equation}
then $V(h)<V_0(h)$.
\end{theorem}

\begin{proof}
Since $-\mathcal{A}_0$ is invertible on $H^0$, the element $\tilde{g}:=(-\mathcal{A}_0)^{-1/2}\hat h\in H^0$ is well defined. By Remark~\ref{rem:V-operator}, the small-stepsize fluctuation constant admits the same Poisson-equation operator representation as the continuous-time asymptotic variance functional. Therefore, \cite[Lemma~2]{rey2016improving} yields
\[
\frac12 V(h) = \left\langle \tilde{g},(I+\mathcal{B}^*\mathcal{B})^{-1}\tilde{g}\right\rangle_{L^2(\mu)}, \qquad
\frac12 V_0(h) = \langle \tilde{g},\tilde{g}\rangle_{L^2(\mu)}.
\]
Since $\mathcal{A}^{\mathrm{as}}$ is anti-self-adjoint in $L^2(\mu)$ and $(-\mathcal{A}_0)^{-1/2}$ is self-adjoint, $\mathcal{B}$ is anti-self-adjoint. Hence $\mathcal{B}^*\mathcal{B}\geq 0$, so that $(I+\mathcal{B}^*\mathcal{B})^{-1}\leq I$ and $V(h)\leq V_0(h)$, with the equality holds if and only if $\mathcal{B}\tilde{g}=0$.

Since $\mathcal{A}_0\phi_0=\hat h$, we have $(-\mathcal{A}_0)^{-1}\hat h=-\phi_0$, so that
$\mathcal{B}\tilde{g}
= (-\mathcal{A}_0)^{-1/2}(\mathcal{A}^{\mathrm{as}}\phi_0)$.
Since $(-\mathcal{A}_0)^{-1/2}$ is injective on $H^0$ and $\mathcal{A}^{\mathrm{as}}\phi_0\in H^0$,
\[
\mathcal{B}\tilde{g}=0 \Longleftrightarrow \mathcal{A}^{\mathrm{as}}\phi_0=0 \Longleftrightarrow
\langle\mathcal{Q}\nabla H,\nabla\phi_0\rangle=0\quad\mu\text{-a.e.}
\]
Condition~\eqref{eq:strict-cond} asserts $\|\mathcal{A}^{\mathrm{as}}\phi_0\|_{L^2(\mu)}^2>0$, which excludes equality. Therefore $V(h)<V_0(h)$.
This completes the proof.
\end{proof}

\begin{remark}\label{rem:strict-vr-scope}
The additional assumptions in Theorem~\ref{thm:strict-vr} are standard operator-theoretic conditions from the continuous-time variance reduction literature. We do not attempt to verify them from Assumptions~\ref{ass:U} and~\ref{ass:diss} in full generality. Theorem~\ref{thm:clt} identifies the small-stepsize fluctuation constant, whereas Theorem~\ref{thm:strict-vr} compares its reversible and non-reversible values under these extra assumptions.
\end{remark}

\subsection{Bounded smooth predictive observables}\label{sec:predictive-observables}

Before turning to the proof of Theorem~\ref{thm:clt}, we record a class of
observables that are directly covered by Theorem~\ref{thm:clt} and are natural
in machine learning applications.

Let $\rho\in C_b^2(\mathbb R)$ and let $a\in\mathbb R^d$ be a fixed feature
vector associated with a test input. For an augmented state
$\mathbf z=(\theta,r)$, define
\[
h_{\rho,a}(\theta,r):=\rho(a^\top\theta).
\]
For a parameter-only state, the same definition is understood without the
$r$ variable.

\begin{proposition}[Bounded smooth predictive observables]\label{prop:predictive-observable}
The observable $h_{\rho,a}$ belongs to $C_b^2(\mathbb R^{d+m})$ in the augmented-state case, and to $C_b^2(\mathbb R^d)$ in the parameter-only case. Consequently, under the
assumptions of Theorem~\ref{thm:clt},
\[
\eta^{-1/2}
\left(
\widehat\mu_\eta^{\rm SG}(h_{\rho,a})-\mu(h_{\rho,a})
\right)
\Rightarrow
\mathcal N(0,V(h_{\rho,a})),
\]
as $\eta\rightarrow 0$, where
$V(h_{\rho,a})
=
\mu\left(\left\|\sigma^\top\nabla\phi_{\rho,a}\right\|^2\right)$,
and $\phi_{\rho,a}$ solves
$h_{\rho,a}-\mu(h_{\rho,a})=\mathcal A\phi_{\rho,a}$.

If, in addition, the assumptions of Theorem~\ref{thm:strict-vr} hold and the
strict condition there is satisfied with $h=h_{\rho,a}$, then the
non-reversible fluctuation constant is strictly smaller than the reversible
baseline.
\end{proposition}

\begin{proof}
Since $\rho\in C_b^2(\mathbb R)$, the function $h_{\rho,a}$ is bounded.
Moreover,
\[
\nabla_\theta h_{\rho,a}(\theta,r)
=
\rho'(a^\top\theta)a,
\qquad
\nabla_r h_{\rho,a}(\theta,r)=0,
\]
and
\[
\nabla_{\theta\theta}^2 h_{\rho,a}(\theta,r)
=
\rho''(a^\top\theta)aa^\top,
\qquad
\nabla_{\theta r}^2 h_{\rho,a}(\theta,r)=0,
\qquad
\nabla_{rr}^2 h_{\rho,a}(\theta,r)=0.
\]
Because $\rho'$, $\rho''$ are bounded and $a$ is fixed, all first and second
derivatives of $h_{\rho,a}$ are bounded. Hence
$h_{\rho,a}$ belongs to the asserted $C_b^2$ space. The CLT follows from Theorem~\ref{thm:clt}, and the
variance-reduction statement follows from Theorem~\ref{thm:strict-vr}.
\end{proof}

\begin{remark}[Logistic classification example]
Taking $h_{\rho,a}(\theta,r):=\rho(a^\top\theta)$ with 
$\rho(t)=\frac{1}{1+e^{-t}}$
gives the posterior predictive probability of the positive class in binary
logistic prediction at the test feature vector $a$; see, for example,
\cite{gelman1995bayesian,bishop2006pattern,murphy2012machine}.
\end{remark}

\begin{remark}[Other smooth predictive observables in machine learning]
The $C_b^2$ condition also covers other bounded smooth quantities used in
probabilistic machine learning, provided the underlying predictive probabilities
are sufficiently smooth functions of the parameter.

For a fixed labeled test point $(x,y)$ in a $K$-class problem, let
$p_c(\theta;x)$ denote the predictive probability of class $c$. The Brier-type
observable
$h_{\rm Br}(\theta)
=
\sum_{c=1}^K
\left(p_c(\theta;x)-\mathds 1_{\{y=c\}}\right)^2$
is bounded and belongs to $C_b^2$ whenever the functions
$p_c(\theta;x)$ belong to $C_b^2$ with bounded derivatives. Brier-type and
related proper scoring rules are standard tools for evaluating probabilistic
predictions and calibration; see, for example,
\cite{gneiting2007strictly,lakshminarayanan2017simple}.

Uncertainty scores used in Bayesian active learning can also be treated under
suitable smoothness conditions. For example, the conditional predictive entropy
$h_{\rm ent}(\theta)
=
-\sum_{c=1}^K p_c(\theta;x)\log p_c(\theta;x)$
is bounded by $\log K$. If the predictive probabilities are $C_b^2$ and are
bounded away from zero, then $h_{\rm ent}\in C_b^2$. Posterior averages of
entropy-type quantities appear in information-theoretic active learning criteria
such as BALD; see, for example, \cite{houlsby2011bayesian}.
\end{remark}

%%%%%%%%%%%%%%%%%%%%%%%%%%%%%%%%%%%%%%%%%%%%%%%%%%%%%%%%%%%%%%%%%%%%%%%%%%
\subsection{Proof of Theorem~\ref{thm:clt}}\label{sec:main:proof}
In this section, we present the proof of Theorem~\ref{thm:clt}.
The first step is to check that the limiting full-gradient drift and diffusion coefficients satisfy the structural hypotheses used in the Poisson-equation analysis of \cite{lu2022central}. The following proposition records these verifications.

\begin{proposition}[Verification of the structural hypotheses of \cite{lu2022central}]\label{prop:ltx}
Suppose Assumptions~\ref{ass:U} and~\ref{ass:diss} hold. Let
$g(\mathbf z):=-(\mathcal D+\mathcal Q)\nabla H(\mathbf z)$
and
$\mathcal D=\frac12\sigma\sigma^\top$,
where $\mathcal D$ is symmetric positive definite and $\mathcal Q$ is constant anti-symmetric. Then the limiting full-gradient drift $g$ and the constant diffusion coefficient $\sigma$ satisfy Assumption~2.1 of \cite{lu2022central}. More precisely:
\begin{enumerate}[label=(\roman*)]
\item $\sigma$ is constant and invertible;
\item $g$ is globally Lipschitz, with
$\|g(\mathbf z_1)-g(\mathbf z_2)\|
\le
(\|\mathcal D\|_{\mathrm{op}}+\|\mathcal Q\|_{\mathrm{op}})L
\|\mathbf z_1-\mathbf z_2\|$;
\item $g$ satisfies the one-sided dissipativity condition
$\langle g(\mathbf z_1)-g(\mathbf z_2),\mathbf z_1-\mathbf z_2\rangle
\le
-K_1\|\mathbf z_1-\mathbf z_2\|^2+K_2$;
\item $g$ is twice continuously differentiable and
$\sup_{\mathbf z\in\mathbb R^n}\left\|\nabla^2 g(\mathbf z)\right\|_{\mathrm{op}}
\le
(\|\mathcal D\|_{\mathrm{op}}+\|\mathcal Q\|_{\mathrm{op}})M$.
\end{enumerate}
Consequently, for every $h\in C_b^2(\mathbb R^n)$, the limiting full-gradient Poisson equation
$h-\mu(h)=\mathcal A\phi$
admits a solution $\phi$ satisfying
$\|\nabla^k\phi\|_\infty<\infty$,
$k=0,1,2,3,4$.
In particular, $\phi\in C_b^4(\mathbb R^n)$.
\end{proposition}

\begin{proof}
Since $\mathcal D$ is symmetric positive definite, the constant matrix $\sigma=(2\mathcal D)^{1/2}$ is well defined and invertible. This verifies the constant non-degenerate diffusion condition in Assumption~2.1 of \cite{lu2022central}.
Next, for any $\mathbf z_1,\mathbf z_2\in\mathbb R^n$, Assumption~\ref{ass:U}(i) gives
\[
\begin{aligned}
\|g(\mathbf z_1)-g(\mathbf z_2)\|
&=
\|-(\mathcal D+\mathcal Q)(\nabla H(\mathbf z_1)-\nabla H(\mathbf z_2))\|\\
&\le
\|\mathcal D+\mathcal Q\|_{\mathrm{op}}
\|\nabla H(\mathbf z_1)-\nabla H(\mathbf z_2)\|
\le
(\|\mathcal D\|_{\mathrm{op}}+\|\mathcal Q\|_{\mathrm{op}})L
\|\mathbf z_1-\mathbf z_2\|.
\end{aligned}
\]
Thus, $g$ is globally Lipschitz. The one-sided dissipativity condition required in Assumption~2.1 of \cite{lu2022central} is exactly Assumption~\ref{ass:diss}. Finally, since
$\nabla g(\mathbf z)=-(\mathcal D+\mathcal Q)\nabla^2H(\mathbf z)$
and
$\nabla^2g(\mathbf z)[u,v]=-(\mathcal D+\mathcal Q)\nabla^3H(\mathbf z)[u,v]$,
Assumption~\ref{ass:U}(i) implies
$\sup_{\mathbf z\in\mathbb R^n}\left\|\nabla^2g(\mathbf z)\right\|_{\mathrm{op}}
\le
(\|\mathcal D\|_{\mathrm{op}}+\|\mathcal Q\|_{\mathrm{op}})M$.
Therefore, the limiting full-gradient coefficients satisfy Assumption~2.1 of \cite{lu2022central}. The asserted Poisson regularity follows from \cite[Lemma~3.1]{lu2022central}.
\end{proof}

\begin{remark}[Use of the \cite{lu2022central} framework]\label{rem:ltx-use}
Proposition~\ref{prop:ltx} is used only to import the Poisson-equation regularity from \cite[Lemma~3.1]{lu2022central}. We do not apply the central limit theorem of \cite{lu2022central} directly, because the stochastic gradient chain
$\mathbf z_{k+1}^{\eta}
=
\mathbf z_k^{\eta}
+\eta g(\mathbf z_k^{\eta})
+\sqrt\eta\sigma\xi_{k+1}
+\eta\varepsilon_{k+1}$
is not the exact Euler--Maruyama scheme considered there. The additional term $\eta\varepsilon_{k+1}$ is handled separately through Assumption~\ref{ass:sg} and the martingale estimates below.
\end{remark}

First, we obtain the following $p$-th moment bound for SG-GNLMC that are uniform in time.

\begin{lemma}[Moment bound for SG-GNLMC]\label{lem:sg-moment}
Suppose Assumptions~\ref{ass:U}, \ref{ass:diss}, and~\ref{ass:sg} hold. Then there exist $\eta_0>0$ and $C>0$ such that, for every $\eta\in(0,\eta_0)$,
\[
\sup_{0\le k\le [\eta^{-2}]}
\mathbb E\left[1+\|\mathbf z_k^\eta\|^p\right]
\le C.
\]
In particular, the same bound holds for every lower moment $q\in[1,p]$.
\end{lemma}

\begin{proof}
We first record two elementary consequences of Assumption~\ref{ass:diss} and the global Lipschitz property of $g$ verified in Proposition~\ref{prop:ltx}. Taking $\mathbf z_2=0$ in Assumption~\ref{ass:diss} and using Young's inequality gives constants $c_0,C_0>0$ such that
\[
\langle \mathbf z,g(\mathbf z)\rangle
\le -c_0\|\mathbf z\|^2+C_0,
\qquad \mathbf z\in\mathbb R^n.
\]
Moreover, since $g$ is globally Lipschitz,
$\|g(\mathbf z)\|\le C_0(1+\|\mathbf z\|)$.
Using
\[
\mathbf z_{k+1}^\eta
=
\mathbf z_k^\eta
+\eta g(\mathbf z_k^\eta)
+\sqrt{\eta}\sigma\xi_{k+1}
+\eta\varepsilon_{k+1},
\]
the conditional unbiasedness of $\varepsilon_{k+1}$, the conditional independence of $\varepsilon_{k+1}$ and $\xi_{k+1}$, and the centering of $\xi_{k+1}$ imply
\[
\mathbb E\left[
\|\mathbf z_{k+1}^\eta\|^2
\mid\mathcal F_k
\right]
\le
\|\mathbf z_k^\eta\|^2
+2\eta\langle \mathbf z_k^\eta,g(\mathbf z_k^\eta)\rangle
+C\eta(1+\|\mathbf z_k^\eta\|^2)
+C\eta^2\mathbb E\left[\|\varepsilon_{k+1}\|^2\mid\mathcal F_k\right].
\]
Using Assumption~\ref{ass:sg} and choosing $\eta_0$ sufficiently small, the last display yields
\[
\mathbb E\left[
1+\|\mathbf z_{k+1}^\eta\|^2
\mid\mathcal F_k
\right]
\le
(1-c\eta)\left(1+\|\mathbf z_k^\eta\|^2\right)+C\eta.
\]
We use the following Taylor-type inequality for the function
$V(x)=\|x\|^p$, valid for every $p>2$:
\[
\|x+y\|^p
\le
\|x\|^p
+p\|x\|^{p-2}\langle x,y\rangle
+C_p\left(\|x\|^{p-2}\|y\|^2+\|y\|^p\right).
\]
This follows from Taylor's formula and the bound
$\|\nabla^2 V(z)\|_{\mathrm{op}}\le C_p\|z\|^{p-2}$.
Apply it conditionally with $x=\mathbf z_k^\eta$ and
\[
y=\eta g(\mathbf z_k^\eta)+\sqrt\eta\sigma\xi_{k+1}+\eta\varepsilon_{k+1}.
\]
The first-order term gives the negative drift
\[
p\eta\|\mathbf z_k^\eta\|^{p-2}
\langle \mathbf z_k^\eta,g(\mathbf z_k^\eta)\rangle
\le
-c\eta\|\mathbf z_k^\eta\|^p+C\eta.
\]
All centered linear terms involving $\xi_{k+1}$ or $\varepsilon_{k+1}$ vanish conditionally. The remaining terms are bounded by
\[
C\eta\left(1+\|\mathbf z_k^\eta\|^p\right)
+C\eta^{p/2}\mathbb E\|\xi_{k+1}\|^p
+C\eta^p\mathbb E\left[\|\varepsilon_{k+1}\|^p\mid\mathcal F_k\right],
\]
which is at most $C\eta(1+\|\mathbf z_k^\eta\|^p)$ for $\eta\in(0,1)$. Absorbing the lower-order contribution into the negative drift by choosing $\eta_0$ small enough gives
\[
\mathbb E\left[
1+\|\mathbf z_{k+1}^\eta\|^p
\mid\mathcal F_k
\right]
\le
(1-c\eta)\left(1+\|\mathbf z_k^\eta\|^p\right)+C\eta.
\]
Iterating the last inequality and using Assumption~\ref{ass:sg} for the initial condition yields the result. The lower-moment bound follows from Jensen's inequality.
\end{proof}

The martingale central limit theorem to be used below requires convergence of the conditional quadratic variation. We therefore first establish a Poisson-equation law of large numbers (LLN) along the stochastic gradient chain.
Before proceeding, we first introduce notations for later use.
Recall that the stochastic gradient chain can be written as
$$
\mathbf{z}_{k+1}^\eta=\mathbf{z}_k^\eta+\eta g\left(\mathbf{z}_k^\eta\right)+\sqrt{\eta} \sigma \xi_{k+1}+\eta \varepsilon_{k+1} .
$$
We define
$$
a_k:=\eta g\left(\mathbf{z}_k^\eta\right), \quad b_{k+1}:=\sqrt{\eta} \sigma \xi_{k+1}, \quad e_{k+1}:=\eta \varepsilon_{k+1} .
$$
Then the exact-gradient Euler increment and the stochastic gradient Euler increment are 
$$
\Delta_k^0:=a_k+b_{k+1}, \quad \Delta_k^\eta:=a_k+b_{k+1}+e_{k+1}=\Delta_k^0+e_{k+1} .
$$
For a test function $\psi\in C_b^4(\mathbb R^n)$, define the third-order
Taylor remainders
\[
T_{k,3}^0
:=
\frac12\int_0^1(1-t)^2
\nabla^3\psi(\mathbf z_k^\eta+t\Delta_k^0)
[\Delta_k^0,\Delta_k^0,\Delta_k^0]\,dt,
\]
and
\[
T_{k,3}^\eta
:=
\frac12\int_0^1(1-t)^2
\nabla^3\psi(\mathbf z_k^\eta+t\Delta_k^\eta)
[\Delta_k^\eta,\Delta_k^\eta,\Delta_k^\eta]\,dt.
\]
We decompose the Taylor remainder into an exact-gradient part and a stochastic gradient part. The exact-gradient Taylor remainder is defined by
\begin{align}\label{exact:gradient:remainder}
\rho_{k,\eta}^0
:=
\frac{\eta}{2}
\operatorname{Tr}\left(\sigma\sigma^\top\nabla^2\psi(\mathbf z_k^\eta)\right)
-\frac12
(\Delta_k^0)^\top\nabla^2\psi(\mathbf z_k^\eta)\Delta_k^0
-T_{k,3}^0 .
\end{align}
This is the same remainder as in the exact-gradient Euler scheme in~\cite{lu2022central}, but evaluated along the stochastic gradient path $\mathbf{z}_k^\eta$. Moreover, the stochastic gradient Taylor remainder is defined by
\begin{align}\label{stochastic:gradient:remainder}
\rho_{k,\eta}^{\varepsilon}
:=
-\frac12
\left[
(\Delta_k^\eta)^\top\nabla^2\psi(\mathbf z_k^\eta)\Delta_k^\eta
-
(\Delta_k^0)^\top\nabla^2\psi(\mathbf z_k^\eta)\Delta_k^0
\right]
-\left(T_{k,3}^\eta-T_{k,3}^0\right).
\end{align}

\begin{lemma}[Poisson LLN for the SG chain]\label{lem:sg-poisson-lln}
Suppose Assumptions~\ref{ass:U}, \ref{ass:diss}, and~\ref{ass:sg} hold. Let $f\in C_b^2(\mathbb R^n)$, and suppose that the Poisson equation
$f-\mu(f)=\mathcal A\psi$
admits a solution $\psi\in C_b^4(\mathbb R^n)$. Then, as $\eta\rightarrow 0$,
\[
\eta^2\sum_{k=0}^{[\eta^{-2}]-1}f(\mathbf z_k^\eta)
\to
\mu(f)
\qquad\text{in probability}.
\]
\end{lemma}

\begin{proof}
It suffices to prove that as $\eta\rightarrow 0$,
$\eta^2\sum_{k=0}^{[\eta^{-2}]-1}(f(\mathbf z_k^\eta)-\mu(f))
\to0$
in probability.

Using $f-\mu(f)=\mathcal A\psi$ and the telescoping identity,
\[
\eta^2\sum_{k=0}^{[\eta^{-2}]-1}\mathcal A\psi(\mathbf z_k^\eta)
=
\eta\left(\psi(\mathbf z_{[\eta^{-2}]}^\eta)-\psi(\mathbf z_0^\eta)\right)
+\eta\sum_{k=0}^{[\eta^{-2}]-1}
\left[
\eta\mathcal A\psi(\mathbf z_k^\eta)
-\left(\psi(\mathbf z_{k+1}^\eta)-\psi(\mathbf z_k^\eta)\right)
\right].
\]
The boundary term converges to zero in $L^1$ because $\psi$ is bounded. A Taylor expansion gives
\[
\eta\mathcal A\psi(\mathbf z_k^\eta)
-\left(\psi(\mathbf z_{k+1}^\eta)-\psi(\mathbf z_k^\eta)\right)
=
-\sqrt\eta\left\langle\nabla\psi(\mathbf z_k^\eta),\sigma\xi_{k+1}\right\rangle
-\eta\left\langle\nabla\psi(\mathbf z_k^\eta),\varepsilon_{k+1}\right\rangle
+\widetilde\rho_{k,\eta},
\]
where the Taylor remainder $\widetilde\rho_{k,\eta} = \rho_{k,\eta}^0 + \rho_{k,\eta}^{\varepsilon}$, with $\rho_{k,\eta}^{0}$ given in \eqref{exact:gradient:remainder} and $\rho_{k,\eta}^{\varepsilon}$ given in \eqref{stochastic:gradient:remainder}, is decomposed into the stochastic gradient terms controlled by Lemma~\ref{lem:sg-negligible} and the exact-gradient terms controlled by Lemma~\ref{lem:exact-remainder-sg-path}. 
Multiplying the first term on the right-hand side of the last display by the
outer factor $\eta$ in the telescoping identity and summing over $k$, we obtain
the Brownian martingale contribution
\[
M_\eta^B
:=
-\eta^{3/2}
\sum_{k=0}^{[\eta^{-2}]-1}
\left\langle
\nabla\psi(\mathbf z_k^\eta),\sigma\xi_{k+1}
\right\rangle .
\]
Since $\nabla\psi$ is bounded and $\xi_{k+1}$ is centered and independent of
$\mathcal F_k$, this is a martingale-difference sum. Hence
\[
\mathbb E|M_\eta^B|^2
\le
C\eta^3
\sum_{k=0}^{[\eta^{-2}]-1}1
\le
C\eta\to0.
\]
Similarly, multiplying the second term on the right-hand side of the last
display by the outer factor $\eta$ and summing over $k$, we obtain the
stochastic gradient martingale contribution
\[
M_\eta^{\rm SG}
:=
-\eta^2
\sum_{k=0}^{[\eta^{-2}]-1}
\left\langle
\nabla\psi(\mathbf z_k^\eta),\varepsilon_{k+1}
\right\rangle .
\]
Since $\nabla\psi$ is bounded and
$\mathbb E[\varepsilon_{k+1}\mid\mathcal F_k]=0$, this is also a
martingale-difference sum. Therefore, by Assumption~\ref{ass:sg} and
Lemma~\ref{lem:sg-moment},
\[
\mathbb E|M_\eta^{\rm SG}|^2
\le
C\eta^4
\sum_{k=0}^{[\eta^{-2}]-1}
\mathbb E\left[1+\|\mathbf z_k^\eta\|^p\right]
\le
C\eta^2\to0.
\]
The Taylor terms containing at least one factor $\eta\varepsilon_{k+1}$ are controlled by Lemma~\ref{lem:sg-negligible}; in the present LLN estimate the outer factor is $\eta$ rather than $\sqrt\eta$, so the bounds there are even smaller. The remaining exact-gradient Taylor terms are controlled by the same estimates as in Lemma~\ref{lem:exact-remainder-sg-path}. Therefore the whole centered average converges to zero in probability.
\end{proof}

The next lemma isolates the Taylor terms that are new relative to the exact Euler--Maruyama decomposition, namely those containing at least one stochastic gradient increment $\eta\varepsilon_{k+1}$, and shows that they are negligible under the CLT normalization.

\begin{lemma}[Negligibility of the stochastic gradient Taylor terms]\label{lem:sg-negligible}
Suppose Assumptions~\ref{ass:U}, \ref{ass:diss}, and~\ref{ass:sg} hold, and let $\psi\in C_b^4(\mathbb R^n)$. Write
\[
a_k:=\eta g(\mathbf z_k^\eta),
\qquad
b_{k+1}:=\sqrt\eta\sigma\xi_{k+1},
\qquad
e_{k+1}:=\eta\varepsilon_{k+1},
\qquad
\Delta_k^\eta:=a_k+b_{k+1}+e_{k+1}.
\]
Then, under the CLT normalization, all Taylor terms in
$\psi(\mathbf z_{k+1}^\eta)-\psi(\mathbf z_k^\eta)$
that contain at least one factor $e_{k+1}$ converge to zero in probability. More explicitly,
\[
\eta^{3/2}
\sum_{k=0}^{[\eta^{-2}]-1}
\left\langle
\nabla\psi(\mathbf z_k^\eta),\varepsilon_{k+1}
\right\rangle
\to0
\qquad\text{in $L^2$ as $\eta\rightarrow 0$},
\]
and the quadratic and cubic stochastic gradient Taylor contributions converge to zero in probability.
\end{lemma}

\begin{proof}
Throughout the proof, $C$ denotes a finite constant independent of $\eta$ and $k$. Since $\psi\in C_b^4$, all derivatives of $\psi$ up to order four are bounded. By the conditional $L^p$-monotonicity inequality, namely, for any random vector $X$,
$\left(\mathbb E[\|X\|^q\mid\mathcal F]\right)^{1/q}
\le
\left(\mathbb E[\|X\|^p\mid\mathcal F]\right)^{1/p}$,
$1\le q\le p$,
and by Assumption~\ref{ass:sg}, for every $q\in[1,p]$,
\[
\mathbb E\left[\|\varepsilon_{k+1}\|^q\mid\mathcal F_k\right]
\le
\left(
\mathbb E\left[\|\varepsilon_{k+1}\|^p\mid\mathcal F_k\right]
\right)^{q/p}
\le
C\left(1+\|\mathbf z_k^\eta\|^q\right).
\]
Therefore, using $e_{k+1}=\eta\varepsilon_{k+1}$ and the Gaussian moment bounds,
\begin{equation}\label{eq:basic-increment-moments}
\mathbb E\left[\|e_{k+1}\|^q\mid\mathcal F_k\right]
\le C\eta^q\left(1+\|\mathbf z_k^\eta\|^q\right),
\qquad
\mathbb E\left[\|b_{k+1}\|^q\mid\mathcal F_k\right]
\le C\eta^{q/2}.
\end{equation}
Moreover,
\begin{equation}\label{eq:a-increment-bound}
\|a_k\|
\le C\eta\left(1+\|\mathbf z_k^\eta\|\right).
\end{equation}
Consequently, for all nonnegative integers $r,s,t$ with $r+s+t\le4$ and $t\ge1$,
\begin{equation}\label{eq:generic-monomial-bound}
\mathbb E\left[
\|a_k\|^r\|b_{k+1}\|^s\|e_{k+1}\|^t
\right]
\le
C\eta^{r+s/2+t}.
\end{equation}
Indeed, after using \eqref{eq:basic-increment-moments} and \eqref{eq:a-increment-bound}, the polynomial factor in $\mathbf z_k^\eta$ has degree at most $r+t\le4<p$, and is therefore uniformly controlled by Lemma~\ref{lem:sg-moment}.

\smallskip
\noindent\textbf{Step 1: the linear stochastic gradient term.}
Since $\nabla\psi$ is bounded and
$
\mathbb E[\varepsilon_{k+1}\mid\mathcal F_k]=0
$, the variables
$
\langle\nabla\psi(\mathbf z_k^\eta),\varepsilon_{k+1}\rangle
$
form martingale differences. Hence, using Assumption~\ref{ass:sg} and Lemma~\ref{lem:sg-moment},
\[
\begin{aligned}
\mathbb E\left|
\eta^{3/2}\sum_{k=0}^{[\eta^{-2}]-1}
\left\langle\nabla\psi(\mathbf z_k^\eta),\varepsilon_{k+1}\right\rangle
\right|^2
\le
C\eta^3\sum_{k=0}^{[\eta^{-2}]-1}
\mathbb E\left[1+\|\mathbf z_k^\eta\|^p\right]
\le C\eta^3[\eta^{-2}]
\le C\eta
\to0,
\end{aligned}
\]
as $\eta\rightarrow 0$.

\smallskip
\noindent\textbf{Step 2: the quadratic stochastic gradient terms.}
The quadratic Taylor contribution containing at least one $e_{k+1}$ is, up to a bounded deterministic factor,
\begin{equation}\label{eq:Q-SG}
Q_\eta^{\rm SG}
=
\sqrt\eta\sum_{k=0}^{[\eta^{-2}]-1}
\left\{
\nabla^2\psi(\mathbf z_k^\eta)[a_k,e_{k+1}]
+
\nabla^2\psi(\mathbf z_k^\eta)[b_{k+1},e_{k+1}]
+
\nabla^2\psi(\mathbf z_k^\eta)[e_{k+1},e_{k+1}]
\right\}.
\end{equation}
For the first term in \eqref{eq:Q-SG},
corresponding to the $(a,e)$
contribution, boundedness of $\nabla^2\psi$ gives
\[
\left|
\nabla^2\psi(\mathbf z_k^\eta)[a_k,e_{k+1}]
\right|
\le
C\|a_k\|\|e_{k+1}\|.
\]
Applying \eqref{eq:generic-monomial-bound} with
$(r,s,t)=(1,0,1)$ yields
$\mathbb E\left[\|a_k\|\|e_{k+1}\|\right]\le C\eta^2$.
Therefore,
\[
\begin{aligned}
\mathbb E\left|
\sqrt\eta\sum_{k=0}^{[\eta^{-2}]-1}
\nabla^2\psi(\mathbf z_k^\eta)[a_k,e_{k+1}]
\right|
\le
C\sqrt\eta
\sum_{k=0}^{[\eta^{-2}]-1}
\mathbb E\left[\|a_k\|\|e_{k+1}\|\right] \le
C\sqrt\eta[\eta^{-2}]\eta^2
=
C\eta^{1/2}\to0,
\end{aligned}
\]
as $\eta\rightarrow 0$.
Similarly, for the third term in \eqref{eq:Q-SG}, corresponding to the
$(e,e)$ contribution, boundedness of $\nabla^2\psi$ and
\eqref{eq:generic-monomial-bound} with $(r,s,t)=(0,0,2)$ give
$\mathbb E\left[\|e_{k+1}\|^2\right]\le C\eta^2$,
and hence
\[
\begin{aligned}
\mathbb E\left|
\sqrt\eta\sum_{k=0}^{[\eta^{-2}]-1}
\nabla^2\psi(\mathbf z_k^\eta)[e_{k+1},e_{k+1}]
\right|
\le
C\sqrt\eta
\sum_{k=0}^{[\eta^{-2}]-1}
\mathbb E\left[\|e_{k+1}\|^2\right]
\le
C\sqrt\eta[\eta^{-2}]\eta^2
\le
C\eta^{1/2}\to0,
\end{aligned}
\]
as $\eta\rightarrow 0$.
The mixed $(b,e)$ term in \eqref{eq:Q-SG} is treated as a martingale-difference sum. Indeed, conditional independence of $\varepsilon_{k+1}$ and $\xi_{k+1}$, together with their conditional centering, gives
\[
\mathbb E\left[
\nabla^2\psi(\mathbf z_k^\eta)[b_{k+1},e_{k+1}]
\mid\mathcal F_k
\right]=0.
\]
Moreover, by \eqref{eq:basic-increment-moments},
\[
\begin{aligned}
\mathbb E\left|
\sqrt\eta\sum_{k=0}^{[\eta^{-2}]-1}
\nabla^2\psi(\mathbf z_k^\eta)[b_{k+1},e_{k+1}]
\right|^2
&\le
C\eta\sum_{k=0}^{[\eta^{-2}]-1}
\mathbb E\left[\|b_{k+1}\|^2\|e_{k+1}\|^2\right]\\
&\le
C\eta\sum_{k=0}^{[\eta^{-2}]-1}\eta^3
=
C\eta^4[\eta^{-2}]
\le C\eta^2\to0,
\end{aligned}
\]
as $\eta\rightarrow 0$.
Therefore $Q_\eta^{\rm SG}\to0$ in probability.

\smallskip
\noindent\textbf{Step 3: the cubic stochastic gradient terms.}
Using Taylor's formula with third-order integral remainder, the cubic
stochastic gradient contribution containing at least one $e_{k+1}$ is bounded
in absolute value by
\begin{equation}\label{eq:C-SG-bound}
\begin{aligned}
&C\sqrt\eta\sum_{k=0}^{[\eta^{-2}]-1}
\Big(
\|b_{k+1}\|^2\|e_{k+1}\|
+\|b_{k+1}\|\|e_{k+1}\|^2
+\|e_{k+1}\|^3\\
&\qquad\qquad\qquad\qquad\qquad
+\|a_k\|\|b_{k+1}\|\|e_{k+1}\|
+\|a_k\|\|e_{k+1}\|^2
+\|a_k\|^2\|e_{k+1}\|
\Big).
\end{aligned}
\end{equation}
The six monomials in \eqref{eq:C-SG-bound} are controlled by
\eqref{eq:generic-monomial-bound}.
The following table records the resulting $L^1$ bounds after summing over $[\eta^{-2}]$ steps and multiplying by the outer factor $\sqrt\eta$:
\[
\begin{array}{c|c|c}
\text{monomial} &
\text{one-step size} &
\text{normalized summed bound} \\
\hline
\|b\|^2\|e\| & \eta^2 & C\sqrt\eta[\eta^{-2}]\eta^2=C\eta^{1/2} \\
\|b\|\|e\|^2 & \eta^{5/2} & C\sqrt\eta[\eta^{-2}]\eta^{5/2}=C\eta \\
\|e\|^3 & \eta^3 & C\sqrt\eta[\eta^{-2}]\eta^3=C\eta^{3/2} \\
\|a\|\|b\|\|e\| & \eta^{5/2} & C\sqrt\eta[\eta^{-2}]\eta^{5/2}=C\eta \\
\|a\|\|e\|^2 & \eta^3 & C\sqrt\eta[\eta^{-2}]\eta^3=C\eta^{3/2} \\
\|a\|^2\|e\| & \eta^3 & C\sqrt\eta[\eta^{-2}]\eta^3=C\eta^{3/2}
\end{array}
\]
All terms in the last column converge to zero. Hence, the full cubic stochastic gradient contribution converges to zero in $L^1$, and therefore in probability.

Combining Steps 1--3 proves the lemma.
\end{proof}

It remains to control the exact-gradient Taylor remainders evaluated along the stochastic gradient path. These are the same types of terms as in the exact Euler--Maruyama decomposition, but the estimates must be carried out using the moment bounds for the SG chain.

\begin{lemma}[Exact-gradient Taylor remainder along the SG path]\label{lem:exact-remainder-sg-path}
Suppose Assumptions~\ref{ass:U}, \ref{ass:diss}, and~\ref{ass:sg} hold, and let $\psi\in C_b^4(\mathbb R^n)$. After removing the Brownian martingale term
\[
-\eta\sum_{k=0}^{[\eta^{-2}]-1}
\langle\nabla\psi(\mathbf z_k^\eta),\sigma\xi_{k+1}\rangle,
\]
all remaining Taylor terms in the expansion of $\psi(\mathbf z_{k+1}^\eta)-\psi(\mathbf z_k^\eta)$ that do not contain the stochastic gradient increment $\eta\varepsilon_{k+1}$ converge to zero in probability under the CLT normalization.
\end{lemma}

\begin{proof}
From the SG-GNLMC update,
$\mathbf z_{k+1}^{\eta}
=
\mathbf z_k^{\eta}
+\eta g(\mathbf z_k^{\eta})
+\sqrt{\eta}\sigma\xi_{k+1}
+\eta\varepsilon_{k+1}$,
we have
\[
\Delta_k^\eta
:=
\mathbf z_{k+1}^\eta-\mathbf z_k^\eta
=
\eta g(\mathbf z_k^\eta)
+
\sqrt\eta\sigma\xi_{k+1}
+
\eta\varepsilon_{k+1}.
\]
For notational convenience, set
\[
a_k:=\eta g(\mathbf z_k^\eta),
\qquad
b_{k+1}:=\sqrt\eta\sigma\xi_{k+1},
\qquad
e_{k+1}:=\eta\varepsilon_{k+1}.
\]
Thus
\[
\Delta_k^\eta=a_k+b_{k+1}+e_{k+1}.
\]

We use Taylor's formula with second-order expansion and third-order integral
remainder:
\begin{equation}\label{eq:taylor-exact-remainder}
\begin{aligned}
\psi(\mathbf z_{k+1}^\eta)-\psi(\mathbf z_k^\eta)
&=
\nabla\psi(\mathbf z_k^\eta)[\Delta_k^\eta]
+
\frac12\nabla^2\psi(\mathbf z_k^\eta)[\Delta_k^\eta,\Delta_k^\eta]  \\
&\quad
+
\frac12\int_0^1(1-t)^2
\nabla^3\psi(\mathbf z_k^\eta+t\Delta_k^\eta)
[\Delta_k^\eta,\Delta_k^\eta,\Delta_k^\eta]dt .
\end{aligned}
\end{equation}
Since $\Delta_k^\eta=a_k+b_{k+1}+e_{k+1}$, the first-order Taylor term in
\eqref{eq:taylor-exact-remainder} decomposes as
\begin{equation}\label{eq:first-order-decomp}
\nabla\psi(\mathbf z_k^\eta)[\Delta_k^\eta]
=
\nabla\psi(\mathbf z_k^\eta)[a_k]
+
\nabla\psi(\mathbf z_k^\eta)[b_{k+1}]
+
\nabla\psi(\mathbf z_k^\eta)[e_{k+1}].
\end{equation}
Similarly, the second-order Taylor term in \eqref{eq:taylor-exact-remainder}
decomposes as
\begin{equation}\label{eq:second-order-decomp}
\begin{aligned}
\frac12\nabla^2\psi(\mathbf z_k^\eta)[\Delta_k^\eta,\Delta_k^\eta]
&=
\frac12\nabla^2\psi(\mathbf z_k^\eta)[a_k+b_{k+1},a_k+b_{k+1}]\\
&\quad
+\nabla^2\psi(\mathbf z_k^\eta)[a_k+b_{k+1},e_{k+1}]
+\frac12\nabla^2\psi(\mathbf z_k^\eta)[e_{k+1},e_{k+1}].
\end{aligned}
\end{equation}
We control only the Taylor terms whose multilinear factors do not contain
\(e_{k+1}\). The terms containing at least one factor \(e_{k+1}\) are handled
separately in Lemma~\ref{lem:sg-negligible}. The integral remainders below may
still be evaluated at \(\mathbf z_k^\eta+t\Delta_k^\eta\), and this dependence
on the full increment is controlled by the moment estimates below.

Consider the CLT-normalized local error
$\sqrt\eta
\left[
\eta\mathcal A\psi(\mathbf z_k^\eta)
-
\left(\psi(\mathbf z_{k+1}^\eta)-\psi(\mathbf z_k^\eta)\right)
\right]$.
Since
$\mathcal A\psi(\mathbf z)
=
\left\langle g(\mathbf z),\nabla\psi(\mathbf z)\right\rangle
+
\operatorname{Tr}\left(\mathcal D\nabla^2\psi(\mathbf z)\right)$,
we have
\begin{equation}\label{eq:generator-local}
\eta\mathcal A\psi(\mathbf z_k^\eta)
=
\eta
\left\langle g(\mathbf z_k^\eta),\nabla\psi(\mathbf z_k^\eta)\right\rangle
+
\eta
\operatorname{Tr}\left(\mathcal D\nabla^2\psi(\mathbf z_k^\eta)\right).
\end{equation}
The drift term in \eqref{eq:generator-local}, namely
$\eta\left\langle g(\mathbf z_k^\eta),
\nabla\psi(\mathbf z_k^\eta)\right\rangle$,
cancels the first term on the right-hand side of
\eqref{eq:first-order-decomp},
because
$\nabla\psi(\mathbf z_k^\eta)[a_k]
=
\eta
\left\langle g(\mathbf z_k^\eta),
\nabla\psi(\mathbf z_k^\eta)\right\rangle$.
The second term on the right-hand side of \eqref{eq:first-order-decomp}
is the
first-order Brownian Taylor term,
$\nabla\psi(\mathbf z_k^\eta)[b_{k+1}]
=
\sqrt\eta
\left\langle\nabla\psi(\mathbf z_k^\eta),\sigma\xi_{k+1}\right\rangle$.
After multiplying by the outer factor $\sqrt\eta$ and taking the minus sign
from the local error, it gives the Brownian martingale term
$-\eta
\left\langle\nabla\psi(\mathbf z_k^\eta),\sigma\xi_{k+1}\right\rangle$.
This term is removed in the statement of the lemma.

We now list the remaining exact-gradient Taylor terms, namely the terms in
\eqref{eq:taylor-exact-remainder} 
that do not contain the stochastic gradient
increment $e_{k+1}$. By \eqref{eq:second-order-decomp},
the second-order
Taylor contribution without $e_{k+1}$ is generated by
\[
\frac12\nabla^2\psi(\mathbf z_k^\eta)
[a_k+b_{k+1},a_k+b_{k+1}]
=
\frac12\nabla^2\psi(\mathbf z_k^\eta)[b_{k+1},b_{k+1}]
+
\nabla^2\psi(\mathbf z_k^\eta)[a_k,b_{k+1}]
+
\frac12\nabla^2\psi(\mathbf z_k^\eta)[a_k,a_k].
\]

The diffusion part of \(\eta\mathcal A\psi\) is
$\eta\operatorname{Tr}\left(\mathcal D\nabla^2\psi(\mathbf z_k^\eta)\right)
=
\frac{\eta}{2}\operatorname{Tr}\left(\sigma\sigma^\top\nabla^2\psi(\mathbf z_k^\eta)\right)$,
because \(\mathcal D=\frac12\sigma\sigma^\top\). On the other hand,
\[
\frac12\nabla^2\psi(\mathbf z_k^\eta)[b_{k+1},b_{k+1}]
=
\frac{\eta}{2}\xi_{k+1}^{\top}\sigma^\top\nabla^2\psi(\mathbf z_k^\eta)\sigma\xi_{k+1}.
\]
Therefore, in
$\eta\mathcal A\psi(\mathbf z_k^\eta)
-
\left(\psi(\mathbf z_{k+1}^\eta)-\psi(\mathbf z_k^\eta)\right)$,
the diffusion term and the Brownian--Brownian quadratic Taylor term combine as
\[
-\frac{\eta}{2}
\left[
\xi_{k+1}^{\top}
\sigma^\top\nabla^2\psi(\mathbf z_k^\eta)\sigma
\xi_{k+1}
-
\operatorname{Tr}\left(\sigma\sigma^\top\nabla^2\psi(\mathbf z_k^\eta)\right)
\right].
\]
After multiplying by the outer factor \(\sqrt\eta\), this gives the centered
quadratic Brownian term, up to the harmless deterministic factor \(-1/2\).

Thus, after removing the Brownian martingale term, the exact-gradient Taylor
terms are, up to harmless deterministic constants,
\[
R_\eta^{\rm ex}
=
B_{\eta,2}^{\rm cen}
+
B_{\eta}^{gB}
+
B_{\eta}^{gg}
+
C_{\eta}^{B}
+
C_{\eta}^{gBB}
+
C_{\eta}^{ggB}
+
C_{\eta}^{ggg},
\]
where
\begin{align}
B_{\eta,2}^{\rm cen}
&:=
\eta^{3/2}
\sum_{k=0}^{[\eta^{-2}]-1}
\left[
\xi_{k+1}^{\top}
\sigma^\top\nabla^2\psi(\mathbf z_k^\eta)\sigma
\xi_{k+1}
-
\operatorname{Tr}\left(\sigma\sigma^\top\nabla^2\psi(\mathbf z_k^\eta)\right)
\right]. \label{eq:B2-cen}\\
B_{\eta}^{gB}
&:=
\eta^2
\sum_{k=0}^{[\eta^{-2}]-1}
\nabla^2\psi(\mathbf z_k^\eta)
\left[g(\mathbf z_k^\eta),\sigma\xi_{k+1}\right], \label{eq:B-gB}\\
B_{\eta}^{gg}
&:=
\eta^{5/2}
\sum_{k=0}^{[\eta^{-2}]-1}
\nabla^2\psi(\mathbf z_k^\eta)
\left[g(\mathbf z_k^\eta),g(\mathbf z_k^\eta)\right], \label{eq:B-gg}\\
C_{\eta}^{B}
&:=
\eta^2
\sum_{k=0}^{[\eta^{-2}]-1}
\int_0^1
\nabla^3\psi(\mathbf z_k^\eta+t\Delta_k^\eta)
[\sigma\xi_{k+1},\sigma\xi_{k+1},\sigma\xi_{k+1}]dt, \label{eq:C-B-full}\\
C_{\eta}^{gBB}
&:=
\eta^{5/2}
\sum_{k=0}^{[\eta^{-2}]-1}
\int_0^1
\nabla^3\psi(\mathbf z_k^\eta+t\Delta_k^\eta)
[g(\mathbf z_k^\eta),\sigma\xi_{k+1},\sigma\xi_{k+1}]dt, \label{eq:C-gBB}\\
C_{\eta}^{ggB}
&:=
\eta^3
\sum_{k=0}^{[\eta^{-2}]-1}
\int_0^1
\nabla^3\psi(\mathbf z_k^\eta+t\Delta_k^\eta)
[g(\mathbf z_k^\eta),g(\mathbf z_k^\eta),\sigma\xi_{k+1}]dt, \label{eq:C-ggB}\\
C_{\eta}^{ggg}
&:=
\eta^{7/2}
\sum_{k=0}^{[\eta^{-2}]-1}
\int_0^1
\nabla^3\psi(\mathbf z_k^\eta+t\Delta_k^\eta)
[g(\mathbf z_k^\eta),g(\mathbf z_k^\eta),g(\mathbf z_k^\eta)]dt. \label{eq:C-ggg}
\end{align}
The constants, signs, and the bounded weight \(\frac12(1-t)^2\) from Taylor's
formula are irrelevant for convergence to zero and are absorbed into the
generic constant \(C\).

We now prove that each term above converges to zero in probability. Since
\(\nabla^2\psi\) is bounded and the fourth Gaussian moment is finite,
\(B_{\eta,2}^{\rm cen}\) is a martingale-difference sum and
\[
\mathbb E\left|B_{\eta,2}^{\rm cen}\right|^2
\le
C\eta^3
\sum_{k=0}^{[\eta^{-2}]-1}1
\le
C\eta^3[\eta^{-2}]
\le
C\eta
\to0.
\]
Hence \(B_{\eta,2}^{\rm cen}\to0\) in probability as $\eta\rightarrow 0$.

Next, since \(g\) has at most linear growth and \(\nabla^2\psi\) is bounded,
\(B_{\eta}^{gB}\) is a martingale-difference sum and
\[
\mathbb E\left|B_{\eta}^{gB}\right|^2
\le
C\eta^4
\sum_{k=0}^{[\eta^{-2}]-1}
\mathbb E\left[1+\|\mathbf z_k^\eta\|^2\right]
\le
C\eta^4[\eta^{-2}]
\le
C\eta^2
\to0,
\]
as $\eta\rightarrow 0$,
where we used Lemma~\ref{lem:sg-moment}. Hence \(B_{\eta}^{gB}\to0\) in
probability.

For \(B_{\eta}^{gg}\), using the linear growth of \(g\), boundedness of
\(\nabla^2\psi\), and Lemma~\ref{lem:sg-moment},
\[
\mathbb E\left|B_{\eta}^{gg}\right|
\le
C\eta^{5/2}
\sum_{k=0}^{[\eta^{-2}]-1}
\mathbb E\left[1+\|\mathbf z_k^\eta\|^2\right]
\le
C\eta^{5/2}[\eta^{-2}]
\le
C\eta^{1/2}
\to0,
\]
as $\eta\rightarrow 0$.
Thus, \(B_{\eta}^{gg}\to0\) in probability $\eta\rightarrow 0$.

We now estimate the cubic terms
\eqref{eq:C-gBB}--\eqref{eq:C-ggg}.
The terms containing at least one drift
factor are directly controlled in \(L^1\). Since \(\nabla^3\psi\) is bounded,
\(g\) has at most linear growth, Gaussian moments are finite, and
Lemma~\ref{lem:sg-moment} holds,
\begin{align*}
&\mathbb E|C_{\eta}^{gBB}|
\le
C\eta^{5/2}
\sum_{k=0}^{[\eta^{-2}]-1}
\mathbb E\left[(1+\|\mathbf z_k^\eta\|)\|\xi_{k+1}\|^2\right]
\le
C\eta^{5/2}[\eta^{-2}]
=
\mathcal O(\eta^{1/2}),
\\
&\mathbb E|C_{\eta}^{ggB}|
\le
C\eta^3
\sum_{k=0}^{[\eta^{-2}]-1}
\mathbb E\left[(1+\|\mathbf z_k^\eta\|^2)\|\xi_{k+1}\|\right]
\le
C\eta^3[\eta^{-2}]
=
\mathcal{O}(\eta),
\\
&\mathbb E|C_{\eta}^{ggg}|
\le
C\eta^{7/2}
\sum_{k=0}^{[\eta^{-2}]-1}
\mathbb E\left[1+\|\mathbf z_k^\eta\|^3\right]
\le
C\eta^{7/2}[\eta^{-2}]
=
\mathcal O(\eta^{3/2}).
\end{align*}
Therefore,
$C_{\eta}^{gBB}\to0$,
$C_{\eta}^{ggB}\to0$,
$C_{\eta}^{ggg}\to0$,
in \(L^1\), and hence in probability.

It remains to estimate the Brownian cubic term \(C_\eta^B\). Decompose
\[
C_{\eta}^{B}=C_{\eta,1}^{B}+C_{\eta,2}^{B},
\]
where
\begin{align*}
&C_{\eta,1}^{B}
:=
\eta^2
\sum_{k=0}^{[\eta^{-2}]-1}
\nabla^3\psi(\mathbf z_k^\eta)
[\sigma\xi_{k+1},\sigma\xi_{k+1},\sigma\xi_{k+1}],
\\
&C_{\eta,2}^{B}
:=
\eta^2
\sum_{k=0}^{[\eta^{-2}]-1}
\int_0^1
\left(
\nabla^3\psi(\mathbf z_k^\eta+t\Delta_k^\eta)
-
\nabla^3\psi(\mathbf z_k^\eta)
\right)
[\sigma\xi_{k+1},\sigma\xi_{k+1},\sigma\xi_{k+1}]dt.
\end{align*}
For \(C_{\eta,1}^{B}\), since \(\mathbf z_k^\eta\) is
\(\mathcal F_k\)-measurable and
\[
\mathbb E\left[
\nabla^3\psi(\mathbf z_k^\eta)
[\sigma\xi_{k+1},\sigma\xi_{k+1},\sigma\xi_{k+1}]
\mid \mathcal F_k
\right]=0
\]
by the vanishing of the third moment of a centered Gaussian vector,
\(C_{\eta,1}^{B}\) is a martingale-difference sum. Boundedness of
\(\nabla^3\psi\) and Gaussian moments give
\[
\mathbb E|C_{\eta,1}^{B}|^2
\le
C\eta^4
\sum_{k=0}^{[\eta^{-2}]-1}1
\le
C\eta^4[\eta^{-2}]
\le
C\eta^2
\to0.
\]
Hence \(C_{\eta,1}^{B}\to0\) in probability.

For \(C_{\eta,2}^{B}\), boundedness of \(\nabla^4\psi\) gives
\[
\left\|
\nabla^3\psi(\mathbf z_k^\eta+t\Delta_k^\eta)
-
\nabla^3\psi(\mathbf z_k^\eta)
\right\|_{\mathrm{op}}
\le
C\|\Delta_k^\eta\|.
\]
Therefore,
\[
\mathbb E|C_{\eta,2}^{B}|
\le
C\eta^2
\sum_{k=0}^{[\eta^{-2}]-1}
\mathbb E\left[
\|\Delta_k^\eta\|\|\xi_{k+1}\|^3
\right].
\]
Using
$\Delta_k^\eta
=
\eta g(\mathbf z_k^\eta)
+\sqrt\eta\sigma\xi_{k+1}
+\eta\varepsilon_{k+1}$,
the growth bound on \(g\), Gaussian moments, Assumption~\ref{ass:sg},
conditional independence of \(\varepsilon_{k+1}\) and \(\xi_{k+1}\), and
Lemma~\ref{lem:sg-moment} yield
\[
\begin{aligned}
\mathbb E\left[
\|\Delta_k^\eta\|\|\xi_{k+1}\|^3
\right]
&\le
C\eta
\mathbb E\left[(1+\|\mathbf z_k^\eta\|)\|\xi_{k+1}\|^3\right]
+
C\sqrt\eta\mathbb E\|\xi_{k+1}\|^4
+
C\eta\mathbb E\left[\|\varepsilon_{k+1}\|\|\xi_{k+1}\|^3\right]  \\
&\le
C\eta+C\sqrt\eta+C\eta
\le
C\sqrt\eta.
\end{aligned}
\]
Thus,
\[
\mathbb E|C_{\eta,2}^{B}|
\le
C\eta^2[\eta^{-2}]\sqrt\eta
\le
C\eta^{1/2}
\to0.
\]
Therefore \(C_{\eta,2}^{B}\to0\) in probability, and hence
\(C_{\eta}^{B}\to0\) in probability.

Combining the estimates for all terms in \(R_\eta^{\rm ex}\) proves the lemma.
\end{proof}

\begin{remark}[Remainder estimates]
Lemma~\ref{lem:sg-negligible} isolates the new Taylor terms generated by the stochastic gradient increment $\eta\varepsilon_{k+1}$ and proves their negligibility by explicit monomial estimates. Lemma~\ref{lem:exact-remainder-sg-path} controls the remaining exact-gradient Taylor terms along the stochastic gradient path. We only need convergence in probability for the central limit theorem, rather than the exponential remainder bounds used for the moderate-deviation result in \cite{lu2022central}.
\end{remark}

Now, we are finally ready to complete the proof of Theorem~\ref{thm:clt}.

\subsubsection{Completing the Proof of Theorem~\ref{thm:clt}}

\begin{proof}
Set $\hat h=h-\mu(h)$. By Proposition~\ref{prop:ltx} and \cite[Lemma~3.1]{lu2022central}, the Poisson equation $\hat h=\mathcal A\phi$ admits a solution $\phi\in C_b^4(\mathbb R^n)$. Write
\[
\Delta_k^\eta:=\mathbf z_{k+1}^\eta-\mathbf z_k^\eta
=
\eta g(\mathbf z_k^\eta)
+\sqrt\eta\sigma\xi_{k+1}
+\eta\varepsilon_{k+1}.
\]
Using $[\eta^{-2}]^{-1}=\eta^2+o(\eta^2)$ and the boundedness of $h$, the difference between $[\eta^{-2}]^{-1}\sum_{k=0}^{[\eta^{-2}]-1}\hat h(\mathbf z_k^\eta)$ and $\eta^2\sum_{k=0}^{[\eta^{-2}]-1}\hat h(\mathbf z_k^\eta)$ is $o(\sqrt\eta)$. Hence, as $\eta\rightarrow 0$,
\[
\eta^{-1/2}
\left(\widehat{\mu}_{\eta}^{\rm SG}(h)-\mu(h)\right)
=
\eta^{3/2}\sum_{k=0}^{[\eta^{-2}]-1}\mathcal A\phi(\mathbf z_k^\eta)+o(1).
\]
By telescoping,
\[
\eta^{3/2}\sum_{k=0}^{[\eta^{-2}]-1}\mathcal A\phi(\mathbf z_k^\eta)
=
\sqrt\eta\left(\phi(\mathbf z_{[\eta^{-2}]}^\eta)-\phi(\mathbf z_0^\eta)\right)
+\sqrt\eta\sum_{k=0}^{[\eta^{-2}]-1}
\left[
\eta\mathcal A\phi(\mathbf z_k^\eta)
-\left(\phi(\mathbf z_{k+1}^\eta)-\phi(\mathbf z_k^\eta)\right)
\right].
\]
Using the Taylor expansion of $\phi(\mathbf z_{k+1}^\eta)-\phi(\mathbf z_k^\eta)$ with third-order integral remainder, as in the Poisson-equation decomposition of \cite{lu2022central}, yields
\[
\eta^{-1/2}
\left(\widehat{\mu}_{\eta}^{\rm SG}(h)-\mu(h)\right)
=
H_\eta^B+H_\eta^{\rm SG}+R_\eta^{\rm ex}+R_\eta^{\rm SG}+o_P(1),
\]
with
\[
H_\eta^B
:=
-\eta\sum_{k=0}^{[\eta^{-2}]-1}
\left\langle\nabla\phi(\mathbf z_k^\eta),\sigma\xi_{k+1}\right\rangle,
\qquad
H_\eta^{\rm SG}
:=
-\eta^{3/2}\sum_{k=0}^{[\eta^{-2}]-1}
\left\langle\nabla\phi(\mathbf z_k^\eta),\varepsilon_{k+1}\right\rangle.
\]
Here $R_\eta^{\rm SG}$ denotes Taylor terms containing at least one factor $\eta\varepsilon_{k+1}$, while $R_\eta^{\rm ex}$ denotes the remaining exact-gradient Taylor terms. Lemma~\ref{lem:sg-negligible} gives $H_\eta^{\rm SG}\to0$ in probability and $R_\eta^{\rm SG}\to0$ in probability. Lemma~\ref{lem:exact-remainder-sg-path} gives $R_\eta^{\rm ex}\to0$ in probability.

It remains to prove the weak convergence of $H_\eta^B$. Define the martingale-difference triangular array
\[
X_{k,\eta}
:=
-\eta\left\langle\nabla\phi(\mathbf z_k^\eta),\sigma\xi_{k+1}\right\rangle,
\qquad 0\le k\le [\eta^{-2}]-1.
\]
Then $H_\eta^B=\sum_{k=0}^{[\eta^{-2}]-1}X_{k,\eta}$ and
\[
\sum_{k=0}^{[\eta^{-2}]-1}
\mathbb E\left[X_{k,\eta}^2\mid\mathcal F_k\right]
=
\eta^2\sum_{k=0}^{[\eta^{-2}]-1}
\left\|\sigma^\top\nabla\phi(\mathbf z_k^\eta)\right\|^2.
\]
Since $\phi\in C_b^4(\mathbb R^n)$, the function
$F(\mathbf z):=\left\|\sigma^\top\nabla\phi(\mathbf z)\right\|^2$
belongs to $C_b^2(\mathbb R^n)$. Applying Proposition~\ref{prop:ltx} and \cite[Lemma~3.1]{lu2022central} again, the Poisson equation
$F-\mu(F)=\mathcal A\psi_F$
admits a solution $\psi_F\in C_b^4(\mathbb R^n)$. Therefore Lemma~\ref{lem:sg-poisson-lln} applies to $F$, and gives
\[
\eta^2\sum_{k=0}^{[\eta^{-2}]-1}
\left\|\sigma^\top\nabla\phi(\mathbf z_k^\eta)\right\|^2
\to
\mu\left(\left\|\sigma^\top\nabla\phi\right\|^2\right)
=
V(h)
\]
in probability as $\eta\rightarrow 0$.

For the conditional Lindeberg condition, for every $\delta>0$, boundedness of $\nabla\phi$ gives
\[
\sum_{k=0}^{[\eta^{-2}]-1}
\mathbb E\left[\left.
X_{k,\eta}^2\mathbf 1_{\{|X_{k,\eta}|>\delta\}}
\right|\mathcal F_k
\right]
\le
C\eta^2[\eta^{-2}]
\mathbb E\left[
\|\xi_1\|^2\mathbf 1_{\{C\eta\|\xi_1\|>\delta\}}
\right]
\to0.
\]
Since the conditional quadratic variation converges to $V(h)$ in probability and the conditional Lindeberg condition holds, the martingale-array central limit theorem (see e.g. \cite[Corollary~3.1]{hall1980martingale}) implies
$H_\eta^B\Rightarrow\mathcal N(0,V(h))$, as $\eta\rightarrow 0$.
Combining this convergence with the negligibility of $H_\eta^{\rm SG}$, $R_\eta^{\rm SG}$, and $R_\eta^{\rm ex}$ proves the theorem.
\end{proof}

%%%%%%%%%%%%%%%%%%%%%%%%%%%%%%%%%%%%%%%%%%%%%%%%%%%%%%%%

\section{Examples}\label{sec:examples}

%%%%%%%%%%%%%%%%%%%%%%%%%%%%%%%%%%%%%%%%%%%%%%%%%%%%%%%%
\subsection{Non-reversible Langevin dynamics}

We first show that the standard non-reversible Langevin dynamics (NLD) fits into the framework of Section~\ref{sec:model_and_assumptions}.
Let $n=d$ and $H(\mathbf z)=U(\mathbf z)$. Consider the choice
$\mathcal D = I_d$ and
$\mathcal Q = J$,
where $J\in\R^{d\times d}$ is a constant anti-symmetric matrix, i.e.\ $J^\top=-J$.
Then the generalized dynamics \eqref{eq:sde} reduces to
\begin{equation}\label{eq:nld-example-sde}
d\theta_t = -(I_d+J)\nabla U(\theta_t)dt + \sqrt{2}dB_t,
\end{equation}
where $B_{t}$ is a standard $d$-dimensional Brownian motion, 
which is exactly the non-reversible Langevin dynamics introduced in \eqref{eqn:anti}.

Next, we introduce the stochastic gradient non-reversible Langevin Monte Carlo algorithm as the stochastic gradient Euler--Maruyama discretization associated with \eqref{eq:nld-example-sde}:
\begin{equation}\label{eq:nld-example-em}
\theta_{k+1}
=
\theta_k
-\eta(I_d+J)\widehat{\nabla U}(\theta_k,\zeta_{k+1})
+\sqrt{2\eta}\xi_{k+1}.
\end{equation}
Here $\widehat{\nabla U}$ is an unbiased stochastic gradient oracle for $\nabla U$ satisfying Assumption~\ref{ass:sg} with $H=U$, $\mathcal D=I_d$, and $\mathcal Q=J$.

The following proposition verifies that this example is covered by our general theory.

\begin{proposition}\label{prop:nld-example}
Assume that $U\in C^4(\R^d)$ satisfies Assumption~\ref{ass:U} with $H=U$, that the drift term
$g_J(x):=-(I_d+J)\nabla U(x)$
satisfies Assumption~\ref{ass:diss}, and that the stochastic gradient oracle in \eqref{eq:nld-example-em} satisfies Assumption~\ref{ass:sg}. Then the conclusions of Theorem~\ref{thm:clt} apply to the stochastic gradient scheme \eqref{eq:nld-example-em}. In particular, for every $h\in C_b^2(\R^d)$,
\[
\eta^{-1/2}\left(\widehat{\mu}_\eta^{J,{\rm SG}}(h)-\mu(h)\right)
\Rightarrow
\mathcal N\left(0,V_J(h)\right),
\qquad\text{as $\eta\to0$},
\]
where 
$\widehat{\mu}_\eta^{J,{\rm SG}}(h):=\frac{1}{[\eta^{-2}]}\sum_{k=0}^{[\eta^{-2}]-1}h(\theta_k)$,
where $\theta_{k}$ is given in \eqref{eq:nld-example-em}. Under the $\eta$-uniform conditional moment bound in Assumption~\ref{ass:sg}, the variance formula is unchanged by the fixed-variance stochastic gradient perturbation:
\[
V_J(h):=2\int_{\R^d}\|\nabla\phi_J(x)\|^2\mu(dx),
\]
since $\sigma=\sqrt{2}I_d$ in this case. Here, $\phi_J$ solves the limiting full-gradient Poisson equation
$h-\mu(h)=\mathcal A_J\phi_J$,
where $\mathcal A_J$ is the infinitesimal generator
\[
\mathcal A_J f(x)= -\ip{(I_d+J)\nabla U(x)}{\nabla f(x)}+\Delta f(x).
\]
Moreover, if the strict inequality condition in Theorem~\ref{thm:strict-vr} holds, then
\[
V_J(h)<V_0(h),
\]
where $V_0(h)$ is the corresponding small-stepsize fluctuation constant for the reversible overdamped Langevin dynamics ($J=0$).
\end{proposition}

\begin{proof}
The identification of \eqref{eq:nld-example-sde} and \eqref{eq:nld-example-em} with the general full-gradient diffusion and stochastic gradient scheme \eqref{eq:sde}--\eqref{eq:em} is immediate by taking $\mathcal D=I_d$ and $\mathcal Q=J$. The invariant measure is $\mu(dx)\propto e^{-U(x)}dx$ by Lemma~\ref{lem:invariance}. The structural assumptions required by Proposition~\ref{prop:ltx} are satisfied by construction, while Assumption~\ref{ass:sg} controls the stochastic gradient perturbation. Therefore, Theorem~\ref{thm:clt} yields the stated CLT with the same $V_J(h)$. The strict comparison of the small-stepsize fluctuation constants is exactly Theorem~\ref{thm:strict-vr}.
\end{proof}

\begin{remark}
Proposition~\ref{prop:nld-example} gives an algorithm-level interpretation of the classical variance reduction effect of non-reversibility for the stochastic gradient scheme. More precisely, the continuous-time asymptotic variance reduction transfers to the leading-order fluctuation constant of the stochastic gradient Euler--Maruyama estimator in the small-stepsize regime $\eta\to0$.
\end{remark}

%%%%%%%%%%%%%%%%%%%%%%%%%%%%%%%%%%%%%%%%%%%%%%%%%%%%%%%%
\subsection{Hessian-free high resolution dynamics}

We next consider an augmented-state example, that is, the Hessian-free high-resolution (HFHR) dynamics in \cite{li2022hessian} that is motivated by the high--resolution ordinary differential equation (ODE) viewpoint on \textit{Nesterov's accelerated gradient} (NAG) method in optimization \cite{Nesterov1983,Nesterov2013}.
Its deterministic backbone is the high--resolution ODE
\begin{equation}\label{eq:hfhr-ode}
\begin{cases}
\dot \theta_t = r_t-\alpha \nabla U(\theta_t),\\
\dot r_t = -\gamma r_t-\nabla U(\theta_t),
\end{cases}
\end{equation}
where $\alpha,\gamma>0$ are fixed parameters.
This is the deterministic backbone underlying the HFHR dynamics introduced by \cite{li2022hessian}, which is a stochastic perturbation of \eqref{eq:hfhr-ode} that fits into the framework of Section~\ref{sec:model_and_assumptions}.
Let the state variable be $\mathbf z=(\theta,r)\in\R^{2d}$ and define the Hamiltonian
\begin{equation}\label{eq:hfhr-H}
H(\theta,r)=U(\theta)+\frac12\|r\|^2.
\end{equation}
Set
\begin{equation}\label{eq:hfhr-DQ}
\mathcal D=
\begin{pmatrix}
\alpha I_d & 0\\
0 & \gamma I_d
\end{pmatrix},
\qquad
\mathcal Q=
\begin{pmatrix}
0 & -I_d\\
I_d & 0
\end{pmatrix}.
\end{equation}
Then $\mathcal D$ is symmetric positive definite and $\mathcal Q$ is anti-symmetric. Moreover,
\[
-(\mathcal D+\mathcal Q)\nabla H(\theta,r)
=
\begin{pmatrix}
r-\alpha \nabla U(\theta)\\
-\gamma r-\nabla U(\theta)
\end{pmatrix}.
\]
Hence, the generalized non-reversible Langevin dynamics \eqref{eq:sde} becomes
\begin{equation}\label{eq:hfhr-sde}
\begin{cases}
d\theta_t = \left(r_t-\alpha \nabla U(\theta_t)\right)dt + \sqrt{2\alpha}dB_t^{(1)},\\
dr_t = \left(-\gamma r_t-\nabla U(\theta_t)\right)dt + \sqrt{2\gamma}dB_t^{(2)},
\end{cases}
\end{equation}
where $B_t^{(1)}$ and $B_t^{(2)}$ are two independent standard $d$-dimensional Brownian motions.
The non-asymptotic convergence guarantees for HFHR dynamics \eqref{eq:hfhr-sde}
and its discretaizations have been studied in \cite{li2022hessian} when $U$ is strong-convex by using
the synchronous coupling method. 
Very recently, by adopting the reflection/synchronous coupling framework, \cite{WWZ2026} obtains
the non-asymptotic convergence guarantees for HFHR dynamics \eqref{eq:hfhr-sde}
when $U$ may be non-convex.

By Lemma~\ref{lem:invariance}, the invariant measure of \eqref{eq:hfhr-sde} is
\begin{equation}\label{eq:hfhr-mu}
\mu(d\theta,dr)=Z^{-1}\exp\left(-U(\theta)-\frac12\|r\|^2\right)d\theta dr,
\end{equation}
where $Z>0$ is the normalizing constant.

The reversible baseline corresponding to this HFHR dynamics is obtained by setting $\mathcal Q=0$
with $\mathcal{D}$ the same as in \eqref{eq:hfhr-DQ}.
In that case, the dynamics \eqref{eq:sde} becomes
\begin{equation}\label{eq:hfhr-baseline-sde}
\begin{cases}
d\theta_t = -\alpha \nabla U(\theta_t)dt + \sqrt{2\alpha}dB_t^{(1)},\\
dr_t = -\gamma r_tdt + \sqrt{2\gamma}dB_t^{(2)},
\end{cases}
\end{equation}
which is precisely the combination of an overdamped Langevin diffusion in the $\theta$-variable and an independent Ornstein--Uhlenbeck process in the $r$-variable. This baseline has the same invariant measure \eqref{eq:hfhr-mu}.

Next, we introduce the stochastic gradient Hessian-free high-resolution Monte Carlo algorithm as the stochastic gradient Euler--Maruyama discretization associated with \eqref{eq:hfhr-sde}:
\begin{equation}\label{eq:hfhr-em}
\begin{cases}
\theta_{k+1}
=
\theta_k+\eta\left(r_k-\alpha\widehat{\nabla U}(\theta_k,\zeta_{k+1})\right)
+\sqrt{2\alpha\eta}\xi_{k+1}^{(1)},
\\[1mm]
r_{k+1}
=
r_k+\eta\left(-\gamma r_k-\widehat{\nabla U}(\theta_k,\zeta_{k+1})\right)
+\sqrt{2\gamma\eta}\xi_{k+1}^{(2)}.
\end{cases}
\end{equation}
Equivalently, this is the SG-GNLMC scheme with
$\widehat{\nabla H}(\theta,r,\zeta)
=
\begin{pmatrix}
\widehat{\nabla U}(\theta,\zeta)\\ r
\end{pmatrix}$.

\begin{remark}
The discretization \eqref{eq:hfhr-em} should be interpreted as a simple stochastic gradient Euler--Maruyama discretization associated with an HFHR-type limiting diffusion. It is not the tailored structure-preserving discretization analyzed in \cite{li2022hessian}. Our purpose here is only to show that, once the HFHR drift is embedded into the constant-coefficient GNLD framework, this stochastic gradient discretization falls within the scope of Theorem~\ref{thm:clt}. It is natural to expect that more sophisticated discretizations of HFHR may lead to further improvements, but analyzing such schemes is beyond the scope of the present paper.
\end{remark}

\begin{proposition}\label{prop:hfhr-example}
Assume that $U\in C^4(\R^d)$ satisfies Assumption~\ref{ass:U} with
$H(\theta,r)=U(\theta)+\frac12\|r\|^2$,
that the corresponding limiting full-gradient drift
$g_{\mathrm{HFHR}}(\theta,r)
=
\begin{pmatrix}
r-\alpha \nabla U(\theta)\\
-\gamma r-\nabla U(\theta)
\end{pmatrix}$
satisfies Assumption~\ref{ass:diss} (for instance, by the sufficient condition in Remark~\ref{rem:aug-diss}), and that the stochastic gradient oracle in \eqref{eq:hfhr-em} satisfies Assumption~\ref{ass:sg}. 
Then the conclusions of Theorem~\ref{thm:clt} apply to the stochastic gradient scheme \eqref{eq:hfhr-em}. In particular, for every $h\in C_b^2(\R^{2d})$,
\[
\eta^{-1/2}\left(\widehat{\mu}_\eta^{\mathrm{HFHR},\mathrm{SG}}(h)-\mu(h)\right)
\Rightarrow
\mathcal N\left(0,V_{\mathrm{HFHR}}(h)\right),
\qquad \text{as $\eta\to0$},
\]
where 
$\widehat{\mu}_\eta^{\mathrm{HFHR},\mathrm{SG}}(h):=\frac{1}{[\eta^{-2}]}\sum_{k=0}^{[\eta^{-2}]-1}h(\theta_k,r_{k})$
,
where $(\theta_{k},r_{k})$ is given in \eqref{eq:hfhr-em} and $\mu$ is the invariant measure in \eqref{eq:hfhr-mu}. Under the $\eta$-uniform conditional moment bound in Assumption~\ref{ass:sg}, the variance formula is unchanged by the fixed-variance stochastic gradient perturbation:
\[
V_{\mathrm{HFHR}}(h)
:=
\mu\left(\left\|\sigma^\top\nabla\phi_{\mathrm{HFHR}}\right\|^2\right),
\quad
\text{with}
\quad
\sigma=
\begin{pmatrix}
\sqrt{2\alpha}I_d & 0\\
0 & \sqrt{2\gamma}I_d
\end{pmatrix},
\]
and $\phi_{\mathrm{HFHR}}$ solving the limiting full-gradient Poisson equation
$h-\mu(h)=\mathcal A_{\mathrm{HFHR}}\phi_{\mathrm{HFHR}}$,
where the limiting full-gradient infinitesimal generator is
\begin{equation}\label{eq:hfhr-generator}
\mathcal A_{\mathrm{HFHR}} f(\theta,r)
=
\ip{r-\alpha \nabla U(\theta)}{\nabla_\theta f(\theta,r)}
+
\ip{-\gamma r-\nabla U(\theta)}{\nabla_r f(\theta,r)}
+
\alpha \Delta_\theta f(\theta,r)+\gamma \Delta_r f(\theta,r).
\end{equation}
\end{proposition}

\begin{proof}
Take
$\mathcal D=
\begin{pmatrix}
\alpha I_d & 0\\
0 & \gamma I_d
\end{pmatrix}$,
$\mathcal Q=
\begin{pmatrix}
0 & -I_d\\
I_d & 0
\end{pmatrix}$
and $H(\theta,r)=U(\theta)+\frac12\|r\|^2$.
Then the limiting dynamics \eqref{eq:hfhr-sde} is exactly of the form \eqref{eq:sde}, and \eqref{eq:hfhr-em} is the corresponding stochastic gradient scheme \eqref{eq:em}. Moreover,
$\mathcal D=\frac12 \sigma\sigma^\top$,
with $\sigma=\mathrm{diag}(\sqrt{2\alpha}I_d,\sqrt{2\gamma}I_d)$. Therefore, Proposition~\ref{prop:ltx} applies once Assumptions~\ref{ass:U} and \ref{ass:diss} are in force, while Assumption~\ref{ass:sg} controls the stochastic gradient perturbation. Theorem~\ref{thm:clt} then yields the stated small-stepsize CLT with the same $V_{\mathrm{HFHR}}(h)$.
\end{proof}

\begin{remark}\label{rem:hfhr-projected}
Theorem~\ref{thm:clt} also applies to observables depending only on the position variable for the stochastic gradient HFHR scheme. Indeed, if $\bar h\in C_b^2(\R^d)$ and we define
$h(\theta,r):=\bar h(\theta)$,
then $h\in C_b^2(\R^{2d})$ and
$\widehat{\mu}_{\eta,\theta}^{\mathrm{HFHR},\mathrm{SG}}(\bar h)
:=
\frac{1}{[\eta^{-2}]}
\sum_{k=0}^{[\eta^{-2}]-1}\bar h(\theta_k)
=
\widehat{\mu}_{\eta}^{\mathrm{HFHR},\mathrm{SG}}(h)$.
Therefore, Theorem~\ref{thm:clt} yields the same small-stepsize CLT for the projected stochastic gradient estimator.

Moreover, the variance comparison in Theorem~\ref{thm:strict-vr} also applies to this class of observables, since it is formulated for arbitrary test functions on the full augmented state space. Thus, for $h(\theta,r)=\bar h(\theta)$, one may compare the HFHR fluctuation constant with that of its reversible baseline obtained by setting $\mathcal Q=0$.

For HFHR, this reversible baseline decouples into an overdamped Langevin dynamics in $\theta$ and an independent Ornstein--Uhlenbeck process in $r$. Hence, for observables of the form $h(\theta,r)=\bar h(\theta)$, the corresponding baseline fluctuation constant coincides with that of the overdamped Langevin dynamics for $\bar h$ on $\R^d$. Therefore, whenever Theorem~\ref{thm:strict-vr} applies, the HFHR dynamics fluctuation constant for position only observables is no larger than that of overdamped Langevin dynamics, with strict inequality under the strict condition there.
\end{remark}

%%%%%%%%%%%%%%%%%%%%%%%%%%%%%

\subsection{Gradient-adjusted underdamped Langevin dynamics}

Next, we consider the gradient-adjusted underdamped Langevin dynamics (GAUL) proposed in \cite{zuo2025gradient}. 
From the optimization viewpoint, \cite{zuo2025gradient} is motivated by recent progress on primal-dual damping and Hessian-driven damping dynamics, which enriches the Nesterov-type second-order flows underlying classical underdamped Langevin dynamics and have been observed to converge faster toward minimizers. 
The idea in \cite{zuo2025gradient} is therefore to inject suitable stochastic perturbations into these accelerated deterministic dynamics so as to obtain a sampling method that preserves the target invariant measure while potentially improving convergence over standard overdamped or underdamped Langevin dynamics.
This provides another augmented-state example covered by the constant-coefficient GNLD framework of Section~\ref{sec:model_and_assumptions}, and it may be viewed as a preconditioned extension of the HFHR dynamics discussed above.

Let the state variable be $\mathbf z=(\theta,r)\in\R^{2d}$ and define the Hamiltonian
\begin{equation}\label{eq:gaul-H}
H(\theta,r)=U(\theta)+\frac12\|r\|^2.
\end{equation}
Let $a,\gamma>0$ and $C\in\R^{d\times d}$ be a symmetric positive-definite matrix. Following \cite{zuo2025gradient}, define
\begin{equation}\label{eq:gaul-bigQ}
\mathbf Q:=
\begin{pmatrix}
aC & -C\\
I_d & \gamma I_d
\end{pmatrix}.
\end{equation}
We decompose $\mathbf Q$ into its symmetric and anti-symmetric parts:
\begin{equation}\label{eq:gaul-DQ}
\mathcal D:=\frac{\mathbf Q+\mathbf Q^\top}{2}
=
\begin{pmatrix}
aC & \frac{I_d-C}{2}\\[1mm]
\frac{I_d-C}{2} & \gamma I_d
\end{pmatrix},
\qquad
\mathcal Q:=\frac{\mathbf Q-\mathbf Q^\top}{2}
=
\begin{pmatrix}
0 & -\frac{I_d+C}{2}\\[1mm]
\frac{I_d+C}{2} & 0
\end{pmatrix}.
\end{equation}
Then $\mathcal D^\top=\mathcal D$ and $\mathcal Q^\top=-\mathcal Q$. Throughout this subsection, we additionally assume that
$\mathcal D$ is positive definite,
so that the resulting dynamics falls within the constant-coefficient GNLD framework of Section~\ref{sec:model_and_assumptions}.
%%%%%%%%%%%%%%%%%%%%%%%%%%
A direct computation gives
\[
-(\mathcal D+\mathcal Q)\nabla H(\theta,r)
=
-\mathbf Q\nabla H(\theta,r)
=
\begin{pmatrix}
Cr-aC\nabla U(\theta)\\
-\gamma r-\nabla U(\theta)
\end{pmatrix}.
\]
Hence, the generalized non-reversible Langevin dynamics \eqref{eq:sde} becomes
\begin{equation}\label{eq:gaul-sde}
d\mathbf z_t
=
-\mathbf Q\nabla H(\mathbf z_t)dt
+\sqrt{2\mathcal D}d\mathbf B_t.
\end{equation}
The original GAUL paper proves that this dynamics preserves the desired joint invariant distribution and studies convergence primarily in the quadratic Gaussian setting, with particular emphasis on the case $C=I_d$; for general preconditioners, positivity of the symmetric part is an additional constraint. Our Assumption~\ref{ass:diss} is a separate stability condition needed for the small-stepsize fluctuation theorem.
Equivalently, \eqref{eq:gaul-sde} can be written componentwise as
\begin{equation}\label{eq:gaul-sde-components}
\begin{cases}
d\theta_t = \left(Cr_t-aC\nabla U(\theta_t)\right)dt + dM_t^{(1)},\\[1mm]
dr_t = \left(-\gamma r_t-\nabla U(\theta_t)\right)dt + dM_t^{(2)},
\end{cases}
\end{equation}
where the Gaussian increment $\left(dM_t^{(1)},dM_t^{(2)}\right)\in\R^{2d}$ is centered with covariance
\[
\mathbb E
\begin{pmatrix}
dM_t^{(1)}\\
dM_t^{(2)}
\end{pmatrix}
\begin{pmatrix}
dM_t^{(1)}\\
dM_t^{(2)}
\end{pmatrix}^{\top}
=
2\mathcal Ddt
=
\begin{pmatrix}
2aC & I_d-C\\
I_d-C & 2\gamma I_d
\end{pmatrix}dt,
\]
where the joint diffusion matrix is $2\mathcal D$. This yields a positive-definite subclass of the GAUL dynamics considered in \cite{zuo2025gradient}.
The infinitesimal generator of \eqref{eq:gaul-sde} is
\begin{equation}\label{eq:gaul-generator}
\mathcal A_{\mathrm{GAUL}} f(\theta,r)
= \ip{Cr-aC\nabla U(\theta)}{\nabla_\theta f(\theta,r)}
+ \ip{-\gamma r-\nabla U(\theta)}{\nabla_r f(\theta,r)}
+ \operatorname{Tr}\left(\mathcal D\nabla^2 f(\theta,r)\right),
\end{equation}
where $\nabla^2 f$ denotes the full $2d\times 2d$ Hessian.

By Lemma~\ref{lem:invariance}, the invariant measure of \eqref{eq:gaul-sde} is
\begin{equation}\label{eq:gaul-mu}
\mu(d\theta,dr)=Z^{-1}\exp\left(-U(\theta)-\frac12\|r\|^2\right)d\theta dr,
\end{equation}
where $Z>0$ is the normalizing constant.

The reversible baseline corresponding to this GAUL dynamics is obtained by setting $\mathcal Q=0$ while keeping the same symmetric matrix $\mathcal D$. In that case, the dynamics \eqref{eq:sde} becomes
\begin{equation}\label{eq:gaul-baseline-sde}
d\mathbf z_t = -\mathcal D\nabla H(\mathbf z_t)dt+\sqrt{2\mathcal D}d\mathbf B_t,
\end{equation}
which has the same invariant measure \eqref{eq:gaul-mu}.

Next, we introduce the stochastic gradient GAULMC scheme associated with \eqref{eq:gaul-sde}:
\begin{equation}\label{eq:gaul-em}
\mathbf z_{k+1}
=
\mathbf z_k
-\eta\mathbf Q
\begin{pmatrix}
\widehat{\nabla U}(\theta_k,\zeta_{k+1})\\
r_k
\end{pmatrix}
+\sqrt{2\eta\mathcal D}\xi_{k+1}.
\end{equation}
The Gaussian increment has covariance $2\eta\mathcal D$. Suppressing this generally correlated Gaussian increment, the componentwise drift part is
\begin{align*}
&\theta_{k+1}
=
\theta_k+\eta\left(Cr_k-aC\widehat{\nabla U}(\theta_k,\zeta_{k+1})\right),
\\
&r_{k+1}
=
r_k+\eta\left(-\gamma r_k-\widehat{\nabla U}(\theta_k,\zeta_{k+1})\right).
\end{align*}

This stochastic gradient Euler--Maruyama scheme is the direct stochastic gradient analogue of the simple discretization considered in \cite[Section 3.3 and Appendix A]{zuo2025gradient}. We note that \cite{zuo2025gradient} also proposes a splitting-based discretization, which is shown there to have a smaller first-order asymptotic bias than the exact-gradient Euler--Maruyama scheme for Gaussian targets.

\begin{remark}
The GAUL dynamics reduces to the HFHR example in Section~\ref{sec:examples} when $C=I_d$ and $a=\alpha$. Thus the present example is a natural preconditioned extension of HFHR within the constant-coefficient GNLD framework.
\end{remark}

\begin{proposition}\label{prop:gaul-example}
Assume that $U\in C^4(\R^d)$ satisfies Assumption~\ref{ass:U} with $H(\theta,r)=U(\theta)+\tfrac12\|r\|^2$, that $\mathcal D$ in \eqref{eq:gaul-DQ} is positive definite, that the limiting full-gradient drift $g_{\mathrm{GAUL}}$ satisfies Assumption~\ref{ass:diss} (for instance, by the preconditioned sufficient condition in Remark~\ref{rem:aug-diss}), and that the stochastic gradient oracle in \eqref{eq:gaul-em} satisfies Assumption~\ref{ass:sg}. Then the conclusions of Theorem~\ref{thm:clt} apply to \eqref{eq:gaul-em}: for every $h\in C_b^2(\R^{2d})$,
\[
\eta^{-1/2}\left(\widehat{\mu}_\eta^{\mathrm{GAUL},\mathrm{SG}}(h)-\mu(h)\right)
\Rightarrow
\mathcal N\left(0,V_{\mathrm{GAUL}}(h)\right),
\qquad \eta\to0,
\]
where $\widehat{\mu}_\eta^{\mathrm{GAUL},\mathrm{SG}}(h)
:=\frac{1}{[\eta^{-2}]}\sum_{k=0}^{[\eta^{-2}]-1}h(\mathbf z_k)$, 
$V_{\mathrm{GAUL}}(h):=\mu\left(\|\sigma^\top\nabla\phi_{\mathrm{GAUL}}\|^2\right)$, with $\mathcal D=\tfrac12\sigma\sigma^\top$ and $\phi_{\mathrm{GAUL}}$ solving the limiting full-gradient Poisson equation $h-\mu(h)=\mathcal A_{\mathrm{GAUL}}\phi_{\mathrm{GAUL}}$. Moreover, if the strict inequality condition in Theorem~\ref{thm:strict-vr} holds for this choice of $\mathcal Q$, then
\[
V_{\mathrm{GAUL}}(h)<V_{\mathrm{GAUL},0}(h),
\]
where $V_{\mathrm{GAUL},0}(h)$ denotes the corresponding fluctuation constant for the reversible baseline \eqref{eq:gaul-baseline-sde}.
\end{proposition}

\begin{proof}
By construction $\mathcal D+\mathcal Q=\mathbf Q$, so $-(\mathcal D+\mathcal Q)\nabla H=g_{\mathrm{GAUL}}$ and the limiting dynamics \eqref{eq:gaul-sde} is of the form \eqref{eq:sde} with the scheme \eqref{eq:gaul-em} as its stochastic gradient discretization \eqref{eq:em}. Since $\mathcal D$ is positive definite, a constant invertible $\sigma$ with $\mathcal D=\tfrac12\sigma\sigma^\top$ exists, and Lemma~\ref{lem:invariance} identifies the invariant measure as \eqref{eq:gaul-mu}. Under Assumptions~\ref{ass:U} and \ref{ass:diss}, Proposition~\ref{prop:ltx} applies, while Assumption~\ref{ass:sg} controls the stochastic gradient perturbation. The CLT then follows from Theorem~\ref{thm:clt}, and the strict comparison from Theorem~\ref{thm:strict-vr}.
\end{proof}

\begin{remark}\label{rem:gaul-projected}
As in Remark~\ref{rem:hfhr-projected}, Theorem~\ref{thm:clt} also applies to observables depending only on the position variable for the stochastic gradient GAUL scheme. Indeed, if $\bar h\in C_b^2(\R^d)$ and we define
\[
h(\theta,r):=\bar h(\theta),
\]
then $h\in C_b^2(\R^{2d})$, and Theorem~\ref{thm:clt} yields a small-stepsize CLT for the projected stochastic gradient estimator
\[
\frac{1}{[\eta^{-2}]}
\sum_{k=0}^{[\eta^{-2}]-1}\bar h(\theta_k).
\]
The variance comparison in Theorem~\ref{thm:strict-vr} also applies to this class of observables, since it is formulated for arbitrary test functions on the full augmented state space. Thus, for observables of the form $h(\theta,r)=\bar h(\theta)$, the comparison established in this paper is naturally between the full augmented GAUL dynamics and its reversible baseline obtained by setting $\mathcal Q=0$ while keeping the same symmetric part $\mathcal D$.
\end{remark}

\begin{remark}
Proposition~\ref{prop:gaul-example} shows that a positive-definite subclass of GAUL fits naturally into our constant-coefficient GNLD framework. In this sense, the stochastic gradient algorithm-level fluctuation analysis developed here applies not only to the classical NLD and the HFHR dynamics, but also to their preconditioned GAUL extension.
\end{remark}

%%%%%%%%%%%%%%%%%%%%%%%%%%%%%%%%%%%%%%%%%%%%%%%%

\section{Explicit Variance Reduction under Quadratic Potentials}\label{sec:quadratic}

In this section, we complement the $C_b^2$ theory of Theorem~\ref{thm:clt} with explicit Gaussian calculations for quadratic Hamiltonians and linear observables. The linear observables considered below are not covered directly by Theorem~\ref{thm:clt}; nevertheless, in the quadratic Gaussian setting, the Poisson equation can be solved explicitly and the martingale decomposition can be verified directly. This yields closed-form expressions for the small-stepsize fluctuation constant and makes the variance-reduction effect of the anti-symmetric perturbation completely transparent.

\begin{remark}[Linear observables and stochastic gradient perturbations]\label{rem:linear-observables-caveat}
The linear observables considered in this section are not covered directly by Theorem~\ref{thm:clt}, because they are not elements of $C_b^2$. Nevertheless, in the quadratic Gaussian setting, the Poisson equation can be solved explicitly, and the same martingale argument can be verified directly using Gaussian moment bounds. Under the fixed-variance scaling in Assumption~\ref{ass:sg}, the stochastic gradient perturbation still contributes only $o_P(1)$ to the normalized estimator. Therefore, the closed-form constants below are the same as those obtained in the exact-gradient discretization.
\end{remark}

\subsection{General linear-Gaussian formula}

Consider the constant-coefficient GNLD \eqref{eq:sde} with the quadratic Hamiltonian
\begin{equation}\label{eq:quad-H-general}
H(\mathbf z)=\frac12 \mathbf z^\top P\mathbf z,
\end{equation}
where $P\in\R^{n\times n}$ is a symmetric positive-definite constant matrix. Then
$\nabla H(\mathbf z)=P\mathbf z$,
and \eqref{eq:sde} becomes the linear diffusion
\begin{equation}\label{eq:quad-linear-sde}
d\mathbf z_t=-A\mathbf z_tdt+\sqrt{2\mathcal D}d\mathbf B_t,
\qquad
A:=(\mathcal D+\mathcal Q)P.
\end{equation}
The invariant measure is
$\mu(d\mathbf z)=Z^{-1}\exp\left(-\frac12\mathbf z^\top P\mathbf z\right)d\mathbf z
=\mathcal N(0,P^{-1})$.

We first compute the fluctuation constant for linear observables.

\begin{proposition}[Explicit formula for linear observables]\label{prop:quadratic-linear}
Assume that $H$ is given by \eqref{eq:quad-H-general}, with $P$ symmetric positive definite, and let
$h(\mathbf z)=u^\top \mathbf z$,
$u\in\R^n$.
Then $\mu(h)=0$, and the Poisson equation
$h=\mathcal A\phi$
admits the linear solution
\[
\phi(\mathbf z)=-\left(\left(A^{-1}\right)^\top u\right)^\top \mathbf z.
\]
Consequently,
\begin{equation}\label{eq:V-general-linear}
V(h)
=
2u^\top A^{-1}\mathcal D \left(A^{-1}\right)^\top u.
\end{equation}
For the reversible baseline $\mathcal Q=0$, one has
\begin{equation}\label{eq:V0-general-linear}
V_0(h)
=
2u^\top ( \mathcal D P )^{-1}\mathcal D\left((\mathcal D P)^{-1}\right)^\top u.
\end{equation}
\end{proposition}

\begin{proof}
Since $h(\mathbf z)=u^\top \mathbf z$ is odd and $\mu$ is centered Gaussian, $\mu(h)=0$.
The infinitesimal generator of \eqref{eq:quad-linear-sde} is given by
\[
\mathcal A f(\mathbf z)
=
-\langle A\mathbf z,\nabla f(\mathbf z)\rangle
+\tr\left(\mathcal D\nabla^2 f(\mathbf z)\right).
\]
Since $\mathcal D$ is symmetric positive definite and $\mathcal Q$ is anti-symmetric, the matrix $\mathcal D+\mathcal Q$ is invertible: if $(\mathcal D+\mathcal Q)v=0$, then
$0=v^\top(\mathcal D+\mathcal Q)v=v^\top\mathcal D v$,
and hence $v=0$. Therefore, $A=(\mathcal D+\mathcal Q)P$ is invertible as well.
We seek a linear solution $\phi(\mathbf z)=c^\top \mathbf z$. Then $\nabla\phi=c$ and $\nabla^2\phi=0$, so that
\[
\mathcal A\phi(\mathbf z)
=
-\langle A\mathbf z,c\rangle
=
-\left(A^\top c\right)^\top \mathbf z.
\]
Thus, $\mathcal A\phi=h$ if and only if
$-A^\top c=u$,
that is,
$c=-\left(A^{-1}\right)^\top u$.
Hence
$\nabla\phi=-\left(A^{-1}\right)^\top u$.
Since $\mathcal D=\frac12\sigma\sigma^\top$,
\[
V(h)=\mu\left(\left\|\sigma^\top\nabla\phi\right\|^2\right)
=
2\langle \nabla\phi,\mathcal D\nabla\phi\rangle
=
2u^\top A^{-1}\mathcal D \left(A^{-1}\right)^\top u,
\]
which proves \eqref{eq:V-general-linear}. Formula \eqref{eq:V0-general-linear} follows by setting $\mathcal Q=0$, so that $A_0=\mathcal D P$.
This completes the proof.
\end{proof}

The previous proposition identifies the variance constant at the Poisson-equation level. The next proposition verifies directly that the same constant indeed governs the stochastic gradient empirical average for linear observables in the quadratic Gaussian setting.

\begin{proposition}[Quadratic Gaussian CLT for linear observables under SG noise]\label{prop:quad-sg-linear-clt}
Assume that the quadratic Hamiltonian \eqref{eq:quad-H-general} is in force with $P$ symmetric positive definite, so that Assumption~\ref{ass:U} is satisfied. Assume further that the corresponding linear drift satisfies Assumption~\ref{ass:diss}, and that Assumption~\ref{ass:sg} holds. Let
$h(\mathbf z)=u^\top\mathbf z$,
$u\in\mathbb R^n$.
Then the stochastic gradient empirical average satisfies
\[
\eta^{-1/2}
\left(
\widehat\mu_\eta^{\rm SG}(h)-\mu(h)
\right)
\Rightarrow
\mathcal N(0,V(h)),
\]
as $\eta\rightarrow 0$, where
\[
V(h)=2u^\top A^{-1}\mathcal D\left(A^{-1}\right)^\top u,
\qquad
A=(\mathcal D+\mathcal Q)P.
\]
\end{proposition}

\begin{proof}
By Proposition~\ref{prop:quadratic-linear}, the Poisson equation has the linear solution
\[
\phi(\mathbf z)
=
-\left(\left(A^{-1}\right)^\top u\right)^\top\mathbf z,
\qquad
A=(\mathcal D+\mathcal Q)P.
\]
Thus $\nabla\phi$ is constant and $\nabla^2\phi=0$.

First, since $\left|[\eta^{-2}]^{-1}-\eta^2\right|=\mathcal O(\eta^4)$ and Lemma~\ref{lem:sg-moment} gives
$\sup_{0\le k\le [\eta^{-2}]}\mathbb E\|\mathbf z_k^\eta\|<\infty$,
we have
\[
\eta^{-1/2}
\left|
\left(\frac1{[\eta^{-2}]}-\eta^2\right)
\sum_{k=0}^{[\eta^{-2}]-1}h(\mathbf z_k^\eta)
\right|
\to0
\]
in probability as $\eta\rightarrow 0$. Hence,
\[
\eta^{-1/2}
\left(
\widehat\mu_\eta^{\rm SG}(h)-\mu(h)
\right)
=
\eta^{3/2}\sum_{k=0}^{[\eta^{-2}]-1}\mathcal A\phi(\mathbf z_k^\eta)
+o_P(1).
\]

Since $\phi$ is linear and
\[
\mathbf z_{k+1}^\eta-\mathbf z_k^\eta
=
\eta g(\mathbf z_k^\eta)
+\sqrt\eta\sigma\xi_{k+1}
+\eta\varepsilon_{k+1},
\]
we have the exact identity
\[
\phi(\mathbf z_{k+1}^\eta)-\phi(\mathbf z_k^\eta)
=
\eta\left\langle\nabla\phi,g(\mathbf z_k^\eta)\right\rangle
+\sqrt\eta\left\langle\nabla\phi,\sigma\xi_{k+1}\right\rangle
+\eta\left\langle\nabla\phi,\varepsilon_{k+1}\right\rangle.
\]
Since $\mathcal A\phi=\langle g,\nabla\phi\rangle$, telescoping gives
\[
\begin{aligned}
\eta^{-1/2}
\left(
\widehat\mu_\eta^{\rm SG}(h)-\mu(h)
\right)
&=
\sqrt\eta\left(\phi(\mathbf z_{[\eta^{-2}]}^\eta)-\phi(\mathbf z_0^\eta)\right)\\
&\quad
-\eta\sum_{k=0}^{[\eta^{-2}]-1}
\langle\nabla\phi,\sigma\xi_{k+1}\rangle
-\eta^{3/2}\sum_{k=0}^{[\eta^{-2}]-1}
\langle\nabla\phi,\varepsilon_{k+1}\rangle
+o_P(1).
\end{aligned}
\]

The boundary term is negligible because $\phi$ is linear and Lemma~\ref{lem:sg-moment} gives
\[
\mathbb E\left[
\eta\left|\phi(\mathbf z_{[\eta^{-2}]}^\eta)-\phi(\mathbf z_0^\eta)\right|^2
\right]
\le
C\eta
\left(
\mathbb E\|\mathbf z_{[\eta^{-2}]}^\eta\|^2+\mathbb E\|\mathbf z_0^\eta\|^2
\right)
\to0,
\]
as $\eta\rightarrow 0$.
The stochastic gradient martingale is negligible since
\[
\mathbb E\left|
\eta^{3/2}\sum_{k=0}^{[\eta^{-2}]-1}
\langle\nabla\phi,\varepsilon_{k+1}\rangle
\right|^2
\le
C\eta^3\sum_{k=0}^{[\eta^{-2}]-1}
\mathbb E\left[1+\|\mathbf z_k^\eta\|^p\right]
\le C\eta\to0,
\]
as $\eta\rightarrow 0$.
Finally, because $\nabla\phi$ is constant,
$-\eta\sum_{k=0}^{[\eta^{-2}]-1}
\langle\nabla\phi,\sigma\xi_{k+1}\rangle$
is a centered Gaussian random variable with variance
\[
\eta^2[\eta^{-2}]\left\|\sigma^\top\nabla\phi\right\|^2
\to
\left\|\sigma^\top\nabla\phi\right\|^2
=
2u^\top A^{-1}\mathcal D\left(A^{-1}\right)^\top u
=
V(h).
\]
This proves the claim.
\end{proof}

\begin{remark}\label{rem:quadratic-general}
Proposition~\ref{prop:quadratic-linear} shows that, in the separate quadratic Gaussian calculation for linear observables, the small-stepsize fluctuation constant is determined by an explicit matrix formula. Therefore, the variance-reduction effect of the anti-symmetric perturbation can be quantified exactly, rather than only through an abstract comparison theorem.
\end{remark}

\subsection{The NLD case}
\label{sec:nld_quadratic}

We first consider the classical NLD example in Section~\ref{sec:examples}. Let
$\mathcal D=I_d$,
$\mathcal Q=J$,
$H(x)=\frac12 x^\top P x$,
where $P\in\R^{d\times d}$ is symmetric positive definite and $J^\top=-J$. Then
$A=(I_d+J)P$.
For a linear observable
$h(x)=u^\top x$,
$u\in\R^d$,
Proposition~\ref{prop:quadratic-linear} yields
\begin{equation}\label{eq:V-nld-general}
V_J(h)
=
2u^\top A^{-1}\left(A^{-1}\right)^\top u
=
2u^\top P^{-1}(I_d+J)^{-1}\left((I_d+J)^{-1}\right)^\top P^{-1}u.
\end{equation}
The reversible baseline corresponds to $J=0$, and its fluctuation constant is
\begin{equation}\label{eq:V-nld-reversible}
V_0(h)=2u^\top P^{-2}u.
\end{equation}

The comparison can also be written in a more transparent form in the isotropic case $P=I_d$. More generally, if we set $w:=P^{-1}u$, then \eqref{eq:V-nld-general} and \eqref{eq:V-nld-reversible} become
\[
V_J(h)=2w^\top (I_d+J)^{-1}\left((I_d+J)^{-1}\right)^\top w,
\qquad
V_0(h)=2\|w\|^2.
\]
Thus, the variance-reduction mechanism is the same as in the isotropic case.

By Proposition~\ref{prop:quad-sg-linear-clt}, these constants also describe the stochastic gradient NLD scheme \eqref{eq:nld-example-em} in the quadratic Gaussian setting under Assumptions~\ref{ass:diss} and~\ref{ass:sg}.

\begin{corollary}[Isotropic case]\label{cor:nld-isotropic}
Consider a linear observable $h(x)=u^\top x$, $u\in\R^d$.
Assume $P=I_d$, so that $H(x)=\frac12\|x\|^2$. Then for every $u\in\R^d$,
\begin{equation}\label{eq:V-nld-isotropic}
V_J(h)
=
2u^\top \left(I_d-J^2\right)^{-1}u.
\end{equation}
In contrast,
$V_0(h)=2\|u\|^2$.
Hence
$V_J(h)\le V_0(h)$
for every $u\in\R^d$, with strict inequality whenever $u$ has a nontrivial component in a direction on which $J\neq 0$.
\end{corollary}

\begin{proof}
If $P=I_d$, then $A=I_d+J$, and \eqref{eq:V-nld-general} gives
\[
V_J(h)=2u^\top (I_d+J)^{-1}\left((I_d+J)^{-1}\right)^\top u.
\]
Since $J^\top=-J$,
\[
(I_d+J)(I_d+J)^\top=(I_d+J)(I_d-J)=I_d-J^2,
\]
and therefore
\[
(I_d+J)^{-1}\left((I_d+J)^{-1}\right)^\top=\left(I_d-J^2\right)^{-1}.
\]
This proves \eqref{eq:V-nld-isotropic}. Since $J^2$ is negative semidefinite, $I_d-J^2\succeq I_d$, and hence
$(I_d-J^2)^{-1}\preceq I_d$,
which implies $V_J(h)\le V_0(h)$. The inequality is strict unless $u$ lies entirely in the nullspace of $J$.
This completes the proof.
\end{proof}

\begin{remark}\label{rem:nld-quadratic}
Corollary~\ref{cor:nld-isotropic} gives an explicit and fully quantitative version of the variance-reduction effect for NLD in the Gaussian case. In particular, the reduction factor depends on the strength and geometry of the anti-symmetric perturbation $J$ through the matrix $(I_d-J^2)^{-1}$.
\end{remark}

\subsection{The HFHR case}

We next consider the HFHR dynamics in Section~\ref{sec:examples} under a quadratic potential
$U(\theta)=\frac12 \theta^\top K\theta$,
where $K\in\R^{d\times d}$ is symmetric positive definite. 
Then, we have the following corollary.

\begin{corollary}\label{cor:HFHR}
Consider a linear position observable
$h(\theta,r)=u^\top\theta$, $u\in\R^d$.
Assume that $U(\theta)=\frac12 \theta^\top K\theta$ where $K\in\R^{d\times d}$ is symmetric positive definite. Then, 
\begin{equation}\label{eq:V-hfhr-closed}
V_{\mathrm{HFHR}}(h)
=
\frac{2\gamma}{1+\alpha\gamma}u^\top K^{-2}u,
\quad
\text{and}
\quad
V_{\mathrm{HFHR},0}(h)
=
\frac{2}{\alpha}u^\top K^{-2}u.
\end{equation}
Consequently,
\begin{equation}\label{eq:V-hfhr-ratio}
\frac{V_{\mathrm{HFHR}}(h)}{V_{\mathrm{HFHR},0}(h)}
=
\frac{\alpha\gamma}{1+\alpha\gamma}
<1.
\end{equation}

By Proposition~\ref{prop:quad-sg-linear-clt}, the same constants describe the stochastic gradient HFHR scheme \eqref{eq:hfhr-em} in the quadratic Gaussian setting under Assumptions~\ref{ass:diss} and~\ref{ass:sg}.

\end{corollary}

\begin{proof}
Under our assumption, 
$U(\theta)=\frac12 \theta^\top K\theta$ where $K\in\R^{d\times d}$ is symmetric positive definite.
Then, we have
\[
H(\theta,r)=\frac12\theta^\top K\theta+\frac12\|r\|^2
=
\frac12
\begin{pmatrix}
\theta\\ r
\end{pmatrix}^{\top}
\begin{pmatrix}
K & 0\\
0 & I_d
\end{pmatrix}
\begin{pmatrix}
\theta\\ r
\end{pmatrix}.
\]
Thus $P=
\begin{pmatrix}
K & 0\\
0 & I_d
\end{pmatrix}$,
$\mathcal D=
\begin{pmatrix}
\alpha I_d & 0\\
0 & \gamma I_d
\end{pmatrix}$,
$\mathcal Q=
\begin{pmatrix}
0 & -I_d\\
I_d & 0
\end{pmatrix}$,
and therefore
\begin{equation}\label{eq:A-hfhr}
A=(\mathcal D+\mathcal Q)P
=
\begin{pmatrix}
\alpha K & -I_d\\
K & \gamma I_d
\end{pmatrix}.
\end{equation}

For a position observable
$h(\theta,r)=u^\top\theta$,
$u\in\R^d$,
write
$\bar u=
\begin{pmatrix}
u\\
0
\end{pmatrix}
\in\R^{2d}$.
Then Proposition~\ref{prop:quadratic-linear} yields
\begin{equation}\label{eq:V-hfhr-general}
V_{\mathrm{HFHR}}(h)
=
2\bar u^\top A^{-1}\mathcal D \left(A^{-1}\right)^\top\bar u.
\end{equation}
In this case, the inverse of $A$ can be computed explicitly:
$A^{-1}
=
\frac{1}{1+\alpha\gamma}
\begin{pmatrix}
\gamma K^{-1} & K^{-1}\\
-I_d & \alpha I_d
\end{pmatrix}$.
Substituting this into \eqref{eq:V-hfhr-general} gives
\begin{equation}
V_{\mathrm{HFHR}}(h)
=
\frac{2\gamma}{1+\alpha\gamma}u^\top K^{-2}u.
\end{equation}

The reversible baseline is obtained by setting $\mathcal Q=0$, so that
$A_0=\mathcal D P
=
\begin{pmatrix}
\alpha K & 0\\
0 & \gamma I_d
\end{pmatrix}$.
Hence,
\begin{equation}
V_{\mathrm{HFHR},0}(h)
=
2\bar u^\top A_0^{-1}\mathcal D (A_0^{-1})^\top\bar u
=
\frac{2}{\alpha}u^\top K^{-2}u.
\end{equation}
Consequently, \eqref{eq:V-hfhr-ratio} holds.
%\begin{equation}
%\frac{V_{\mathrm{HFHR}}(h)}{V_{\mathrm{HFHR},0}(h)}
%=
%\frac{\alpha\gamma}{1+\alpha\gamma}
%<1.
%\end{equation}
This completes the proof.
\end{proof}

\begin{remark}
The baseline formula in \eqref{eq:V-hfhr-closed} agrees with the fluctuation constant of the overdamped Langevin dynamics in the $\theta$-variable, which is consistent with Remark~\ref{rem:hfhr-projected}. Thus, for quadratic targets and linear position observables, the variance reduction of HFHR relative to its reversible baseline can be interpreted directly as variance reduction relative to overdamped Langevin dynamics.
\end{remark}

\begin{remark}\label{rem:hfhr-quadratic}
Formulas \eqref{eq:V-hfhr-general}--\eqref{eq:V-hfhr-closed} provide a convenient quantitative diagnostic for the effect of the HFHR coupling. In particular, one may compare
$\frac{V_{\mathrm{HFHR}}(h)}{V_{\mathrm{HFHR},0}(h)}$
as a function of the parameters $\alpha$ and $\gamma$. By \eqref{eq:V-hfhr-ratio}, this ratio is exactly $\alpha\gamma/(1+\alpha\gamma)$, so that in the quadratic case the reduction factor is independent of the spectrum of $K$.
\end{remark}

\subsection{The GAUL case}

We now turn to the GAUL dynamics in Section~\ref{sec:examples}, again under the quadratic potential
$U(\theta)=\frac12 \theta^\top K\theta$,
with $K\in\R^{d\times d}$ symmetric positive definite. Then
$H(\theta,r)=\frac12\theta^\top K\theta+\frac12\|r\|^2$
and $P=
\begin{pmatrix}
K & 0\\
0 & I_d
\end{pmatrix}$.
Recall that
$\mathbf Q=
\begin{pmatrix}
aC & -C\\
I_d & \gamma I_d
\end{pmatrix}$,
$\mathcal D=\frac{\mathbf Q+\mathbf Q^\top}{2}$,
$\mathcal Q=\frac{\mathbf Q-\mathbf Q^\top}{2}$,
where $a,\gamma>0$ and $C$ is symmetric positive definite. Since $\mathcal D+\mathcal Q=\mathbf Q$, we have
\begin{equation}\label{eq:A-gaul}
A=(\mathcal D+\mathcal Q)P=\mathbf QP
=
\begin{pmatrix}
aCK & -C\\
K & \gamma I_d
\end{pmatrix}.
\end{equation}

\begin{corollary}[Quadratic GAUL for linear position observables]\label{cor:GAUL}
Consider the position observable
$h(\theta,r)=u^\top\theta$, $u\in\R^d$,
and assume that
$U(\theta)=\frac12\theta^\top K\theta$,
where $K\in\R^{d\times d}$ is symmetric positive definite. Let
$\mathbf Q=
\begin{pmatrix}
aC & -C\\
I_d & \gamma I_d
\end{pmatrix}$,
$\mathcal D=\frac{\mathbf Q+\mathbf Q^\top}{2}$,
$\mathcal Q=\frac{\mathbf Q-\mathbf Q^\top}{2}$,
where $a,\gamma>0$, $C$ is symmetric positive definite, and $\mathcal D$ is positive definite. Then
\begin{equation}\label{eq:V-gaul-closed-cor}
V_{\mathrm{GAUL}}(h)
=
\frac{2\gamma}{1+a\gamma}u^\top K^{-1}C^{-1}K^{-1}u,
\end{equation}
and
\begin{equation}\label{eq:V-gaul-baseline-closed-cor}
V_{\mathrm{GAUL},0}(h)
=
2u^\top K^{-1}S^{-1}K^{-1}u,
\qquad
S:=aC-\frac{1}{4\gamma}(I_d-C)^2.
\end{equation}
Under the assumptions of Theorem~\ref{thm:strict-vr}, it follows that
\[
V_{\mathrm{GAUL}}(h)\le V_{\mathrm{GAUL},0}(h).
\]
Moreover, the inequality is strict whenever the strict condition in Theorem~\ref{thm:strict-vr} holds.

By Proposition~\ref{prop:quad-sg-linear-clt}, the same constants describe the stochastic gradient GAUL scheme \eqref{eq:gaul-em} in the quadratic Gaussian setting under Assumptions~\ref{ass:diss} and~\ref{ass:sg}.

\end{corollary}

\begin{proof}
For the position observable
$h(\theta,r)=u^\top\theta$,
$u\in\R^d$,
define
$\bar u=
\begin{pmatrix}
u\\
0
\end{pmatrix}$.
By Proposition~\ref{prop:quadratic-linear},
\begin{equation}\label{eq:V-gaul-general}
V_{\mathrm{GAUL}}(h)
=
2\bar u^\top A^{-1}\mathcal D \left(A^{-1}\right)^\top\bar u,
\end{equation}
where
$A=(\mathcal D+\mathcal Q)P=\mathbf QP
=
\begin{pmatrix}
aCK & -C\\
K & \gamma I_d
\end{pmatrix}$.

A direct computation gives
$A^{-1}
=
\frac{1}{1+a\gamma}
\begin{pmatrix}
\gamma K^{-1}C^{-1} & K^{-1}\\
-C^{-1} & a I_d
\end{pmatrix}$.
Substituting this into \eqref{eq:V-gaul-general} yields
$V_{\mathrm{GAUL}}(h)
=
\frac{2\gamma}{1+a\gamma}u^\top K^{-1}C^{-1}K^{-1}u$.

For the reversible baseline, $A_0=\mathcal DP$. Since
$\mathcal D
=
\begin{pmatrix}
aC & \frac{I_d-C}{2}\\[1mm]
\frac{I_d-C}{2} & \gamma I_d
\end{pmatrix}$,
the Schur complement of the lower-right block is
$S:=aC-\frac{1}{4\gamma}(I_d-C)^2$.
Because $\mathcal D$ is positive definite, $S$ is positive definite as well. A block computation then yields
$V_{\mathrm{GAUL},0}(h)
=
2u^\top K^{-1}S^{-1}K^{-1}u$.
Finally, under the assumptions of Theorem~\ref{thm:strict-vr}, the inequality
$V_{\mathrm{GAUL}}(h)\le V_{\mathrm{GAUL},0}(h)$
follows from that theorem, and it is strict under the strict condition stated there.
\end{proof}

Unlike the HFHR case, the GAUL formulas do not immediately reduce to a scalar multiplicative factor. Therefore, the variance-reduction effect is not as transparent directly from \eqref{eq:V-gaul-closed-cor}--\eqref{eq:V-gaul-baseline-closed-cor}. Nevertheless, under the assumptions of Theorem~\ref{thm:strict-vr}, the general comparison theorem established earlier implies
$V_{\mathrm{GAUL}}(h)\le V_{\mathrm{GAUL},0}(h)$,
with strict inequality under the strict condition stated there. Thus, in the quadratic GAUL setting, the explicit formulas are consistent with the general variance-reduction principle, even though the comparison is less transparent than in the HFHR case.

\begin{remark}\label{rem:gaul-quadratic}
Formulas \eqref{eq:V-gaul-closed-cor} and \eqref{eq:V-gaul-baseline-closed-cor} show that the variance-reduction effect of GAUL can also be computed explicitly in the quadratic setting. In contrast to HFHR, the comparison does not reduce to a scalar factor independent of the geometry, because it depends on both $C^{-1}$ and the Schur complement $S^{-1}$. When $C=I_d$ and $a=\alpha$, these formulas reduce to the HFHR expressions \eqref{eq:V-hfhr-closed}.
\end{remark}

%%%%%%%%%%%%%%%%%%%%%%%%%%%%%%%%%%%%%%%%%%%%%%%%
\section{Numerical Experiments}\label{sec:numerical}

% rewrite the section to stochastic graident version and reorder the whole section. The old version is moved to numerical_junk_file.

In this section, we will first report stochastic gradient versions of three non-reversible samplers, then we will implement Bayesian linear regression with synthetic data and Bayesian logistic regression with real data.

\subsection{Basic Numerical Experiments}
\label{sec:basic:numerical}
We now report basic numerical experiments for the stochastic gradient versions of the three non-reversible samplers: SGNLMC, SGHFHRMC, and SGGAULMC. In all experiments, the stochastic gradient oracle is
$$
\widehat{\nabla U}(\theta_k,\zeta_{k+1})
=\nabla U(\theta_k)+\tau\zeta_{k+1},\qquad\zeta_{k+1}\sim \mathcal N(0,I_d),
$$
with $\tau=1$, where the oracle noise is independent of the Gaussian noise in the numerical scheme. We take $d=3$ and compute empirical RMSEs from $2000$ independent trajectories, each averaged over $N_\eta=\lfloor \eta^{-2}\rfloor$ iterates, with $\eta\in\left\{0.1,0.1\times 2^{-1},\ldots,0.1\times 2^{-6}\right\}$. For each algorithm, we consider both a quadratic test case $U(\theta)=\frac12\|\theta\|^2$, with a linear observable, and a nonquadratic test case $U(\theta)=\frac12\|\theta\|^2+\beta\sum_{i=1}^3\cos(\theta_i)$ and $\beta=1.2$ with a nonlinear observable. In both cases the observable is odd and the potential is even, so that $\mu(h)=0$.

\subsubsection{The SGNLMC case}
For SGNLMC, we use $J =
\begin{pmatrix}
0 & -1 & 1 \\
1 & 0 & -1\\
-1 & 1 & 0
\end{pmatrix}$, $u=\left(\frac{\sqrt2}{2},-\frac{\sqrt2}{2},0\right)^\top$, and take $h(\theta)=u^\top\theta$ in the quadratic experiment and
$h(\theta)=\tanh(u^\top\theta)$ in the nonquadratic experiment. The reversible baseline is obtained by setting $J=0$.

\begin{figure}[H]
    \centering
    \begin{subfigure}{0.48\linewidth}
        \centering
        \includegraphics[width=0.85\linewidth]{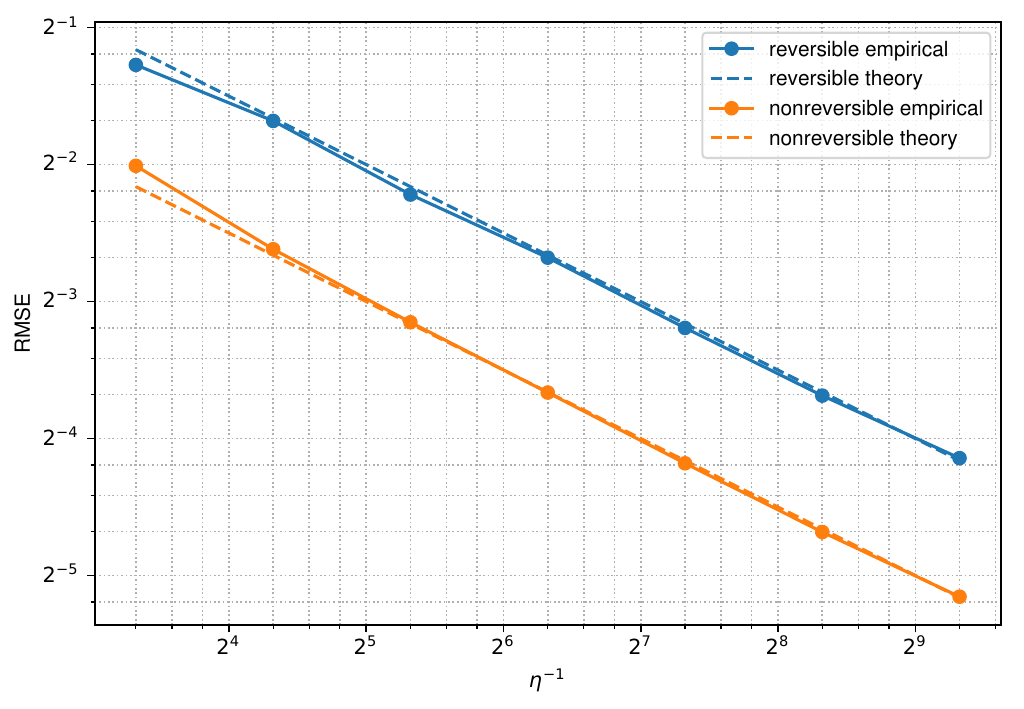}
        \caption{RMSE comparison.}
        \label{fig:nld_quadratic_loglog}
    \end{subfigure}
    \hfill
    \begin{subfigure}{0.48\linewidth}
        \centering
        \includegraphics[width=0.85\linewidth]{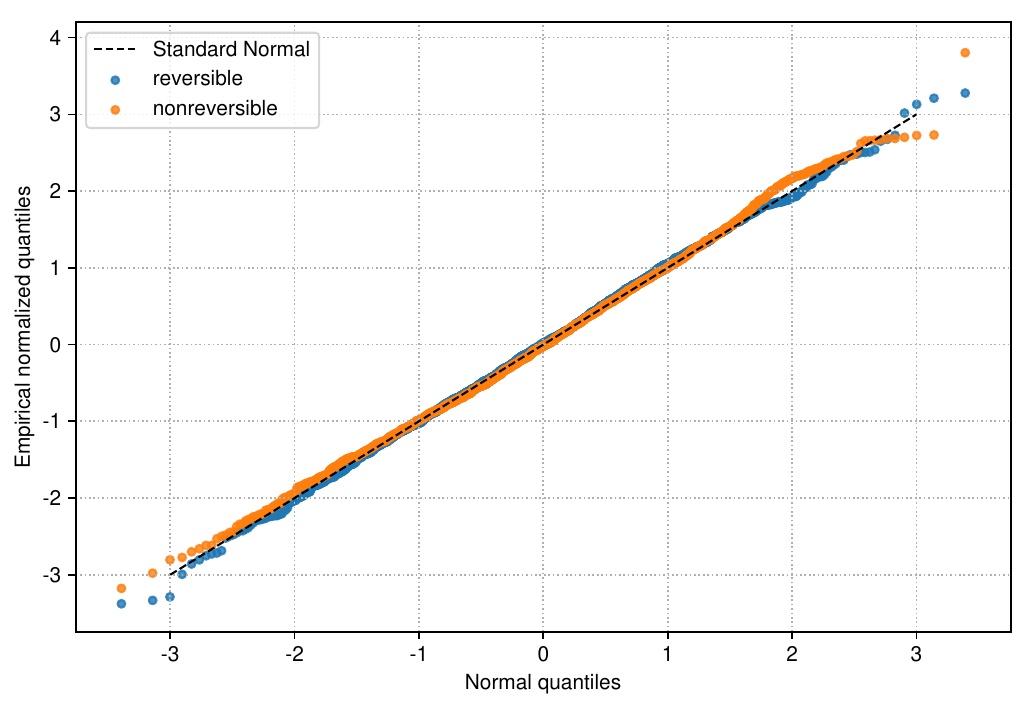}
        \caption{Q-Q plot for the normalized empirical errors.}
        \label{fig:nld_quadratic_qq_plot}
    \end{subfigure}
    \caption{Numerical results for SGNLMC with quadratic potential.}
    \label{fig:qq:nld}
\end{figure}

\begin{figure}[H]
    \centering
    \begin{subfigure}{0.48\linewidth}
        \centering
        \includegraphics[width=0.85\linewidth]{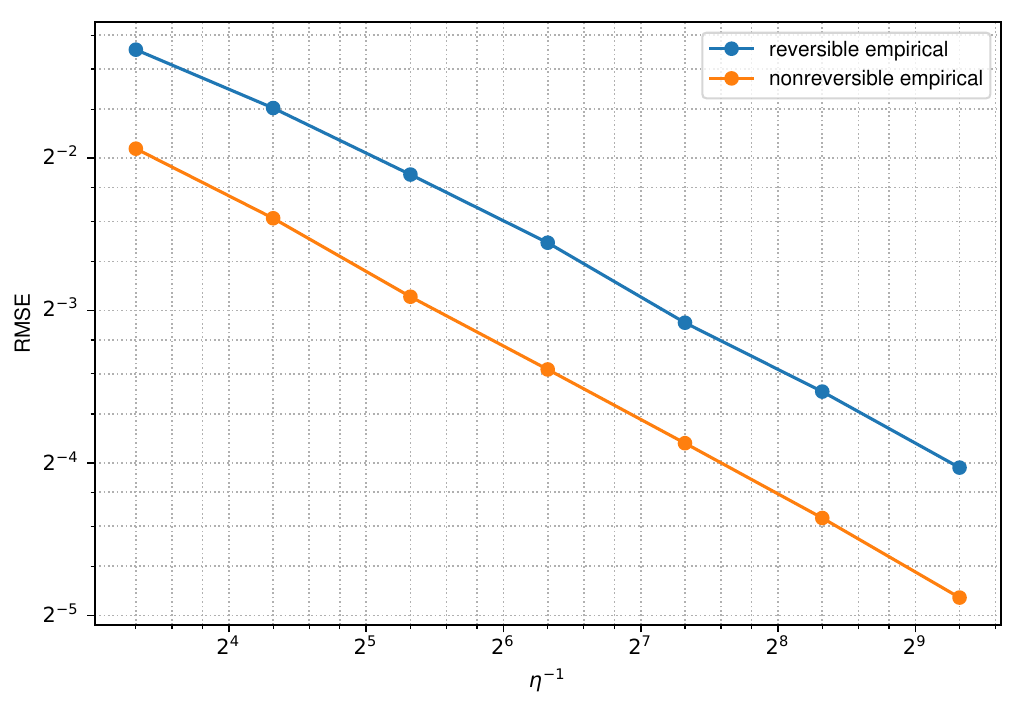}
        \caption{RMSE comparison.}
        \label{fig:nld_nonquadratic_loglog}
    \end{subfigure}
    \hfill
    \begin{subfigure}{0.48\linewidth}
        \centering
        \includegraphics[width=0.85\linewidth]{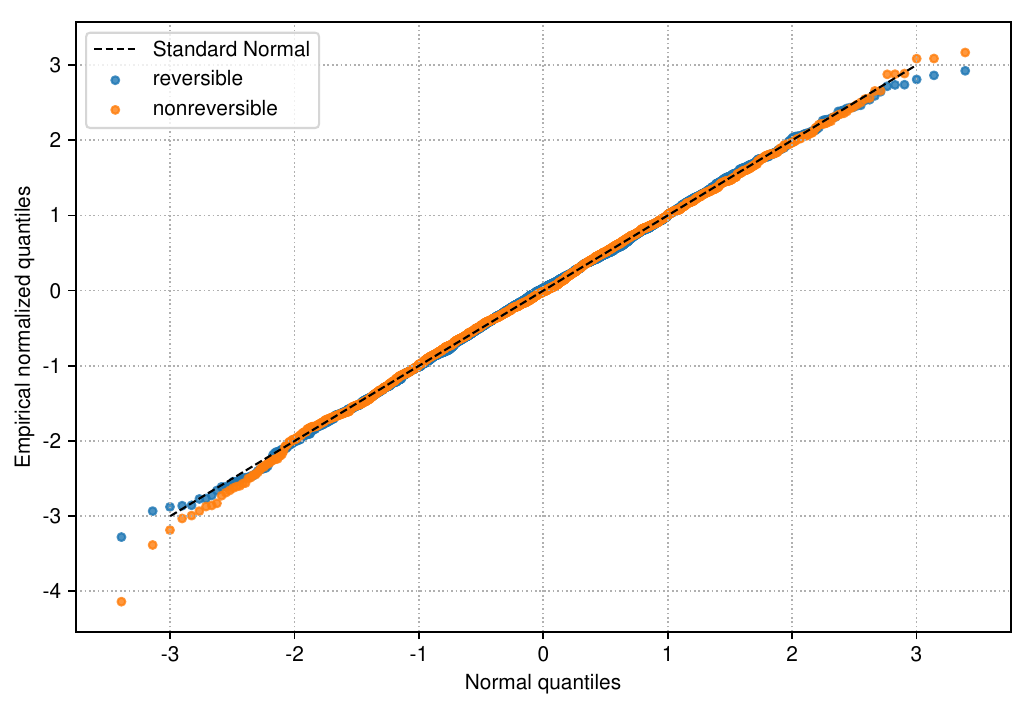}
        \caption{Q-Q plot for the studentized normalized empirical errors.}
        \label{fig:nld_nonquadratic_qq_plot}
    \end{subfigure}
    \caption{Numerical results for SGNLMC with nonquadratic potential.}
    \label{fig:qq:nld:2}
\end{figure}

\subsubsection{The SGHFHRMC case}
For SGHFHRMC, we take $\alpha=1$ and $\gamma=2$, with state variable $(\theta,r)\in\mathbb R^{2d}$. The observable is position only: $h(\theta,r)=u^\top\theta$ in the quadratic experiment and $h(\theta,r)=\tanh(u^\top\theta)$ in the nonquadratic experiment. The reversible baseline is given by the stochastic gradient version of Euler–Maruyama discretization of \eqref{eq:hfhr-baseline-sde}.

\begin{figure}[H]
    \centering
    \begin{subfigure}{0.48\linewidth}
        \centering
        \includegraphics[width=0.85\linewidth]{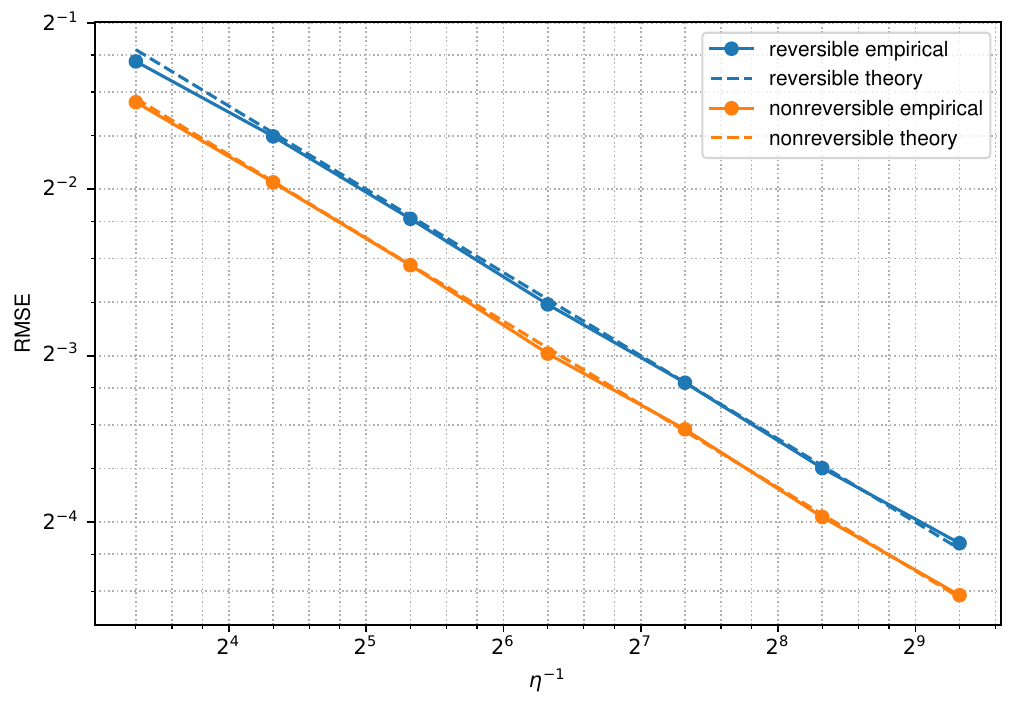}
        \caption{RMSE comparison.}
        \label{fig:hfhr_quadratic_loglog}
    \end{subfigure}
    \hfill
    \begin{subfigure}{0.48\linewidth}
        \centering
        \includegraphics[width=0.85\linewidth]{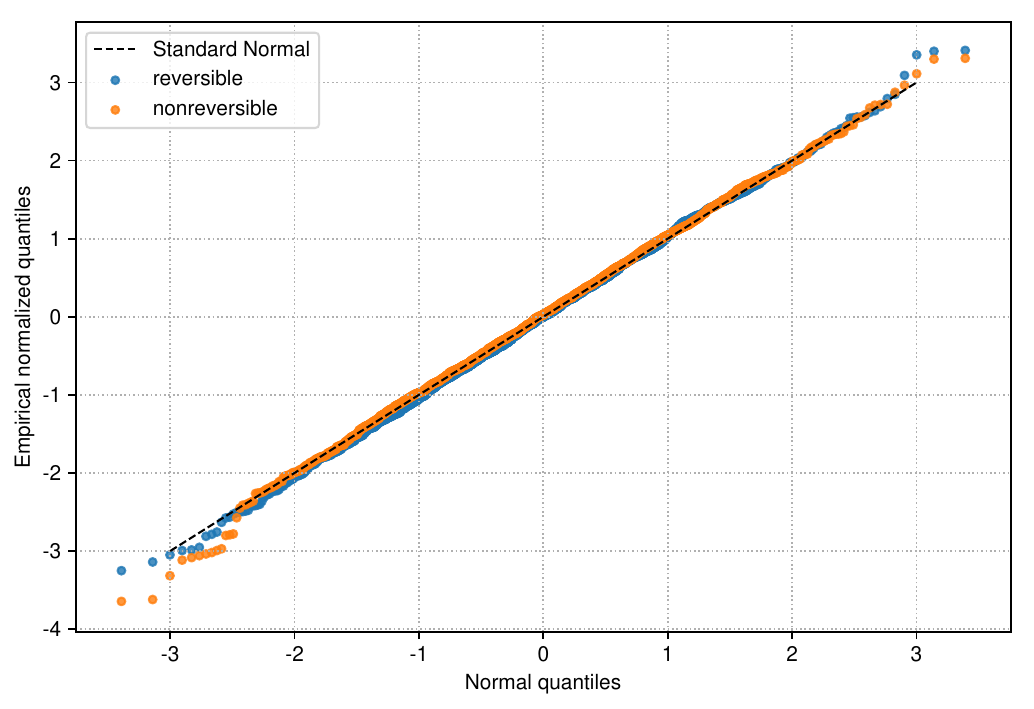}
        \caption{Q-Q plot for the normalized empirical errors.}
        \label{fig:hfhr_quadratic_qq_plot}
    \end{subfigure}
    \caption{Numerical results for SGHFHRMC with quadratic potential.}
    \label{fig:qq:hfhr}
\end{figure}

\begin{figure}[H]
    \centering
    \begin{subfigure}{0.48\linewidth}
        \centering
        \includegraphics[width=0.85\linewidth]{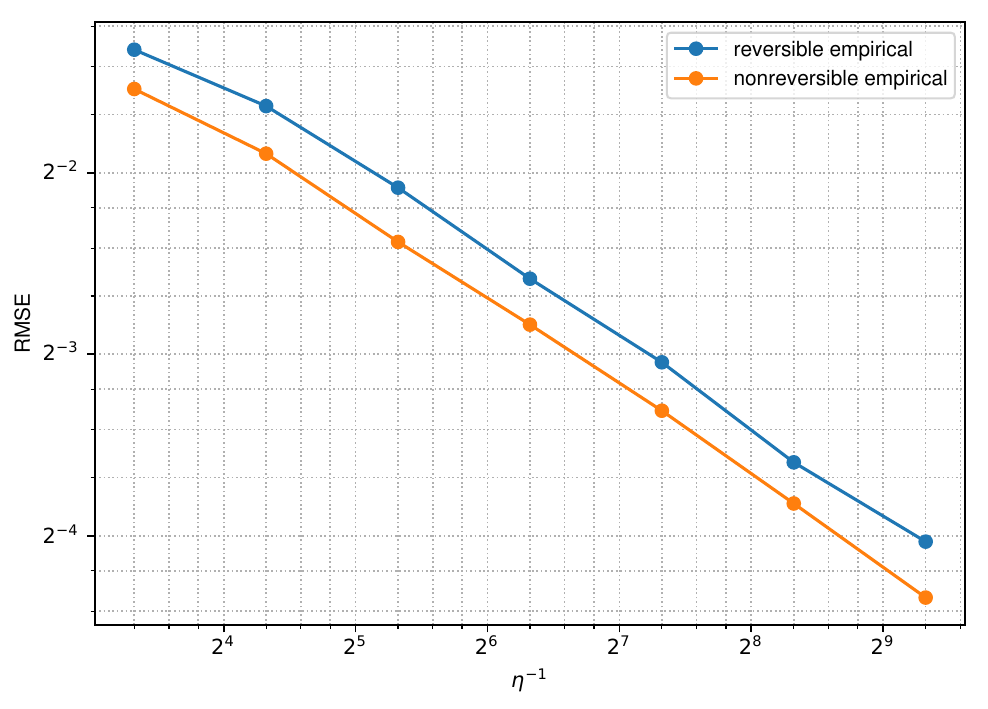}
        \caption{RMSE comparison.}
        \label{fig:hfhr_nonquadratic_loglog}
    \end{subfigure}
    \hfill
    \begin{subfigure}{0.48\linewidth}
        \centering
        \includegraphics[width=0.85\linewidth]{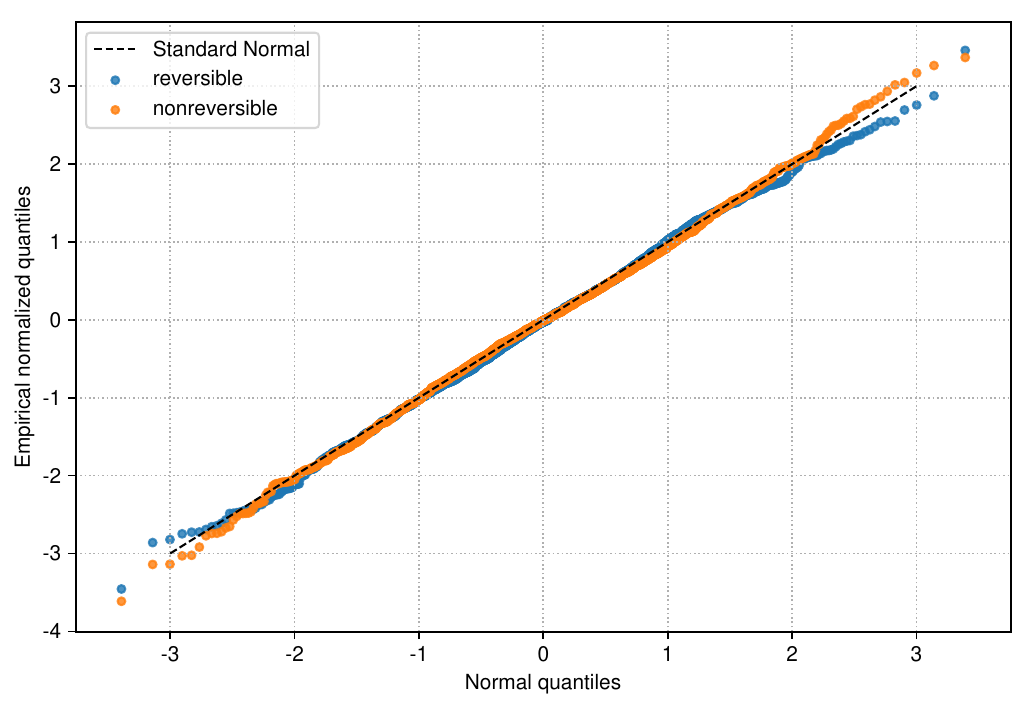}
        \caption{Q-Q plot for the studentized normalized empirical errors.}
        \label{fig:hfhr_nonquadratic_qq_plot}
    \end{subfigure}
    \caption{Numerical results for SGHFHRMC with nonquadratic potential.}
    \label{fig:qq:hfhr:2}
\end{figure}

\subsubsection{The SGGAULMC case}
For SGGAULMC, we take $a=1$, $\gamma=2$, and $C=
\begin{pmatrix}
\frac12 & 0 & 0\\
0 & \frac32 & 0\\
0 & 0 & 2
\end{pmatrix}$. Again, we use the position only observables $h(\theta,r)=u^\top\theta$ and $h(\theta,r)=\tanh(u^\top\theta)$ for the quadratic and nonquadratic experiments, respectively. The reversible baseline is defined by Euler–Maruyama discritization in the stochastic gradient version of \eqref{eq:gaul-baseline-sde}.

\begin{figure}[H]
    \centering
    \begin{subfigure}{0.48\linewidth}
        \centering
        \includegraphics[width=0.85\linewidth]{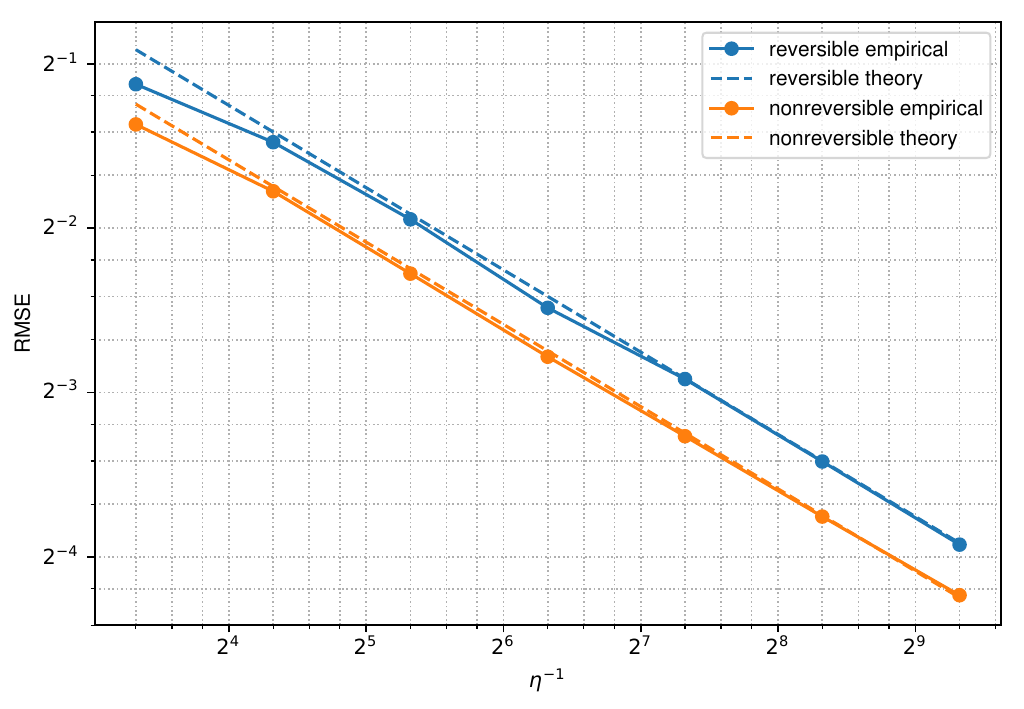}
        \caption{RMSE comparison.}
        \label{fig:gaul_quadratic_loglog}
    \end{subfigure}
    \hfill
    \begin{subfigure}{0.48\linewidth}
        \centering
        \includegraphics[width=0.85\linewidth]{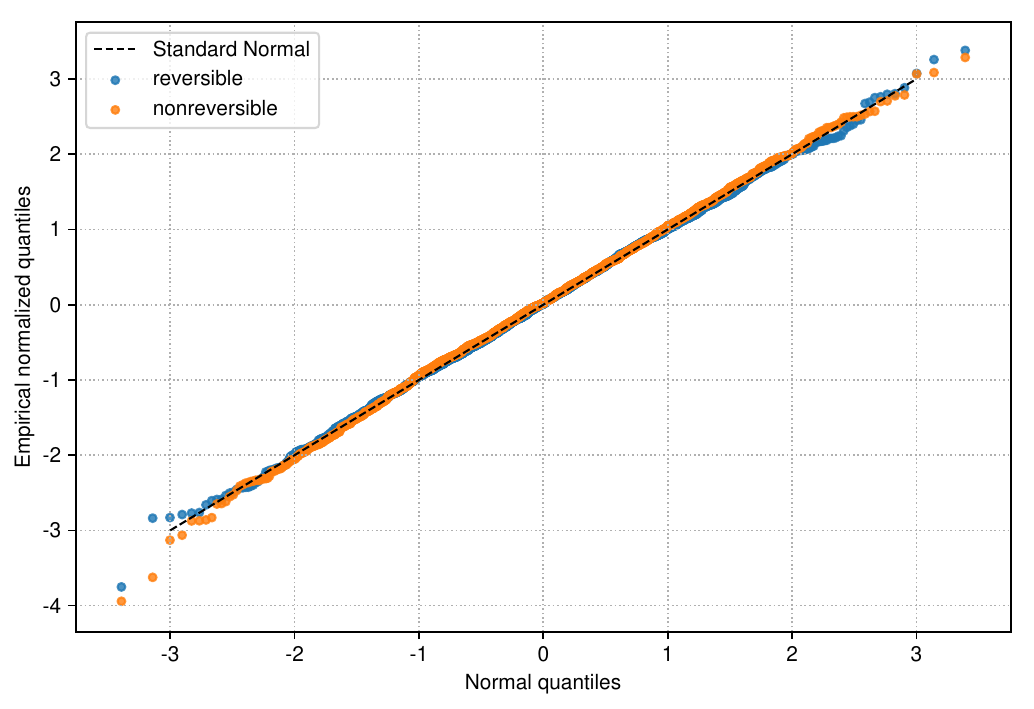}
        \caption{Q-Q plot for the normalized empirical errors.}
        \label{fig:gaul_quadratic_qq_plot}
    \end{subfigure}
    \caption{Numerical results for SGGAULMC with quadratic potential.}
    \label{fig:qq:gaul}
\end{figure}

\begin{figure}[H]
    \centering
    \begin{subfigure}{0.48\linewidth}
        \centering
        \includegraphics[width=0.85\linewidth]{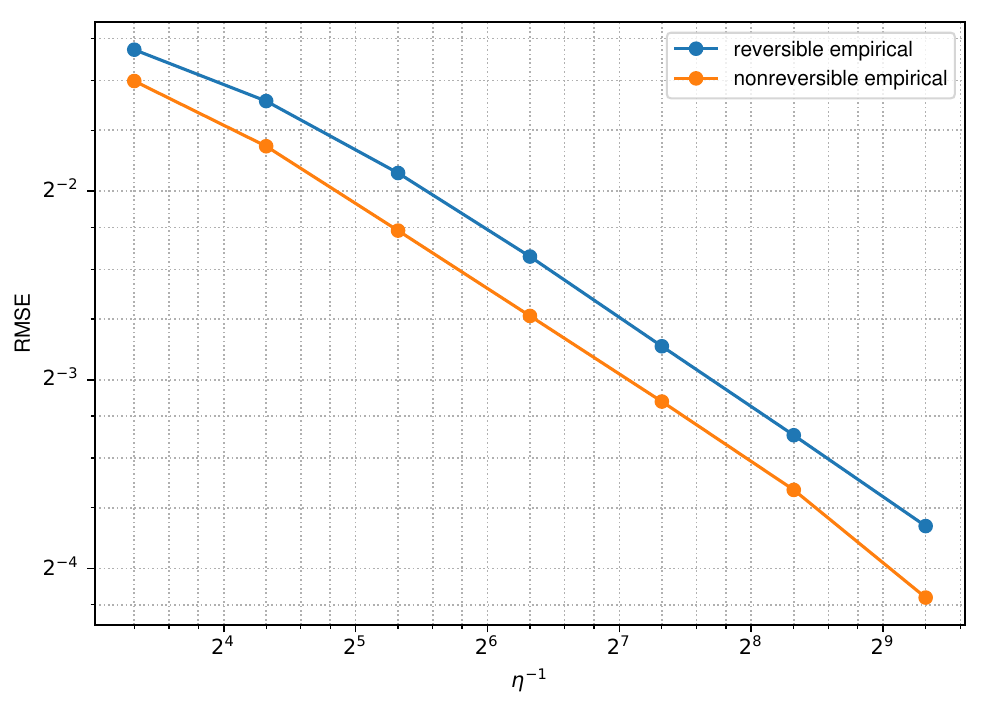}
        \caption{RMSE comparison.}
        \label{fig:gaul_nonquadratic_loglog}
    \end{subfigure}
    \hfill
    \begin{subfigure}{0.48\linewidth}
        \centering
        \includegraphics[width=0.85\linewidth]{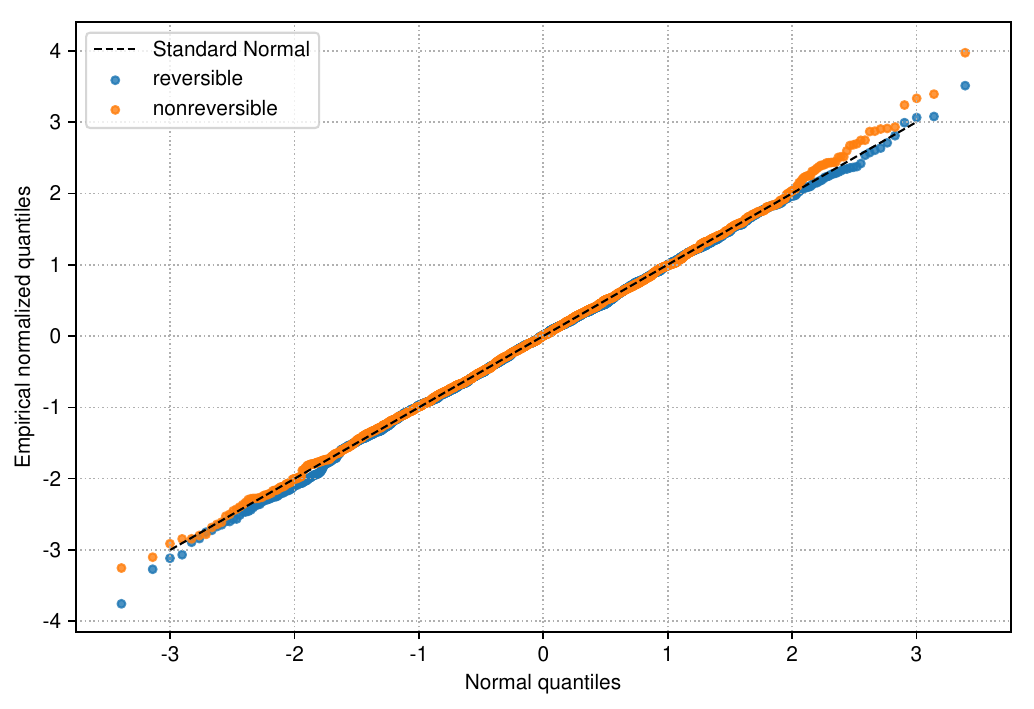}
        \caption{Q-Q plot for the studentized normalized empirical errors.}
        \label{fig:gaul_nonquadratic_qq_plot}
    \end{subfigure}
    \caption{Numerical results for SGGAULMC with nonquadratic potential.}
    \label{fig:qq:gaul:2}
\end{figure}

Figures~\ref{fig:qq:nld}, \ref{fig:qq:hfhr}, and \ref{fig:qq:gaul} summarize the quadratic experiments. In all three algorithms, the non-reversible sampler achieves a smaller RMSE than its reversible baseline across the tested stepsizes. Moreover, the empirical RMSE curves closely follow the theoretical reference lines determined by the corresponding fluctuation constants. The Q-Q plots, computed at the smallest stepsize $\eta=0.1\times 2^{-6}$, show that the normalized empirical errors are close to the standard Gaussian reference line, supporting the predicted fluctuation limit.

Figures~\ref{fig:qq:nld:2}, \ref{fig:qq:hfhr:2}, and \ref{fig:qq:gaul:2} summarize the nonquadratic experiments. Since closed-form fluctuation constants are not available in this setting, we do not include theoretical dashed reference lines, and we normalize the Q-Q errors using the empirical proxy $\widehat V_\eta(h)=\widehat{\mathrm{MSE}}_\eta(h)/\eta$.

The RMSE plots again show a consistent variance-reduction effect: the non-reversible algorithms have smaller RMSEs than their reversible baselines, and in the small-stepsize regime the two curves are approximately parallel on the log-log scale, with the non-reversible curve having a lower vertical offset. The studentized Q-Q plots remain close to the standard Gaussian reference line, indicating that the same asymptotic normal behavior is observed beyond the quadratic setting.

%%%%%%%%%%%%%%%%%%%%%%%%%%%%%%%%%%%%

\subsection{Bayesian Linear Regression with Synthetic Data}\label{sec:Bayesian:linear}

\subsubsection{General Settings}
We consider a Bayesian linear regression model with covariate $x\in\mathbb R^d$, response $y\in\mathbb R$, and
$$
    y=x^\top\beta+\varepsilon, \quad
    \varepsilon\sim\mathcal N(0,1),
$$
where $\beta\in\mathbb R^d$ is the regression coefficient. Throughout this experiment we take $d=3$ and impose the isotropic Gaussian prior
$
    \beta\sim\mathcal N(0,\sigma_0^2I_d)$, $
    \sigma_0=2$.
Given $N$ independent observations $\{(x_i,y_i)\}_{i=1}^N$, with design matrix $X=(x_1,\ldots,x_N)^\top$ and response vector $Y=(y_1,\ldots,y_N)^\top$, the posterior density of $\beta$ is proportional to $\exp\{-U(\beta)\}$, where
\begin{align}
    U(\beta)
    =
    \frac12\sum_{i=1}^N(y_i-x_i^\top\beta)^2
    +
    \frac{1}{2\sigma_0^2}\|\beta\|^2 .
    \label{log-posterior potential}
\end{align}
Equivalently,
\begin{align}
    U(\beta)
    =
    \frac12(\beta-m_N)^\top K_N(\beta-m_N)+\mathrm{const},
    \quad
    K_N=X^\top X+\sigma_0^{-2}I_d,
    \quad
    m_N=K_N^{-1}X^\top Y .
    \label{eq:linear regression quadratic potential}
\end{align}

We generate a synthetic dataset by sampling $x_{ij}\sim \mathrm{Unif}[-1/2,1/2]$ and
$$
    y_i=x_i^\top\beta^\star+\varepsilon_i \quad \varepsilon_i\sim\mathcal N(0,1)\quad 
    \beta^\star=(1,-1,0.5)^\top,
$$
with $N=128$. After the dataset is fixed, the posterior mean of the test functional
$$
    h_\beta(\beta)=x_{\rm test}^\top\beta, \quad x_{\rm test}=\left(\frac{\sqrt2}{2},-\frac{\sqrt2}{2},0\right)^\top,
$$
is available in closed form, $\mu(h_\beta)=x_{\rm test}^\top m_N$. For augmented samplers with state $(\beta,r)$, we use the same position only observable $h(\beta,r)=x_{\rm test}^\top\beta$.

All three algorithms are implemented with a mini-batch stochastic gradient oracle. The full gradient is
$
    \nabla U(\beta)
    =
    \sum_{i=1}^N(x_i^\top\beta-y_i)x_i+\sigma_0^{-2}\beta .
$
At each Euler step and for each independent trajectory, we sample a mini-batch $B_k\subset\{1,\ldots,N\}$ of size $b=16$, with replacement, independently across steps and replications, and set
\begin{align}
    \widehat{\nabla U}(\beta,\zeta_{k+1})
    =
    \frac{N}{b}\sum_{i\in B_k}(x_i^\top\beta-y_i)x_i+\sigma_0^{-2}\beta .
    \label{eq:linear regression oracle beta}
\end{align}
This estimator is conditionally unbiased, it follows $\mathbb E\left[\widehat{\nabla U}(\beta,\zeta_{k+1})\mid \beta\right]=\nabla U(\beta)$.
The quadratic representation in \eqref{eq:linear regression quadratic potential} allows the three posterior sampling experiments to be reduced to the quadratic setting studied earlier. Indeed, after the translation $z=\beta-m_N$, we have
$$
    U(\beta)=U_z(z)+\mathrm{const},
    \qquad
    U_z(z)=\frac12 z^\top K_Nz,
    \qquad
    \nabla U(\beta)=K_Nz .
$$
Moreover,
$
    x_{\rm test}^\top\beta
    =
    x_{\rm test}^\top m_N+x_{\rm test}^\top z,
$
so the fluctuation constant for $h_\beta(\beta)$ is the same as that for the centered observable $h_z(z)=x_{\rm test}^\top z$. Therefore, for SGNLMC the constants can be evaluated explicitly as
\begin{align}
    V_{\rm NLD}(h_z)
    =
    2x_{\rm test}^\top A_J^{-1}(A_J^{-1})^\top x_{\rm test},
    \qquad
    A_J=(I_d+J)K_N,
    \label{eq:linear_regression_nld_potential_nonreversible}
\end{align}
and, for the reversible baseline $J=0$,
\begin{align}
    V_{{\rm NLD},0}(h_z)
    =
    2x_{\rm test}^\top K_N^{-2}x_{\rm test}.
    \label{eq:linear_regression_nld_potential_reversible}
\end{align}
For each tested stepsize, the time-average estimator is computed over $N_\eta=\lfloor\eta^{-2}\rfloor$ iterates, and RMSEs are estimated from $2000$ independent trajectories. In the Q-Q plots, we use the smallest tested stepsize for each algorithm and normalize the errors by $Z_{\eta,r}=\frac{\widehat\mu_{\eta,r}(h_\beta)-\mu(h_\beta)}{\eta^{1/2}\sqrt{V(h)}}$, where $V(h)$ denotes the appropriate non-reversible or reversible fluctuation constant. For SGHFHRMC and SGGAULMC, the corresponding non-reversible and reversible fluctuation constants are computed from the closed-form quadratic formulas in Section~\ref{sec:quadratic}. 

\subsubsection{The SGNLMC case}
For SGNLMC, we use the stochastic gradient NLD discretization in~\eqref{eq:nld-example-em} with state $\beta_{k}$ and take $J=0.3377
\begin{pmatrix}
0 & -1 & 1\\
1 & 0 & -1\\
-1 & 1 & 0
\end{pmatrix}.$ The reversible baseline is obtained by setting $J=0$. The tested stepsizes are $\eta \in \{0.05,\ 0.025,\ 0.0125,\ 0.00625,\ 0.003125,\ 0.0015625\}.$ We get the following plots.
\begin{figure}[H]
    \centering
    \begin{subfigure}{0.48\linewidth}
        \centering
        \includegraphics[width=0.85\linewidth]{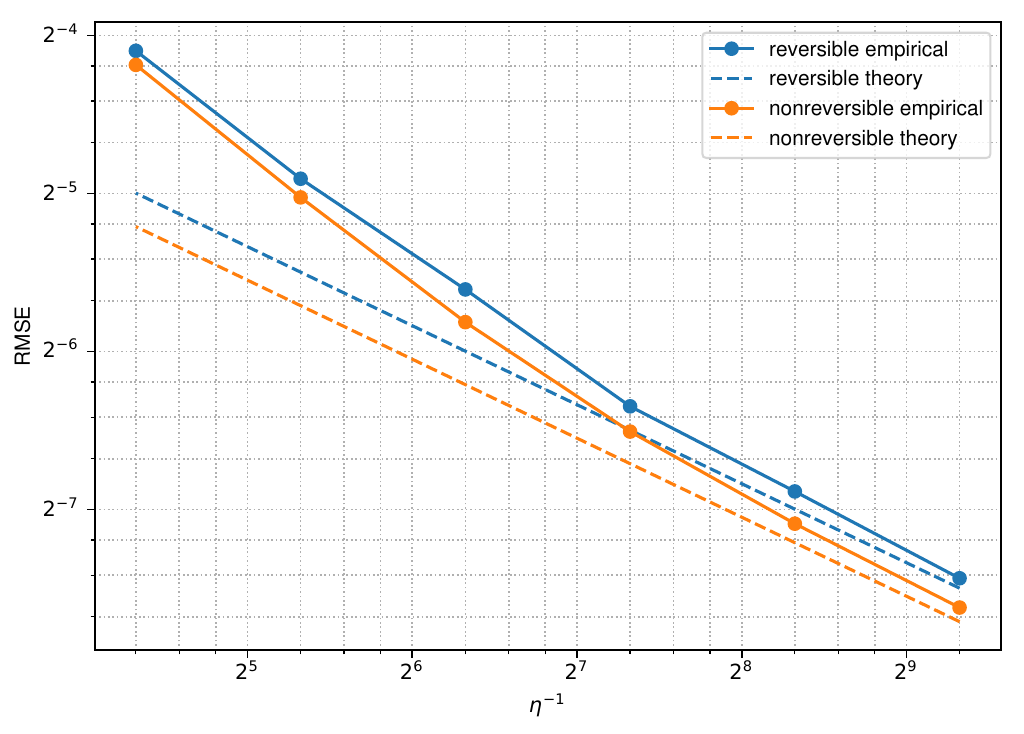}
        \caption{RMSE comparison.}
        \label{fig:linear_regression_nld_loglog}
    \end{subfigure}
    \hfill
    \begin{subfigure}{0.48\linewidth}
        \centering
        \includegraphics[width=0.85\linewidth]{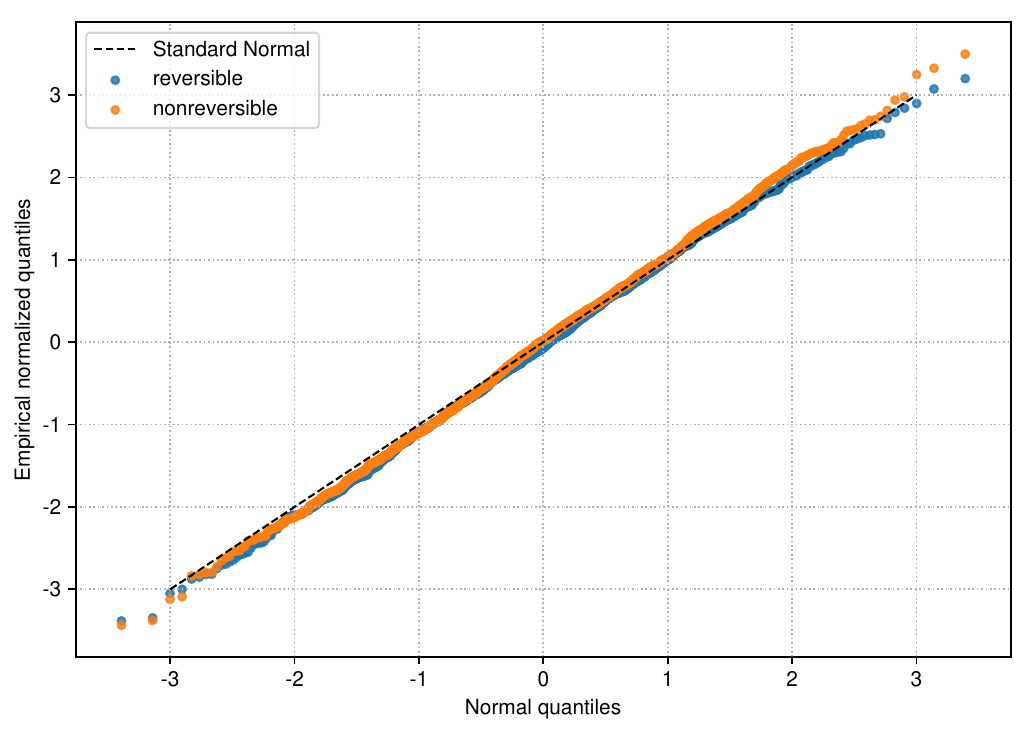}
        \caption{Q-Q plot for normalized errors.}
        \label{fig:linear_regression_nld_qq}
    \end{subfigure}
    \caption{Bayesian linear regression using SGNLMC. The RMSE is computed relative to $\mu(h_\beta)=x_{\rm test}^\top m_N$. The Q-Q plot uses $\eta=0.0015625$.}
    \label{fig:linear:reg:nld}
\end{figure}

\subsubsection{The SGHFHRMC case}
For SGHFHRMC, we use the stochastic graident SGHFHR discretization in~\eqref{eq:hfhr-em} with state $(\beta_k,r_k)$. In the experiment, $\alpha=0.5$ and $\gamma=6.0$. The reversible baseline is obtained by removing the anti-symmetric coupling while keeping the same symmetric diffusion matrix. The tested stepsizes are $\eta \in \{0.025,\ 0.0125,\ 0.00625,\ 0.003125,\ 0.0015625,\ 0.00078125\}$. We get the following plots.
\begin{figure}[H]
    \centering
    \begin{subfigure}{0.48\linewidth}
        \centering
        \includegraphics[width=0.85\linewidth]{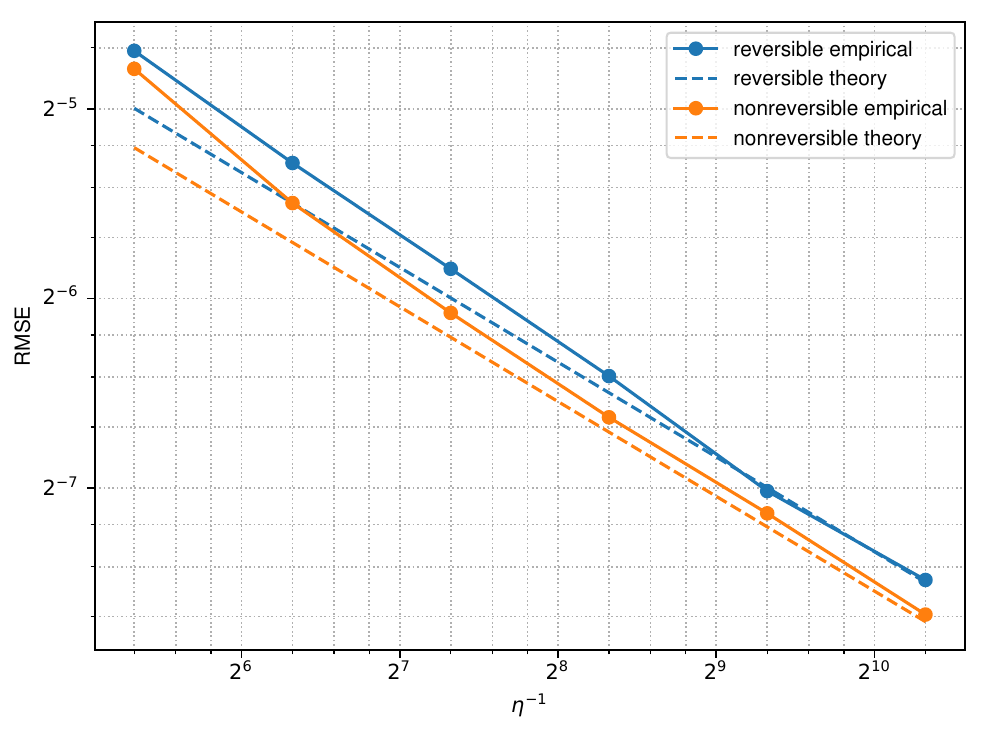}
        \caption{RMSE comparison.}
        \label{fig:linear_regression_hfhr_loglog}
    \end{subfigure}
    \hfill
    \begin{subfigure}{0.48\linewidth}
        \centering
        \includegraphics[width=0.85\linewidth]{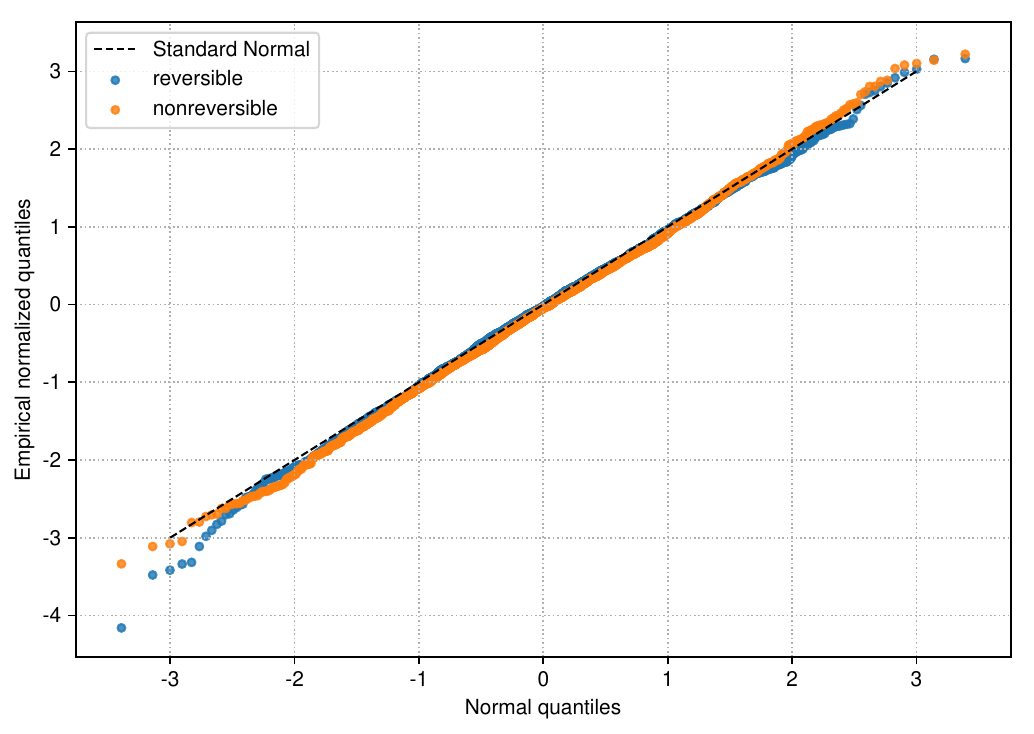}
        \caption{Q-Q plot for normalized errors.}
        \label{fig:linear_regression_hfhr_qq}
    \end{subfigure}
    \caption{Bayesian linear regression using SGHFHRMC. The RMSE is computed relative to $\mu(h_\beta)=x_{\rm test}^\top m_N$. The Q-Q plot uses $\eta=0.00078125$.}
    \label{fig:linear:reg:hfhr}
\end{figure}

\subsubsection{The SGGAULMC case}
For SGGAULMC, we use the stochastic gradient GAUL discretization \eqref{eq:gaul-em} with state $(\beta_k,r_k)$ and parameters $C=2I_d, a=0.3,\gamma=4.0$. The reversible baseline is obtained by setting the anti-symmetric part to zero while keeping the same symmetric diffusion matrix $\mathcal D$. The tested stepsizes are $\eta \in \{ 0.025,\ 0.0125,\ 0.00625,\ 0.003125,\ 0.0015625\}$. We get the following plots.

\begin{figure}[H]
    \centering
    \begin{subfigure}{0.48\linewidth}
        \centering
        \includegraphics[width=0.85\linewidth]{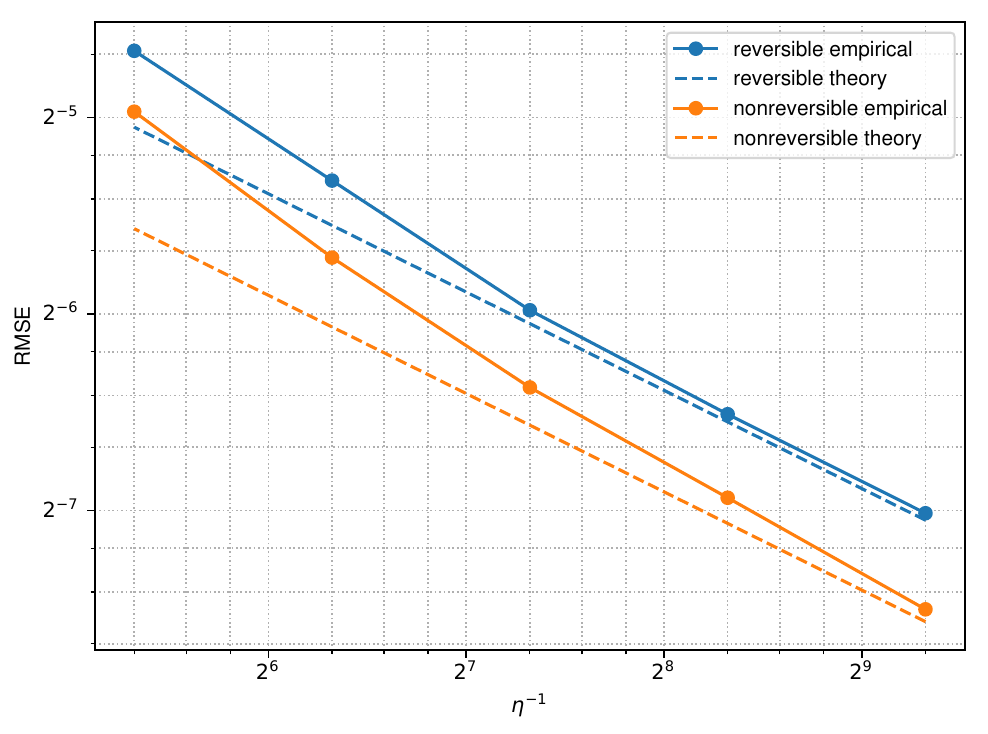}
        \caption{RMSE comparison.}
        \label{fig:linear_regression_gaul_loglog}
    \end{subfigure}
    \hfill
    \begin{subfigure}{0.48\linewidth}
        \centering
        \includegraphics[width=0.85\linewidth]{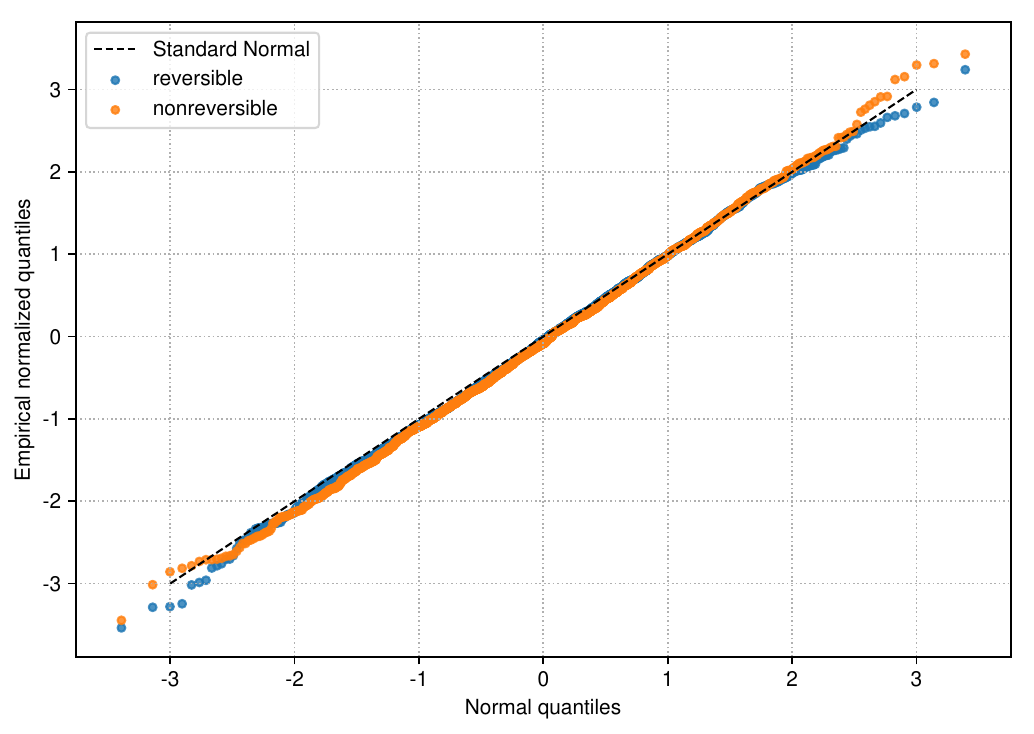}
        \caption{Q-Q plot for normalized errors.}
        \label{fig:linear_regression_gaul_qq}
    \end{subfigure}
    \caption{Bayesian linear regression using SGGAULMC. The RMSE is computed relative to $\mu(h_\beta)=x_{\rm test}^\top m_N$. The Q-Q plot uses $\eta=0.0015625$.}
    \label{fig:linear:reg:gaul}
\end{figure}

Figures~\ref{fig:linear:reg:nld}, \ref{fig:linear:reg:hfhr}, and \ref{fig:linear:reg:gaul} validate the fluctuation constant predictions for Bayesian linear regression. Across all three samplers, the non-reversible scheme achieves a smaller RMSE than the corresponding reversible baseline. As $\eta$ decreases, the empirical RMSE curves approach the theoretical reference lines determined by the quadratic fluctuation constants, confirming that the translated posterior potential falls within the quadratic theory.

The Q-Q plots in Figures~\ref{fig:linear:reg:nld}, \ref{fig:linear:reg:hfhr}, and \ref{fig:linear:reg:gaul} show the normalized empirical errors at the smallest tested stepsizes. In each case, the empirical quantiles for both the non-reversible scheme and the reversible baseline lie close to the standard normal reference line $y=x$, supporting the predicted Gaussian fluctuation limit.

%%%%%%%%%%%%%%%%%%%%%%%%%%%%%%%%%%%

\subsection{Bayesian Logistic Regression with Real Dataset}
\label{sec:real:data:numerical}

\subsubsection{General Settings}
We apply the stochastic gradient sampling algorithms to the Heart Disease (Statlog) dataset from the UCI repository~\cite{uci_statlog_heart}. The dataset contains 270 patient records and 13 clinical predictors, including age, sex, chest pain type, resting blood pressure, serum cholesterol, fasting blood sugar, resting ECG results, maximum heart rate achieved, exercise induced angina, ST depression, slope of the peak exercise ST segment, number of major vessels, and thalassemia status. The response is binary: class 0 indicates absence of heart disease, and class 1 indicates presence of heart disease. We use a stratified split with 202 observations for training and 68 observations for testing, so that the class proportions are preserved. Predictors are standardized using only the training set mean and standard deviation, and an intercept is appended to the design matrix.

We consider Bayesian logistic regression on the standardized training data. For each observation $i=1,\ldots,N$, let $x_i\in\mathbb R^p$ denote the standardized covariate vector and $y_i\in\{0,1\}$ denote the binary response. After adding an intercept term, we write
$$
    \bar x_i=(1,x_i^\top)^\top\in\mathbb R^{p+1}.
$$
Given $\theta\in\mathbb R^{p+1}$, the positive class probability is modeled as
$$
    \mathbb P(y_i=1\mid \bar x_i,\theta)
    =
    \rho(\bar x_i^\top\theta),
    \qquad
    \rho(t)=\frac{1}{1+e^{-t}}.
$$
Thus
$$
    p(y_i\mid \bar x_i,\theta)
    =
    \rho(\bar x_i^\top\theta)^{y_i}
    \left(1-\rho(\bar x_i^\top\theta)\right)^{1-y_i}
    =
    \frac{\exp(y_i\bar x_i^\top\theta)}
    {1+\exp(\bar x_i^\top\theta)}.
$$
With the standard Gaussian prior $\theta\sim\mathcal N(0,I_{p+1})$, the posterior density has Gibbs form
$$
    \pi(d\theta\mid\mathcal D)
    \propto
    e^{-U(\theta)}d\theta,
$$
where
$$
    U(\theta)
    =
    \sum_{i=1}^N
    \left[
        \log\left(1+\exp(\bar x_i^\top\theta)\right)
        -
        y_i\bar x_i^\top\theta
    \right]
    +
    \frac12\|\theta\|^2.
$$
Unlike the Bayesian linear regression experiment, this posterior potential is not quadratic, and neither the posterior expectation nor the fluctuation constant is available in closed form.

For prediction, we fix an augmented test feature vector $\bar x_{\rm test}\in\mathbb R^{p+1}$ selected from the test set. The Bayesian posterior predictive probability is
$$
    \mathbb P(y_{\rm test}=1\mid \bar x_{\rm test},\mathcal D)
    =
    \int_{\mathbb R^{p+1}}
    \rho(\bar x_{\rm test}^\top\theta)
    \pi(d\theta\mid\mathcal D).
$$
Thus the observable is
$
    h(\theta)=\rho(\bar x_{\rm test}^\top\theta).
$
For augmented samplers with state $(\theta,r)$, we use the same position only observable
$
    h(\theta,r)=\rho(\bar x_{\rm test}^\top\theta).
$
This bounded smooth observable belongs to the predictive observable class discussed in Proposition~\ref{prop:predictive-observable}.

The posterior potential has a finite sum structure
$$
    U(\theta)
    =
    \sum_{i=1}^N \ell_i(\theta)
    +
    \frac12\|\theta\|^2,
    \qquad
    \ell_i(\theta)
    =
    \log\left(1+\exp(\bar x_i^\top\theta)\right)
    -
    y_i\bar x_i^\top\theta .
$$
Hence
$
    \nabla U(\theta)
    =
    \sum_{i=1}^N
    \left[
        \rho(\bar x_i^\top\theta)-y_i
    \right]\bar x_i
    +
    \theta .
$
At each Euler step and for each independent trajectory, we sample a mini-batch $B_k\subset\{1,\ldots,N\}$ of size $b=32$ with replacement, independently across steps and replications. The stochastic gradient oracle is
\begin{align}
    \widehat{\nabla U}(\theta,\zeta_{k+1})
    =
    \frac{N}{b}
    \sum_{i\in B_k}
    \left[
        \rho(\bar x_i^\top\theta)-y_i
    \right]\bar x_i
    +
    \theta .
    \label{eq:logistic_stochastic_gradient_oracle}
\end{align}
The prior gradient term is included exactly, and only the likelihood gradient is estimated by mini-batching. The oracle is conditionally unbiased such that
$
    \mathbb E\!\left[
        \widehat{\nabla U}(\theta,\zeta_{k+1})
        \mid \theta
    \right]
    =
    \nabla U(\theta).
$
The oracle randomness is independent of the Gaussian noise in the discretized sampling schemes.

Since $\pi(h)$ is not available in closed form, we use an empirical reference for each dynamics. Specifically, for each sampler and its reversible baseline, we approximate $\pi(h)$ by averaging the $R=2000$ trajectory estimates at the smallest tested stepsize $\eta_{\min}=0.00125$. RMSEs are computed relative to this empirical reference. For each tested stepsize, the estimator is averaged over
$
    N_\eta=\lfloor \eta^{-2}\rfloor
$
iterates, and RMSEs are computed from $2000$ independent trajectories. In the normal quantile plots, we use the smallest tested stepsize and normalize the trajectory level errors by
$
    Z_{\eta,r}^{d}
    =
    \frac{\widehat\mu_{\eta,r}^{d}(h)-\mu_{\rm ref}^{d}}
    {\widehat{\mathrm{RMSE}}_{\eta}^{d}},
$
where $d$ denotes the chosen dynamics, $\mu_{\rm ref}^{d}$ is the empirical reference, and $\widehat{\mathrm{RMSE}}_{\eta}^{d}$ is computed from the trajectory level errors at the same stepsize. All three samplers use the stepsizes
$
 \eta \in \{   0.01,\ 0.005,\ 0.0025,\ 0.00125 \}.
$

\subsubsection{The SGNLMC case}
For SGNLMC, with dimension $p+1=14$, we use the stochastic gradient non-reversible Langevin Monte Carlo scheme \eqref{eq:nld-example-em}. The reversible baseline is obtained by setting $J=0$. In the non-reversible sampler, we take
$
    J=\rho J_0,
    \rho=1.0113,
$
where $J_0\in\mathbb R^{(p+1)\times(p+1)}$ is block diagonal with $2\times2$ antisymmetric blocks
$
    \begin{pmatrix}
        0 & -1\\
        1 & 0
    \end{pmatrix}.
$ We can get the following plots.

\begin{figure}[H]
    \centering
    \begin{subfigure}{0.48\linewidth}
        \centering
        \includegraphics[width=0.85\linewidth]{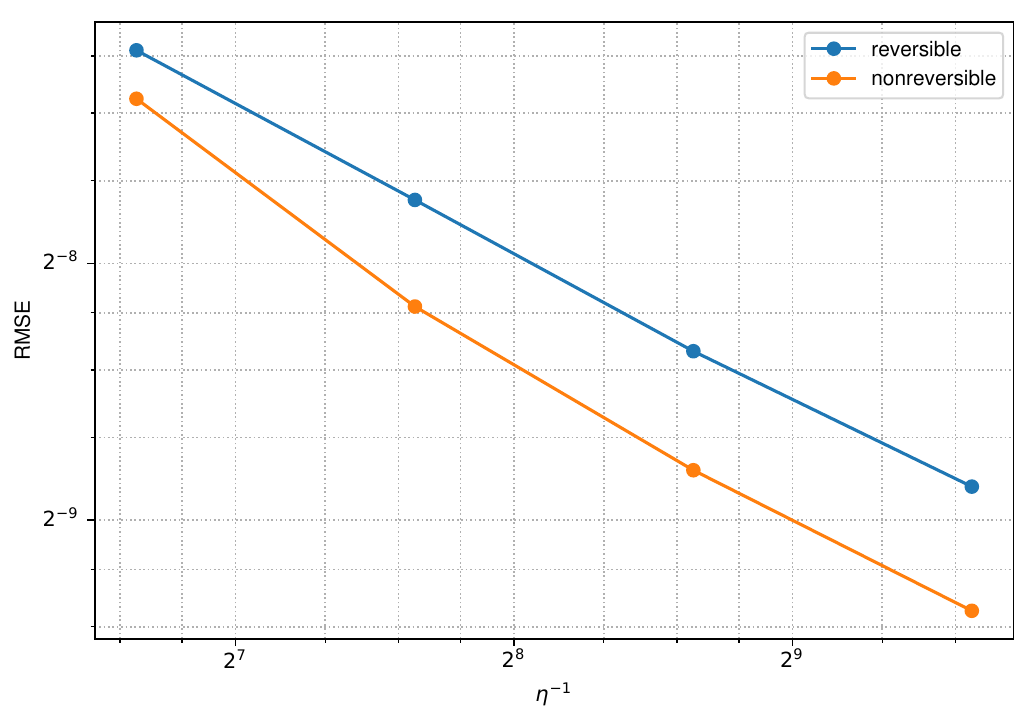}
        \caption{RMSE comparison.}
        \label{fig:logistic_classification_nld_loglog}
    \end{subfigure}
    \hfill
    \begin{subfigure}{0.48\linewidth}
        \centering
        \includegraphics[width=0.85\linewidth]{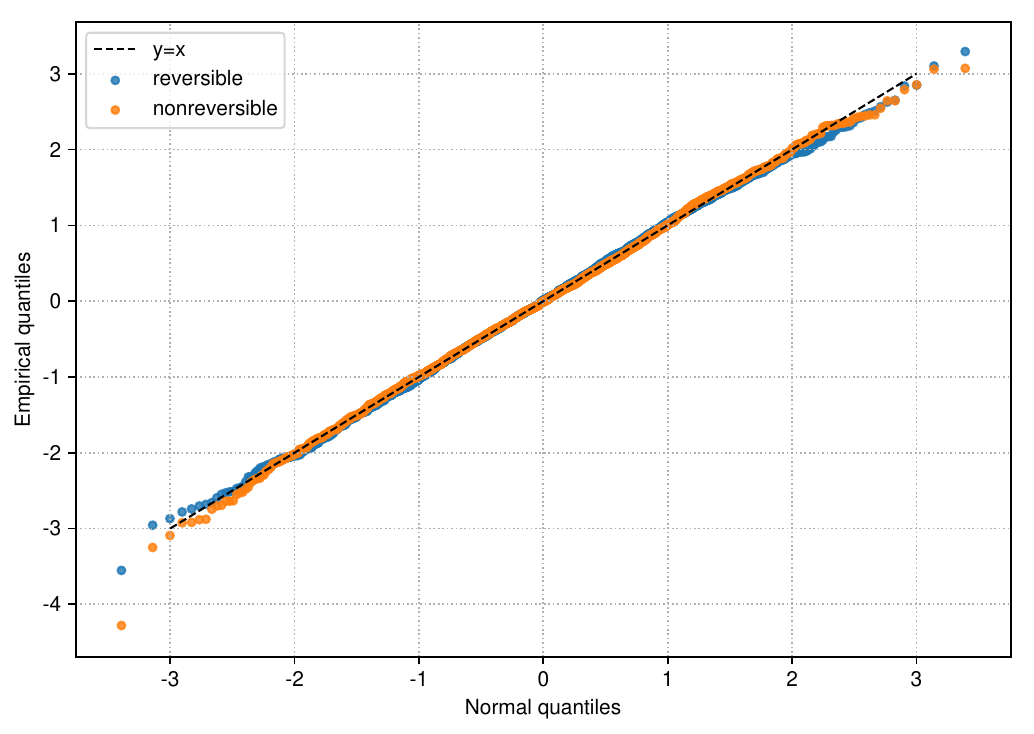}
        \caption{Normal quantile plot.}
        \label{fig:logistic_classification_nld_qq}
    \end{subfigure}
    \caption{Bayesian logistic classification using SGNLMC.}
    \label{fig:log:reg:nld}
\end{figure}

\subsubsection{The SGHFHRMC case}
For SGHFHRMC, we use the stochastic gradient Hessian free high resolution Monte Carlo scheme \eqref{eq:hfhr-em} with state $(\theta,r)\in\mathbb R^{2(p+1)}$. In the experiment, we take
$
    \alpha=0.1,
    \gamma=40.0.
$
The reversible baseline is obtained by setting the antisymmetric coupling to zero while keeping the same symmetric diffusion matrix, it can be obtained by Euler-Maruyama discretization of~\eqref{eq:hfhr-baseline-sde}.
%\begin{equation}
%\begin{cases}
%\theta_{k+1}
%=
%\theta_k-\eta\alpha \widehat{\nabla U}(\theta_k,\zeta_{k+1})
%+\sqrt{2\alpha\eta}\xi_{k+1}^{(1)},\\
%r_{k+1}
%=
%r_k-\eta\gamma r_k
%+\sqrt{2\gamma\eta}\xi_{k+1}^{(2)}.
%\end{cases}
%\end{equation}
In both schemes, the stochastic gradient $\widehat{\nabla U}$ is in \eqref{eq:logistic_stochastic_gradient_oracle}.

\begin{figure}[H]
    \centering
    \begin{subfigure}{0.48\linewidth}
        \centering
        \includegraphics[width=0.85\linewidth]{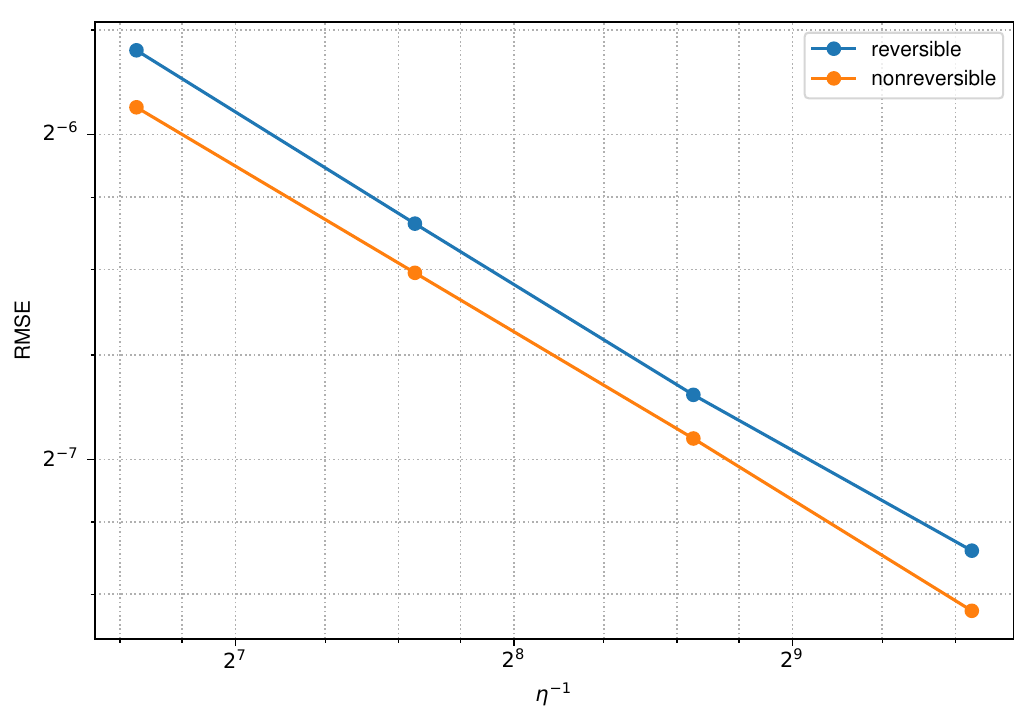}
        \caption{RMSE comparison.}
        \label{fig:logistic_classification_hfhr_loglog}
    \end{subfigure}
    \hfill
    \begin{subfigure}{0.48\linewidth}
        \centering
        \includegraphics[width=0.85\linewidth]{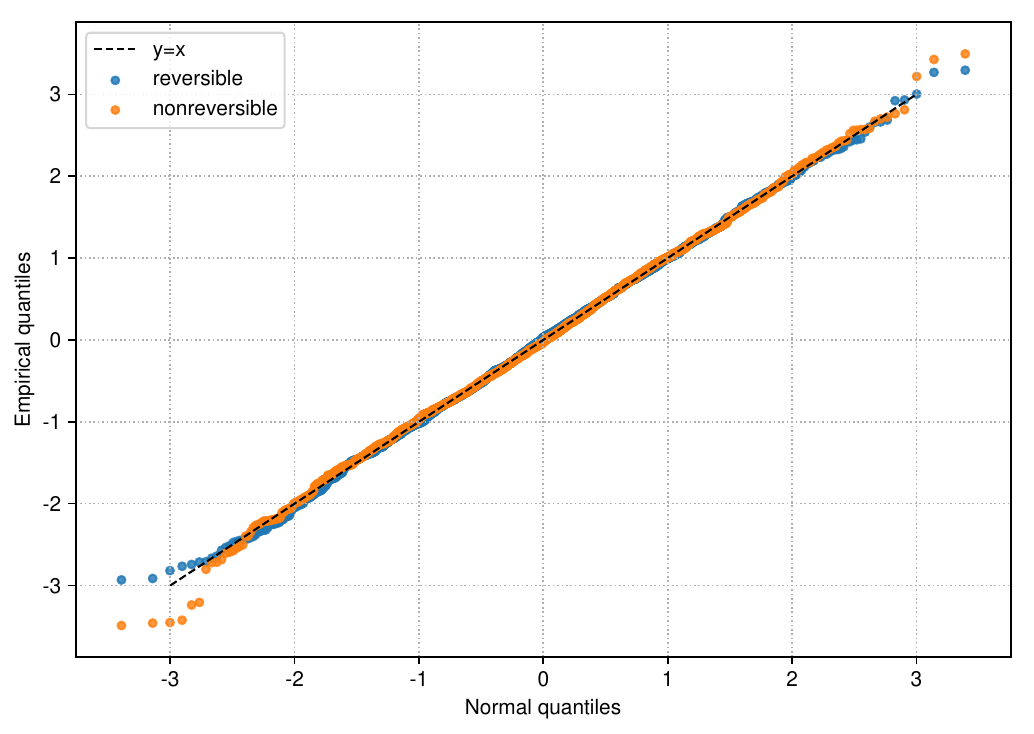}
        \caption{Normal quantile plot.}
        \label{fig:logistic_classification_hfhr_qq}
    \end{subfigure}
    \caption{Bayesian logistic classification using SGHFHRMC.}
    \label{fig:log:reg:hfhr}
\end{figure}

\subsubsection{The SGGAULMC case}
For SGGAULMC, we use the stochastic gradient gradient adjusted underdamped Langevin Monte Carlo scheme \eqref{eq:gaul-em} with state $(\theta,r)\in\mathbb R^{2(p+1)}$. In the experiment, we take
$
    C=5I_{p+1},
    a=0.1,
    \gamma=73.5.
$
The reversible baseline is obtained by setting the antisymmetric part to zero while keeping the same symmetric diffusion matrix $\mathcal D$, it follows~\eqref{eq:gaul-em} by taking $\mathbf{Q} = \mathcal{D}$.
%\begin{equation}
%    \begin{pmatrix}
%        \theta_{k+1}\\
%        r_{k+1}
%    \end{pmatrix}
%    =
%    \begin{pmatrix}
%        \theta_k\\
%        r_k
%    \end{pmatrix}
%    -
%    \eta \mathcal D
%    \begin{pmatrix}
%        \widehat{\nabla U}(\theta_k,\zeta_{k+1})\\
%        r_k
%    \end{pmatrix}
%    +
%    \sqrt{2\eta\mathcal D}\,\xi_{k+1}.
%\end{equation}
In both schemes, the stochastic gradient $\widehat{\nabla U}$ is in \eqref{eq:logistic_stochastic_gradient_oracle}.

\begin{figure}[H]
    \centering
    \begin{subfigure}{0.48\linewidth}
        \centering
        \includegraphics[width=0.85\linewidth]{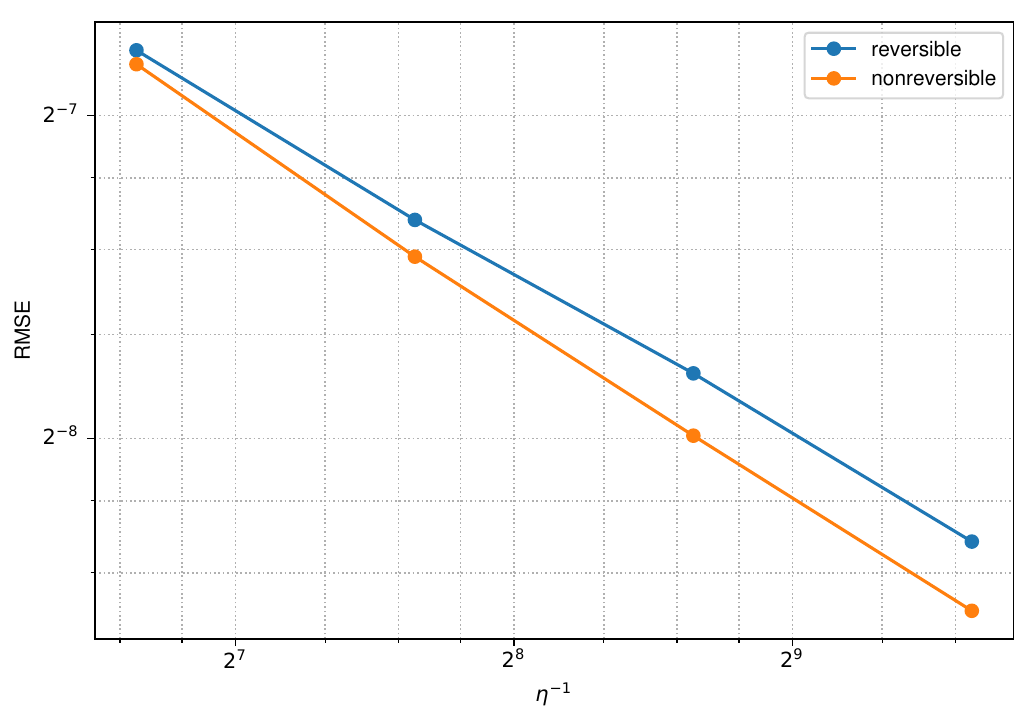}
        \caption{RMSE comparison.}
        \label{fig:logistic_classification_gaul_loglog}
    \end{subfigure}
    \hfill
    \begin{subfigure}{0.48\linewidth}
        \centering
        \includegraphics[width=0.85\linewidth]{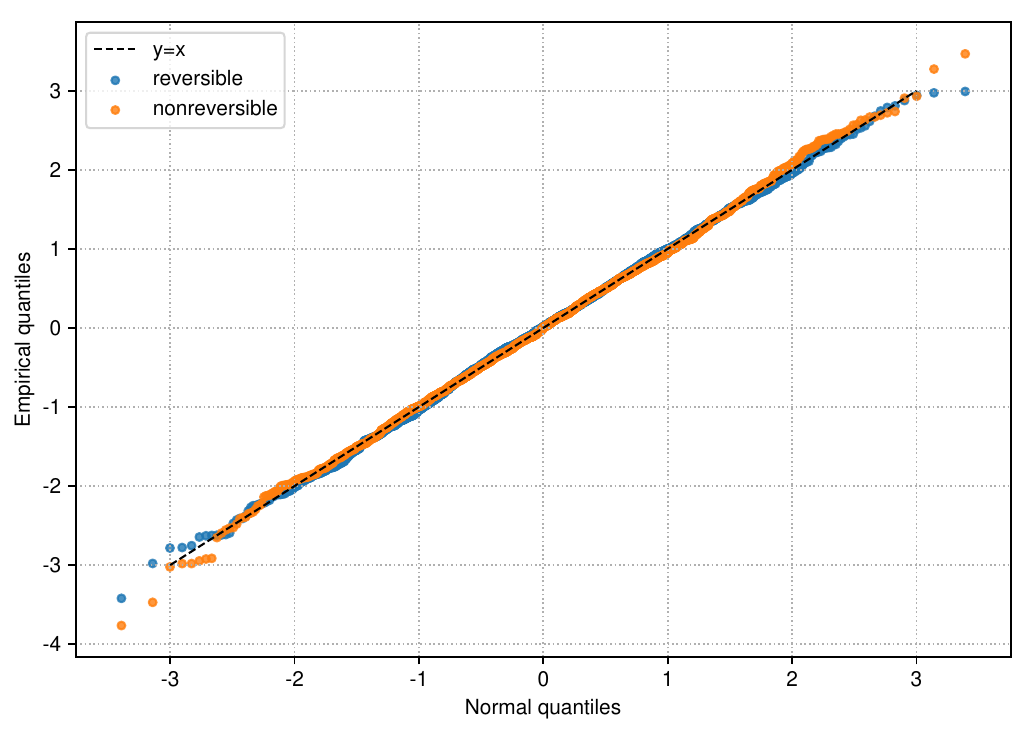}
        \caption{Normal quantile plot.}
        \label{fig:logistic_classification_gaul_qq}
    \end{subfigure}
    \caption{Bayesian logistic classification using SGGAULMC.}
    \label{fig:log:reg:gaul}
\end{figure}

Figures~\ref{fig:log:reg:nld}, \ref{fig:log:reg:hfhr}, and \ref{fig:log:reg:gaul} report the RMSE behavior for the three logistic regression experiments. In all cases, the non-reversible sampler achieves a smaller RMSE than the corresponding reversible baseline over the tested stepsizes. This suggests that the variance reduction effect observed in the quadratic and synthetic experiments persists for a nonquadratic posterior arising from a real classification dataset.

The normal quantile plots in Figures~\ref{fig:log:reg:nld}, \ref{fig:log:reg:hfhr}, and \ref{fig:log:reg:gaul} show the normalized empirical errors at $\eta=0.00125$. For all three algorithms, the empirical quantiles of both the non-reversible sampler and the reversible baseline lie close to the standard normal reference line $y=x$, supporting our theory on the Gaussian fluctuation behavior predicted by the small stepsize.

%%%%%%%%%%%%%%%%%%%%%%%%%%%%%%%%%%%%%%%%%%%%%%%%
\section{Conclusion}

We analyzed the leading-order fluctuation constant of stochastic gradient Euler--Maruyama discretizations for constant-coefficient generalized non-reversible Langevin dynamics (GNLD) in the small-stepsize regime. The main theorem follows by adopting a Poisson/martingale decomposition and showing that the additional stochastic gradient martingale is negligible under the CLT normalization. Consequently, fixed-variance unbiased stochastic gradient noise, in the sense of a conditional moment bound, does not change the leading fluctuation constant. This conclusion should not be read as covering high-variance stochastic gradient regimes in which the gradient noise grows as the stepsize goes to zero.
The comparison result shows how non-reversible variance reduction translates into a strict reduction of the stochastic gradient small-stepsize fluctuation constant under an explicit sufficient condition. The framework covers stochastic gradient versions of classical non-reversible Langevin dynamics (NLD), Hessian-free high-resolution (HFHR) dynamics, and a positive-definite subclass of gradient-adjusted underdamped Langevin (GAUL) dynamics. We also identified bounded smooth predictive observables, such as logistic predictive probabilities, Brier-type scores, and entropy-type uncertainty scores, that fall directly within our theory. As a separate Gaussian calculation beyond this bounded-test-function regime, we further computed the fluctuation constant explicitly for quadratic Hamiltonians and linear observables, yielding closed-form variance-reduction formulas including a simple reduction factor for HFHR and explicit matrix formulas for NLD and GAUL. 
The numerical experiments confirm these theoretical predictions across quadratic/nonquadratic targets and Bayesian linear regression using synthetic data, 
and Bayesian logistic regression using real data. In particular, the non-reversible schemes consistently reduce RMSE relative to their reversible baselines, and the normalized errors are well aligned with the Gaussian limits predicted by the small-stepsize CLT.

%%%%%%%%%%%%%%%%%%%%%%%%%

\section*{Acknowledgements}

Xiaoyu Wang is supported by the Guangzhou-HKUST(GZ) Joint Funding Program (No.2025A03J3556).
Lingjiong Zhu is partially supported by the grants NSF DMS-2053454 and NSF DMS-2208303.

\bibliographystyle{alpha}
\bibliography{bibtex}

\end{document}